\documentclass{article} %
\usepackage{conference,times}

\usepackage{amsmath,amsfonts,bm}

\def\eqref#1{equation~\ref{#1}}

\def\1{\bm{1}}

\def\rw{{\textnormal{w}}}

\def\vx{{\bm{x}}}

\def\mW{{\bm{W}}}

\DeclareMathAlphabet{\mathsfit}{\encodingdefault}{\sfdefault}{m}{sl}
\SetMathAlphabet{\mathsfit}{bold}{\encodingdefault}{\sfdefault}{bx}{n}

\def\sR{{\mathbb{R}}}

\usepackage[hidelinks]{hyperref}
\usepackage{url}
\usepackage{amssymb}
\usepackage{graphicx}
\graphicspath{{figures/}}
\usepackage{booktabs}
\usepackage{multirow}
\usepackage{algorithm}
\usepackage[noend]{algorithmic}   %
\usepackage{float}          %
\usepackage{caption}        %
\usepackage{wrapfig}        %
\usepackage{pgf}           %
\usepackage{tikz}          %
\usetikzlibrary{calc,arrows.meta}
\usepackage{amsthm}
\newtheorem{proposition}{Proposition}
\newtheorem{theorem}{Theorem}

\newtheorem{corollary}{Corollary}

\usepackage{cleveref} %
\crefname{appendix}{Appendix}{Appendices}
\Crefname{appendix}{Appendix}{Appendices}
\usepackage{enumitem}

\title{Activation Denoising: A Robustness View on\\ Parallel vs Sequential LLM Quantization}

\author{%
Yan Scholten\textsuperscript{1}, Rachel Lawrence\textsuperscript{2}, James Hensman\textsuperscript{2}, Stephan Günnemann\textsuperscript{1},\\
\textbf{Alicia Curth\textsuperscript{2,$*$}, Riccardo Grazzi\textsuperscript{2,$*$}}\\%[3pt]
\textsuperscript{1}Dept.\ of Computer Science \& Munich Data Science Institute, Technical University of Munich\\
\textsuperscript{2}Microsoft Research Cambridge, UK\\%[2pt]
\small\texttt{\{y.scholten,s.guennemann\}@tum.de}, \\
\small\texttt{\{rachel.lawrence,jameshensman,aliciacurth,t-rgrazzi\}@microsoft.com}%
}

\footnotetext[1]{Equal supervision}

\iclrfinalcopy %
\begin{document}

\maketitle

\vspace{-0.5cm}
\begin{abstract}
\vspace{-0.15cm}
Post-training quantization is a powerful tool for compressing large language models. The most scalable methods quantize every layer in parallel, but quantization errors then compound through the residual stream, as no layer corrects for the errors of the layers before it. Sequential quantization accounts for this error compounding by re-calibrating each layer on the already-quantized outputs of its predecessors, yielding stronger results but at the cost of a serial schedule that becomes a bottleneck at scale. As a solution, we propose \emph{parallel quantization with activation denoising}, which recovers much of the sequential benefit while keeping quantization fully parallel. Rather than re-calibrating layer-by-layer, we take a robustness perspective and model the upstream error as noise, regularizing to be robust to it through a preprocessing step followed by metric-weighted rounding. Applied at every layer, this regularization forms a depth-compounding smoothness penalty that dampens how strongly quantization errors amplify through the~model. Unlike orthogonal rotations commonly used in quantization, which must preserve the model's function, we multiply the weights by a more general linear transformation. We find that the two are complementary and their effects compound. Empirically, our robustness regularization recovers a significant part of sequential quantization's benefit in a single parallel pass, at a fraction of its time. Overall, by treating compounding quantization errors as a robustness problem, we offer a principled foundation for more efficient and accurate LLM quantization at~scale.
\end{abstract}
\vspace{-0.5cm}

\section{Introduction}
\label{sec:introduction}
\vspace{-0.1cm}

Post-training quantization (PTQ) is a powerful tool for deploying large language models (LLMs) under memory and bandwidth budgets, mapping high-precision weights to a few bits without retraining \citep{dettmers2022int8,dettmers2023kbit,frantar2023optq}. Quantization usually involves two steps: (1) a forward pass that collects each layer's input activations, and (2)~rounding the weights onto the quantization grid, a challenging optimization problem that dominates quantization time. The most scalable PTQ methods therefore round all layers independently and in~parallel.

\begin{figure}[t!]
    \centering
    \vspace{-0.5cm}
    \input{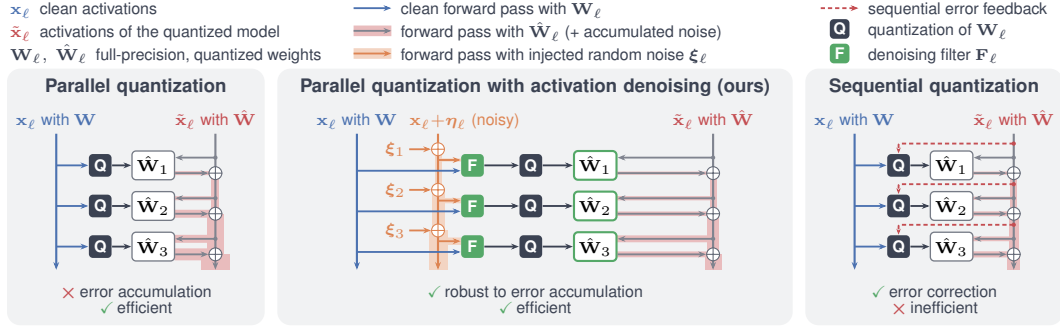}
    \vspace{-0.3cm}
    \caption{\textbf{A robustness perspective on the parallel-sequential quantization gap.}
    Quantization errors enter the residual stream at every layer and accumulate with depth. Parallel quantization~(left) is efficient but ignores this error; sequential quantization~(right) corrects it, but only one layer after another.
    Instead of measuring and correcting this error layer by layer, we regularize every layer against noise~(middle): we inject random noise in a single extra forward pass and use its moments to construct a denoising filter $\mathbf{F}_\ell$ that is absorbed into each layer's weights before rounding. Every layer thereby becomes more robust to noise, at no additional inference cost.
    Overall, our findings suggest that accumulated quantization error can be treated as a robustness problem: regularizing all layers against noise in parallel offers an efficient alternative to correcting the error sequentially.}
    \label{fig:hero}
\end{figure}

The key problem of existing parallel methods is that they sacrifice accuracy for efficiency: each layer is calibrated on clean inputs and never corrects for quantization errors it receives at inference time from previous layers, so errors compound through the model \citep{arai2025qep}. Sequential quantization corrects this by calibrating each layer on outputs of the layers quantized before it \citep{frantar2023optq,li2025gptaq}. This becomes a bottleneck at scale since layers must then be quantized one after another. Recovery fine-tuning \citep{tseng2024quipsharp,malinovskii2024pvtuning} can partially recover lost accuracy, but only as a separate costly training~stage.

As a solution, we take a robustness perspective on the gap between parallel and sequential quantization and propose \emph{parallel quantization with activation denoising}~(\autoref{fig:hero}): instead of correcting for the upstream quantization error one layer at a time, we make each layer robust to upstream perturbations. In a single additional forward pass through the full-precision model, we inject random noise into the activations and propagate it through the model. Then, for each weight matrix in parallel, we construct a linear denoising filter that estimates cleaner from noisy inputs and is absorbed into~the~weights. After this step, we round the modified weights in a noise-aware metric, independently and in parallel.
The denoising filter acts as a depth-compounding smoothness penalty that limits how strongly errors from quantized layers amplify through the LLM at inference.

Activation denoising adds weighted column-wise $L_2$ regularization, complementing rotations. Rotations reduce weight outliers by redistributing magnitude across channels \citep{chee2023quip,ashkboos2024quarot}, but preserve the Frobenius norm. To preserve the model's function, each rotation must be paired with its inverse, either fused into adjacent weights or applied online at inference. Our denoising filter instead absorbs its regularization into a general linear weight transform. Empirically, it lowers the Frobenius norm and improves over both random and learned rotations. \\
Our main contributions~are:

\begin{itemize}[itemsep=1pt,topsep=-3pt,leftmargin=0.5cm] 
    \item \textbf{Activation denoising framework.} We propose a quantization objective that makes each layer robust to perturbed inputs and generalizes existing parallel and sequential objectives. We show that it reduces exactly to metric-weighted rounding with a transformed target and metric, allowing the use of existing rounding algorithms  and adding no inference cost.
    \item \textbf{Noise injection as a robustness regularizer.} To enable parallel and robust quantization, we estimate the required statistics from a single additional forward pass with injected Gaussian noise at the input of each weight, and prove that the local injection induces a weighted column-wise $L_2$ penalty toward down-scaled targets.
    \item \textbf{Empirical validation.} We demonstrate that parallel activation denoising consistently improves $3$- and $4$-bit scalar and $2$-bit vector quantization across models from $1$B to $13$B parameters, recovering a significant part of the benefit of sequential quantization while all layers are still quantized in parallel.
    Finally, we show that the remaining gap to sequential quantization can be closed through cheap trade-offs rather than a full sequential pass. 
    
\end{itemize}

Overall, addressing quantization errors from a robustness perspective pushes the efficiency--accuracy frontier of LLM quantization, narrowing the gap to sequential quantization at a fraction of its time.

\vspace{-0.6cm}
\section{Related work}
\vspace{-0.1cm}
\label{sec:related-work}

\textbf{Post-training quantization for LLMs.} PTQ compresses a pretrained LLM using a small calibration set without retraining \citep{yao2022zeroquant,dettmers2022int8,dettmers2023kbit}. Below 4 bits, the dominant approach minimizes a layer-wise reconstruction error with second-order information \citep{nagel2020upordown,frantar2022optimalbrain,frantar2023optq}, recently refined with Kronecker-factored curvature \citep{tseng2025modelpreserving,lee2026kronq}. Complementary directions protect salient or outlier weights \citep{lin2024awq,dettmers2024spqr,kim2024squeezellm}, allocate mixed precision \citep{huang2025slimllm,zhao2024atom}, learn quantization parameters block-wise \citep{shao2024omniquant,liu2024qllm}, or replace the uniform grid with floating-point formats \citep{liu2023llmfp4} or vector or trellis codes \citep{egiazarian2024extreme,vanbaalen2024gptvq,tseng2024quipsharp,tseng2024qtip,ouderaa2024pyramid}.  Our proposed framework is general and supports both scalar and vector quantization. Moreover, we focus on finetuning-free PTQ, avoiding the extra training stage of quantization-aware training and recovery fine-tuning \citep{liu2024llmqat,malinovskii2024pvtuning} or the inference cost of full-precision low-rank corrections \citep{scetbon2024lowrank}.

\textbf{Outliers and incoherence processing.}
Outliers in weights and activations limit low-bit precision \citep{wei2022outlier,bondarenko2023quantizable,zhang2024magr} and can be migrated between activations and weights by equivalent channel-wise scaling \citep{xiao2023smoothquant,wei2023outliersuppressionplus}. Incoherence processing instead rotates weights and activations using random orthogonal or Hadamard matrices \citep{chee2023quip,tseng2024quipsharp}, which spread outliers without changing the model output, but modify the forward pass. Exploiting the computational invariance of transformers \citep{ashkboos2024slicegpt}, some of these rotations can be fully fused into weights, leaving the forward pass unchanged but tying the rotations across layers \citep{ashkboos2024quarot}. Recent work optimizes these transformations rather than sampling them \citep{lin2024duquant,liu2025spinquant,akhondzadeh2025kurtail,shao2025dart,gadhikar2025optrot,sun2025flatquant}. These transformations make weights and activations easier to quantize while preserving the full-precision model's output. Our denoising instead changes each layer's mapping to compensate for upstream errors and reduces the weights' Frobenius norm, complementing the outlier redistribution of orthogonal~rotations.

\textbf{Cross-layer error compensation.}
Sequential methods compensate for the error accumulated in already-quantized layers 
\citep{nagel2020upordown,li2021brecq,ding2025cbq}: GPTAQ \citep{li2025gptaq} matches each quantized layer to the full-precision output within GPTQ's column-wise updates, QEP \citep{arai2025qep} achieves the same through a closed-form weight correction applied before rounding, Qronos \citep{zhang2025qronos} corrects it within the rounding updates, and CoreQ \citep{cha2026coreq} derives the correction strength in closed form. The corrections above require measuring the accumulated error in a fully sequential pass. We instead replace the measurement with a noise model estimated in parallel and account for it in the rounding target and metric; our objective recovers those of GPTQ, GPTAQ, and QEP as special cases (\autoref{sec:wiener}). Our random noise model relates to randomly dropping activation quantization during reconstruction \citep{wei2022qdrop}, but acts as a quantization-time regularizer that requires no gradient-based~optimization.

\vspace{-0.3cm}
\section{Background}
\label{sec:background}
\vspace{-0.25cm}

\textbf{Post-training quantization.}
Post-training quantization (PTQ) compresses a trained model by mapping its weights from high precision (e.g.\ $16$ bits) to a few bits, without retraining. Let $\mathcal{G}_b$ denote the set of weight matrices representable with $b$ bits per weight: uniform scalar quantization rounds each weight to one of $2^b$ levels of a uniform grid, with scale and zero-point shared within a row or group, whereas vector quantization rounds groups of $g$ weights jointly to one of $2^{bg}$ entries. Since rounding weights independently ignores how their errors interact, a standard PTQ objective instead minimizes the layer's output error on a small calibration set: for weights $\mathbf{W}\in\mathbb{R}^{d_{\mathrm{out}}\times d}$ and inputs $\mathbf{x}\in\mathbb{R}^{d}$ from a calibration set $\mathcal{D}$, $\min_{\hat{\mathbf{W}}\in\mathcal{G}_b}\mathbb{E}_{\mathbf{x}\sim\mathcal{D}}\|(\hat{\mathbf{W}}-\mathbf{W})\mathbf{x}\|^2$. To second order, and treating each layer independently, this approximates the KL divergence between the full-precision and the quantized model, decoupling row by row within each weight matrix $\mathbf{W}$, with curvature $\mathbf{H}=\mathbb{E}[\mathbf{x}\mathbf{x}^\top]$.

\textbf{Metric-weighted rounding.}
More generally, this is a metric-weighted rounding problem: given a continuous target $\mathbf{W}^\star$, a positive-definite metric $\mathbf{M}$, and a grid $\mathcal{G}_b$,
{
\setlength{\abovedisplayskip}{2pt}
\setlength{\belowdisplayskip}{1pt}
\begin{equation}
\min_{\hat{\mathbf{W}}\in\mathcal{G}_b}\ \operatorname{tr}\!\big[(\hat{\mathbf{W}}-\mathbf{W}^\star)\,\mathbf{M}\,(\hat{\mathbf{W}}-\mathbf{W}^\star)^\top\big],
\label{eq:gptq-operator}
\end{equation}
}%
where $\mathbf{M}$ decides which weight perturbations are cheap. In the simple case above $\mathbf{M}=\mathbf{H}$ and $\mathbf{W}^\star =\mathbf{W}$, while in general they can be different. \autoref{eq:gptq-operator} is an expensive combinatorial optimization problem. Existing rounding algorithms solve it approximately: GPTQ~\citep{frantar2023optq}, for example, greedily quantizes one column at a time and feeds each column's rounding error onto the remaining columns through $\mathbf{M}^{-1}$. We write $\operatorname{Round}^{\,b}_{\mathbf{M}}(\mathbf{W}^\star)$ for the approximate solution returned by any such rounding algorithm in the metric $\mathbf{M}$.

\textbf{Sequential vs.\ parallel quantization.}
Layer-wise PTQ methods differ in whether a layer is calibrated on clean inputs $\mathbf{x}$ or on perturbed inputs $\tilde{\mathbf{x}}$ it receives once the earlier layers are quantized. Parallel quantization uses $\mathbf{x}$, so all layers quantize independently but ignore the upstream error. Existing sequential quantization compensates for the upstream error either symmetrically, $\mathbb{E}\|(\hat{\mathbf{W}}-\mathbf{W})\tilde{\mathbf{x}}\|^2$ using only $\tilde{\mathbf{x}}$, or asymmetrically, $\mathbb{E}\|\hat{\mathbf{W}}\tilde{\mathbf{x}}-\mathbf{W}\mathbf{x}\|^2$ using both $\tilde{\mathbf{x}}$ and $\mathbf{x}$ (see e.g.\ \citep{li2025gptaq}).

\section{Activation denoising quantization: A filtering perspective}
\label{sec:wiener}

Quantization errors compound across layers: each layer's rounding error propagates downstream. Existing parallel methods ignore the upstream error during quantization. Sequential quantization corrects it, but has to round the layers one after another, which prevents parallelizing the rounding step. This leaves a gap between efficient parallel and accurate sequential quantization. As a solution, our main idea is to denoise each layer's input activations in parallel. To this end, we first propose a denoising quantization objective and show that it reduces to a linear weight transform followed by rounding in a new metric, at zero additional inference cost.

\vspace{-0.05cm}
\subsection{Quantization with activation denoising}
\vspace{-0.05cm}

We formulate quantization with activation denoising using the following objective:
{%
\setlength{\abovedisplayskip}{3pt}
\setlength{\belowdisplayskip}{2pt}
\begin{equation}
\min_{\hat{\mathbf{W}}\in\mathcal{G}_b}\;\; \mathbb{E}_{\mathbf{x},\boldsymbol{\eta}}\big\|\hat{\mathbf{W}}\big(\mathbf{x}+\lambda\,\boldsymbol{\eta}\big)-\mathbf{W}\big(\mathbf{x}+\rho\,\boldsymbol{\eta}\big)\big\|^2 \;+\; \alpha\,\|\hat{\mathbf{W}}-\mathbf{W}\|_F^2 ,
\label{eq:wiener-objective}
\end{equation}
}%
where $\mathbf{W}$ are the full-precision weights and $\hat{\mathbf{W}}\in\mathcal{G}_b$ the quantized weights on the $b$-bit grid, $\mathbf{x}$~are the clean input activations from the full-precision model and $\boldsymbol{\eta}$ a perturbation of them, with the expectation over their joint distribution ($\mathbf{x}\,{\sim}\,\mathcal{D}$). The first term is a reconstruction objective on the layer's output: the quantized weights receive the noisy input $\mathbf{x}+\lambda\boldsymbol{\eta}$ and must match the full-precision output evaluated on the target input $\mathbf{x}+\rho\boldsymbol{\eta}$. The input noise scale $\lambda\in[0,1]$ sets how much error the quantized layer's input carries, and the target noise scale $\rho\in[0,1]$ how much of it the target keeps. 
The second term is a proximity penalty on the weights, where $\alpha\,{\ge}\,0$ is a proximity~ridge.

\textbf{Partial vs.\ full denoising.} Keeping less noise in the target than in the input ($\rho\,{<}\,\lambda$) asks the layer to denoise its input. Here, $\rho\,{=}\,0$ is an asymmetric correction that requires the layer to remove the noise entirely, and $\rho\,{=}\,\lambda$ yields a symmetric correction, where the layer adapts to the noise without removing it. Values in between ($0\,{<}\,\rho\,{<}\,\lambda$) correspond to partial denoising. We treat $\lambda$ and $\rho$ as hyperparameters and find that partial denoising typically performs best in practice, even when $\boldsymbol{\eta}$ is the exact sequential error (see \Cref{app:sweep-range}).

\textbf{The activation denoising objective establishes a general and flexible quantization framework.}
Notably, the objective we aim to minimize in \autoref{eq:wiener-objective} generalizes existing parallel and sequential quantization objectives into a flexible framework. Specifically, we can recover the parallel GPTQ objective \citep{frantar2023optq} for $\lambda{=}0$ and $\rho{=}0$. When using the sequential error $\boldsymbol{\eta}_{\mathrm{seq}}=\tilde{\mathbf{x}}-\mathbf{x}$ (where $\tilde{\mathbf{x}}$ is the input when already quantized up to this layer), we can further recover sequential GPTQ ($\lambda{=}1$, $\rho{=}1$), as well as sequential GPTAQ \citep{li2025gptaq} and QEP \citep{arai2025qep} for $\lambda{=}1$ with $1{-}\rho$ being their correction strength hyperparameter. We show this in detail in \cref{sec:relation}.
In contrast to these baselines, our framework is more general and supports general noise models beyond the sequential $\boldsymbol{\eta}_{seq}$ and flexible noise scales $\lambda$ and $\rho$. Together, this allows us to find better trade-offs at the efficiency-accuracy frontier of LLM quantization.

\vspace{-0.05cm}
\subsection{Denoising quantization in practice: Preprocess then round}
\vspace{-0.05cm}

In the following, we show how to solve the optimization problem in \autoref{eq:wiener-objective} by transforming the weights and rounding them in a new metric using existing rounding algorithms.
\begin{theorem}[Pretransform-round equivalence]
\label{prop:pretransform-exact}
For $\alpha>0$, the optimization problem in \autoref{eq:wiener-objective} is equivalent to the one in \autoref{eq:gptq-operator} with a specific metric and target: it is equivalent to
{%
\setlength{\abovedisplayskip}{4pt}
\setlength{\belowdisplayskip}{2pt}
\begin{equation*}
\min_{\hat{\mathbf{W}}\in\mathcal{G}_b}\ \operatorname{tr}\!\big[(\hat{\mathbf{W}}-\tilde{\mathbf{W}})\,\mathbf{G}\,(\hat{\mathbf{W}}-\tilde{\mathbf{W}})^\top\big], \quad \tilde{\mathbf{W}} \triangleq \mathbf{W}\mathbf{F}
\end{equation*}
}
where
\begin{equation}
\begin{aligned}
\mathbf{G}\;\triangleq\;\mathbf{H}+\alpha\mathbf{I}+\lambda^2\mathbf{N}+\lambda(\mathbf{C}+\mathbf{C}^\top),\quad
\mathbf{F}\;\triangleq\;\big(\mathbf{H}+\alpha\mathbf{I}+\lambda\mathbf{C}+\rho\mathbf{C}^\top+\rho\lambda\mathbf{N}\big)\,\mathbf{G}^{-1},
\end{aligned}
\label{eq:wiener-filter}
\end{equation}
and $\mathbf{H}=\mathbb{E}[\mathbf{x}\mathbf{x}^\top]$ is the input covariance, $\mathbf{C}=\mathbb{E}[\mathbf{x}\boldsymbol{\eta}^\top]$ the signal-noise cross-covariance, and $\mathbf{N}=\mathbb{E}[\boldsymbol{\eta}\boldsymbol{\eta}^\top]$ the noise covariance.
\end{theorem}
\textit{Proof sketch.} Expanding \autoref{eq:wiener-objective} and completing the square shows that the objectives differ only by a constant independent of $\hat{\mathbf{W}}$, so they have the same minimizers. Full proof in \Cref{app:proofs-pretransform}. \hfill\qedsymbol

As a result, we can use any approximate solver of the rounding problem in \autoref{eq:gptq-operator} and apply it to the pretransformed weights $\tilde{\mathbf{W}}$ in the metric $\mathbf{G}$ to solve our denoising objective in~\autoref{eq:wiener-objective}.

Interestingly, at $\rho=0$ (and $\alpha{=}0$ when the inverse exists), $\mathbf{F}$ corresponds to the linear transformation known as the generalized Wiener filter \citep{wiener1949extrapolation,kailath2000linear}, which denoises the input: $\mathbf{F}\tilde{\mathbf{x}}$ estimates the clean input activations $\mathbf{x}$ from the corrupted input $\tilde{\mathbf{x}}=\mathbf{x}+\lambda\boldsymbol{\eta}$.

\textbf{Parallel denoising quantization in practice.}
\autoref{prop:pretransform-exact} turns the denoising quantization objective into a practical quantization procedure as follows. Given upstream noise $\boldsymbol{\eta}$ (either the true sequential quantization error $\boldsymbol{\eta}_{\mathrm{seq}}$, or a parallel approximation of it), we estimate the second-order statistics required by \autoref{prop:pretransform-exact} as empirical token averages over a calibration set of $n$ tokens: $\mathbf{H}=\tfrac1n\sum_i \mathbf{x}_i\mathbf{x}_i^\top$, $\mathbf{C}=\tfrac1n\sum_i \mathbf{x}_i\boldsymbol{\eta}_i^\top$, and $\mathbf{N}=\tfrac1n\sum_i \boldsymbol{\eta}_i\boldsymbol{\eta}_i^\top$.%

\begin{minipage}[t]{0.46\textwidth}
If $\boldsymbol{\eta}$ can be computed in parallel for all layers, each layer's solution depends only on its own moments and not on other quantized layers, 
and all layers can be quantized independently: we form metric $\mathbf{G}$, compute the pretransform $\tilde{\mathbf{W}}=\mathbf{W}\mathbf{F}$, and round it onto the grid $\mathcal{G}_b$ in the metric $\mathbf{G}$. In practice, we round using the GPTQ rounding proposed by \citet{frantar2023optq} and QuIP\# \citep{tseng2024quipsharp}. We summarize this procedure in \Cref{alg:wiener-quant}.
\end{minipage}%
\hfill
\begin{minipage}[t]{0.52\textwidth}%
\vspace{-0.6cm}
\begin{algorithm}[H]
\footnotesize
\captionof{algorithm}{\footnotesize Activation denoising quantization}
\label{alg:wiener-quant}
\begin{algorithmic}[1]
\setlength{\itemsep}{0.3pt}%
\REQUIRE full-precision weights $\{\mathbf{W}_m\}$, bit-width $b$, emp.\ moments $\{\mathbf{H}_m,\mathbf{N}_m,\mathbf{C}_m\}$,  noise scales $\lambda,\rho$, ridge~$\alpha$%
\FORALL{linears $m$ \textbf{in parallel}}
  \STATE $\mathbf{G}_m \gets \mathbf{H}_m+\alpha\mathbf{I}+\lambda^2\mathbf{N}_m+\lambda(\mathbf{C}_m+\mathbf{C}_m^\top)$
  \STATE $\mathbf{F}_m \gets (\mathbf{H}_m+\alpha\mathbf{I}+\lambda\mathbf{C}_m+\rho\mathbf{C}_m^\top+\rho\lambda\mathbf{N}_m)\mathbf{G}_m^{-1}$%
  \STATE $\tilde{\mathbf{W}}_m \gets \mathbf{W}_m\mathbf{F}_m$
  \STATE $\hat{\mathbf{W}}_m \gets \smash[b]{\operatorname{Round}^{\,b}_{\mathbf{G}_m}}(\tilde{\mathbf{W}}_m)$
\ENDFOR%
\RETURN $\{\hat{\mathbf{W}}_m\}$
\end{algorithmic}
\end{algorithm}
\end{minipage}

Parallel denoising can thus be implemented as an efficient quantization-time, per-layer weight pretransform followed by standard rounding, adding no additional inference cost and requiring no sequential computation given an estimate for the upstream noise $\boldsymbol{\eta}$.

\vspace{-0.2cm}
\section{Activation noise model for efficient parallel denoising}
\vspace{-0.1cm}
\label{sec:injection}

Sequential methods instantiate the noise $\boldsymbol{\eta}$ in the denoising objective (\autoref{eq:wiener-objective}) with the true sequential upstream error, which is only available after quantizing all previous layers in sequence. Instead, our key idea is to inject random noise into the input activations of every layer during an additional calibration forward pass through the full-precision model, and to propagate it downstream. The input to each module $m$ is thereby perturbed by a total noise $\boldsymbol{\eta}_m = \boldsymbol{\eta}^{\uparrow}_m + \boldsymbol{\xi}_m$: the upstream noise $\boldsymbol{\eta}^{\uparrow}_m$ propagated from earlier layers, plus a new local noise injection $\boldsymbol{\xi}_m$.

\textbf{Local random noise injection.} We inject independent, zero-mean, per-channel Gaussian noise into the input activations at each linear weight $\mathbf{W}_m$, sampled i.i.d.\ for every calibration token $t$:
{
\setlength{\abovedisplayskip}{3pt}
\setlength{\belowdisplayskip}{3pt}
\begin{equation*}
\boldsymbol{\xi}_{m,t}\sim\mathcal{N}\big(\mathbf{0},\boldsymbol{\Sigma}_m\big),\qquad \boldsymbol{\Sigma}_m=\operatorname{diag}\big(\sigma_{m,1}^2,\dots,\sigma_{m,d}^2\big),\qquad \sigma_{m,c} \;\triangleq\; \gamma\;\max_t\,\big|x_{m,t,c} + \eta^{\uparrow}_{m,t,c}\big|,
\end{equation*}
}%
where $x_{m,t,c}$ denotes channel $c$ of input $\mathbf{x}_{m,t}\in\mathbb{R}^d$. Here, the noise injected into channel $c$ has its own standard deviation $\sigma_{m,c}$ corresponding to the largest absolute value the channel attains across calibration tokens, scaled by a model-wide constant $\gamma\ge0$. High-magnitude outlier channels thus receive proportionally more noise, so the layer is regularized to be robust where its inputs are largest. %

\textbf{Noise propagation.} We propagate the injected noise downstream. Starting without upstream noise ($\boldsymbol{\eta}^{\uparrow}_1=\mathbf{0}$), the total noise $\boldsymbol{\eta}_m = \boldsymbol{\eta}^{\uparrow}_m + \boldsymbol{\xi}_m$ at module $m$ propagates through the model to form the next module's upstream noise $\boldsymbol{\eta}^{\uparrow}_{m+1}$. Early injections thus accumulate through depth, whereas each local noise injection $\boldsymbol{\xi}_m$ is a new, independent sample.

\begin{minipage}[t]{0.41\textwidth}
\textbf{Efficient parallel solution.} Crucially, no separate optimization is needed: one clean and one noise-injected forward pass (both at full-precision) yield the per-weight moments $\mathbf{H}_m,\mathbf{C}_m,\mathbf{N}_m$, which the pretransform absorbs before rounding each weight independently at no additional optimization or inference cost. We summarize this procedure in \Cref{alg:moments-new}. Since the moments come from full-precision forward passes, we can quantize all layers independently in parallel afterwards using \Cref{alg:wiener-quant}.
\end{minipage}%
\hfill
\begin{minipage}[t]{0.57\textwidth}%
\vspace{-0.6cm}
\begin{algorithm}[H]
\caption{Estimating denoising moments}
\label{alg:moments-new}
\small
\begin{algorithmic}[1]
\setlength{\itemsep}{1.2pt}
\REQUIRE $n$ calibration tokens; linear weights $\{\mathbf{W}_m\}$
\STATE $\{\mathbf{x}_{m,t}\} \gets$ per-linear inputs from full-precision forward
\STATE $\{\tilde{\mathbf{x}}_{m,t}\} \gets$ per-linear inputs from full-precision forward\\that locally injects and propagates noise $\boldsymbol{\xi}_{m,t}{\sim}\mathcal{N}(\mathbf{0},\boldsymbol{\Sigma}_m)$
\FORALL{linears $m$ \textbf{in parallel}}
  \STATE $\boldsymbol{\eta}_{m,t} \gets \tilde{\mathbf{x}}_{m,t} - \mathbf{x}_{m,t}$ \quad(calib.\ token $t{=}1{:}n$)
  \STATE $\mathbf{H}_m \gets \tfrac1n\sum_t \mathbf{x}_{m,t}\mathbf{x}_{m,t}^\top$
  \STATE $\mathbf{C}_m \gets \tfrac1n\sum_t \mathbf{x}_{m,t}\boldsymbol{\eta}_{m,t}^\top$
  \STATE $\mathbf{N}_m \gets \tfrac1n\sum_t \boldsymbol{\eta}_{m,t}\boldsymbol{\eta}_{m,t}^\top$
\ENDFOR
\RETURN moments $\{\mathbf{H}_m,\mathbf{N}_m,\mathbf{C}_m\}$
\end{algorithmic}
\end{algorithm}
\end{minipage}%

\subsection{Random noise injection as robustness regularization}
\vspace{-0.1cm}
\label{sec:random-noise}

We show that denoising against random noise regularizes weights to be more robust to input perturbations.
It induces two effects: \emph{(i)~local column-wise regularization} from injected noise $\boldsymbol{\xi}$, and \emph{(ii)~depth-compounding regularization} from propagated noise~$\boldsymbol{\eta}^\uparrow$, which controls how much error accumulates across layers. The following theorem formalizes effect~(i): the local injection combines with the proximity ridge into weighted $L_2$ penalties toward down-scaled weight~columns:
\begin{theorem}[Local noise injection acts as a magnitude-reducing ridge]
\label{thm:inject-lipschitz}
Let $\boldsymbol{\eta}=\boldsymbol{\eta}^{\uparrow}+\boldsymbol{\xi}$, where $\boldsymbol{\xi}\sim\mathcal{N}(\mathbf{0},\boldsymbol{\Sigma})$ is independent of $(\mathbf{x},\boldsymbol{\eta}^{\uparrow})$ and $\boldsymbol{\Sigma}=\operatorname{diag}(\sigma_1^2,\dots,\sigma_d^2)$. Up to a constant independent of the weights, the objective in \autoref{eq:wiener-objective} equals its reconstruction term with $\boldsymbol{\eta}=\boldsymbol{\eta}^{\uparrow}$ plus the modified ridge regularization
{
\setlength{\abovedisplayskip}{0pt}
\setlength{\belowdisplayskip}{1pt}
\begin{equation*}
\sum_c\big(\alpha+\lambda^2\sigma_{c}^2\big)\,\big\|\hat{\mathbf{W}}_{:,c}-s_{c}\,\mathbf{W}_{:,c}\big\|_2^2, \qquad s_{c}=\frac{\alpha+\rho\lambda\sigma_{c}^2}{\alpha+\lambda^2\sigma_{c}^2}.
\end{equation*}
}
For $\rho<\lambda$ and $\sigma_c>0$, the penalty favors a down-scaled target $s_c\mathbf{W}_{:,c}$ for column $c$.
\end{theorem}
\vspace{-0.15cm}
{
\textit{Proof sketch.}  Splitting the $\boldsymbol{\xi}$-dependent part, the cross term vanishes due to independence, and the quadratic term simplifies to $\sum_c\sigma_{c}^2\|\lambda\tilde{\mathbf{W}}_{:,c}-\rho\mathbf{W}_{:,c}\|_2^2$. Since $\boldsymbol{\Sigma}$ is diagonal, adding the ridge decouples over columns, resulting in the stated~form. Full proof in \Cref{app:proofs-injection}.\hfill\qedsymbol
}%
\vspace{-0.05cm}

In the denoising regime $\rho<\lambda$, this penalty encourages smaller column-wise $L_2$ norms by pulling weights toward down-scaled targets. In \autoref{sec:experiments}, we empirically show that this reduces the matrix Frobenius norm, which upper-bounds the module's spectral norm and hence its Lipschitz constant. The smaller norm reduces average amplification of isotropic input perturbations, making layers more robust to them. This complements rotations, which redistribute magnitude but preserve norm.

{
\vspace{-0.05cm}
Taking both effects (i) and (ii) together, we regularize each layer to be robust to noise: quantization stays parallel, and the robustness dampens how strongly quantization errors compound across layers.%
}%

\vspace{-0.22cm}
\section{Experimental evaluation}
\vspace{-0.18cm}
\label{sec:experiments}

Parallel quantization processes layers independently but ignores upstream quantization errors. Sequential quantization corrects them, but must process layers one after another.
Our method keeps quantization fully parallel: it regularizes each layer to be robust to input noise (\autoref{sec:injection}) via activation denoising (\autoref{sec:wiener}).
In this section we demonstrate the effectiveness of our parallel denoising and show that, across model scales, it recovers a significant fraction of the sequential quality.  
We report additional results and ablations in \Cref{app:additional-results} and describe the full setup in \Cref{app:setup}.

\textbf{Setup.} We quantize Llama-3.2-1B, Llama-3.2-3B, Llama-3-8B, and Llama-2-13B \citep{llama2, llama3}, which are widely used in the quantization literature and span a broad range of scales. We focus on weight quantization, keeping activations and KV cache at $16$ bits, and calibrate on $128$ 2048-token sequences of the WikiText-2 \citep{merity2017pointer} training set.

\textbf{Solvers.} Our approach is agnostic to the rounding solver, and we evaluate it with scalar and vector quantization via GPTQ \citep{frantar2023optq} at $3$~bits and QuIP\# \citep{tseng2024quipsharp} (without fine-tuning) at $2$~bits. For GPTQ, we rotate weights before quantization, using the random Hadamard rotations $R_1$, $R_2$, $R_4$ of QuaRot \citep{ashkboos2024quarot}, optionally replacing $R_1$, $R_2$ with learned rotations as proposed in SpinQuant \citep{liu2025spinquant}. QuIP\# applies its own randomized Hadamard incoherence transform. The rotations $R_1$ and $R_2$ are fused into the model's weights and modify the forward pass during moment estimation, whereas QuIP\#'s  rotations are applied during quantization, leave the forward pass unchanged during moment estimation, and modify it instead at~inference.

\textbf{Baselines.} For each solver, we compare against the two ends of the parallel--sequential spectrum: parallel quantization on clean inputs, and sequential quantization on outputs of the already quantized layers, with either the symmetric correction of GPTQ \citep{frantar2023optq} or the asymmetric correction of GPTAQ \citep{li2025gptaq}. The asymmetric correction is the strongest sequential baseline.

\textbf{Metrics.} For each method, we select the best configuration by WikiText-2 \citep{merity2017pointer} test perplexity and report perplexity averaged over WikiText-2, C4 \citep{raffel2020exploring}, and FineWeb \citep{penedo2024fineweb} test (i.e.\ not calibration) sequences. Excluding WikiText-2 from the reported average yields the same conclusions (\cref{app:generalization}). We also report zero-shot accuracy averaged over eight commonsense tasks \citep{liu2025spinquant}, and the fraction of the gap in avg.\ perplexity from $0\%$ (parallel) to $100\%$ (sequential asymmetric) for each solver separately.

\begin{table}[t]
\centering
\caption{\textbf{Parallel denoising closes a large part of the parallel--sequential gap for both scalar and vector quantization.} $3$-bit GPTQ and $2$-bit QuIP\# on Llama-3.2-1B/3B, Llama-3-8B, and Llama-2-13B. Acc: zero-shot accuracy; Avg: average perplexity; Gap: fraction of the perplexity gap closed, from $0\%$ (parallel) to $100\%$ (sequential asym.\ correction). On Llama-2-13B at $2$ bits, the two methods nearly coincide, so we only indicate whether a method is below the sequential (${>}100\%$) or above the parallel (${<}0\%$). Per-dataset and learned-rotation results in \cref{app:additional-results}.}
\label{tab:main}
\vspace{-0.2cm}
{\scriptsize\setlength{\tabcolsep}{2.2pt}\begin{tabular}{@{}lll ccc ccc ccc ccc@{}}
\toprule
& & Quantization objective & \multicolumn{3}{c}{Llama\,1B} & \multicolumn{3}{c}{Llama\,3B} & \multicolumn{3}{c}{Llama\,8B} & \multicolumn{3}{c}{Llama\,13B} \\
\cmidrule(lr){4-6}\cmidrule(lr){7-9}\cmidrule(lr){10-12}\cmidrule(lr){13-15}
& & & Acc\,$\uparrow$ & Avg\,$\downarrow$ & Gap\,$\uparrow$ & Acc\,$\uparrow$ & Avg\,$\downarrow$ & Gap\,$\uparrow$ & Acc\,$\uparrow$ & Avg\,$\downarrow$ & Gap\,$\uparrow$ & Acc\,$\uparrow$ & Avg\,$\downarrow$ & Gap\,$\uparrow$ \\
\midrule
FP16 & & & 54.8 & 11.33 & -- & 62.6 & 9.17 & -- & 67.1 & 7.57 & -- & 66.5 & 5.73 & -- \\
\midrule
\multirow{4}{*}{\rotatebox[origin=c]{90}{\scriptsize \shortstack{\textbf{3-bit}\\\textbf{GPTQ}}}} & \multirow{2}{*}{\rotatebox[origin=c]{90}{\scriptsize Par.}} & symmetric correction & 49.4 & 17.07 & $0\%$ & 58.5 & 12.11 & $0\%$ & 63.9 & 9.94 & $0\%$ & 64.2 & 6.40 & $0\%$ \\
 & & Denoising reg.\ (\textbf{ours}) & 49.3 & 16.51 & $\mathbf{31\%}$ & 57.1 & 11.97 & $\mathbf{21\%}$ & 64.9 & 9.81 & $\mathbf{30\%}$ & 64.5 & 6.37 & $\mathbf{18\%}$ \\
\cmidrule{2-15}
 & \multirow{2}{*}{\rotatebox[origin=c]{90}{\scriptsize Seq.}} & symmetric correction & 49.8 & 16.53 & $31\%$ & 59.0 & 11.99 & $17\%$ & 63.5 & 9.81 & $29\%$ & 63.7 & 6.38 & $13\%$ \\
 & & asymmetric correction & 50.3 & 15.29 & $100\%$ & 58.4 & 11.41 & $100\%$ & 63.8 & 9.51 & $100\%$ & 64.5 & 6.21 & $100\%$ \\
\midrule[\heavyrulewidth]
\multirow{4}{*}{\rotatebox[origin=c]{90}{\scriptsize \shortstack{\textbf{2-bit}\\\textbf{QuIP\#}}}} & \multirow{2}{*}{\rotatebox[origin=c]{90}{\scriptsize Par.}} & symmetric correction & 44.3 & 27.05 & $0\%$ & 52.7 & 16.62 & $0\%$ & 57.7 & 13.76 & $0\%$ & 61.0 & 7.28 & $0\%$ \\
 & & Denoising reg.\ (\textbf{ours}) & 45.1 & 25.51 & $\mathbf{38\%}$ & 53.9 & 15.57 & $\mathbf{86\%}$ & 58.8 & 12.70 & $\mathbf{119\%}$ & 61.0 & 7.19 & $\mathbf{{>}100\%}$ \\
\cmidrule{2-15}
 & \multirow{2}{*}{\rotatebox[origin=c]{90}{\scriptsize Seq.}} & symmetric correction & 43.8 & 26.36 & $17\%$ & 52.6 & 16.64 & $-1\%$ & 58.0 & 13.94 & $-20\%$ & 60.6 & 7.48 & ${<}0\%$ \\
 & & asymmetric correction & 45.1 & 23.04 & $100\%$ & 51.3 & 15.40 & $100\%$ & 56.2 & 12.87 & $100\%$ & 59.3 & 7.28 & $100\%$ \\
\bottomrule
\end{tabular}
}
\end{table}%

\vspace{-0.8cm}
\subsection{Activation denoising improves the quality of parallel quantization}
\vspace{-0.1cm}

\Cref{tab:main} reports our main results. At $3$ bits, parallel denoising closes $18$--$31\%$ of the perplexity gap across all four scales, recovering a significant fraction of the sequential benefit in a single parallel pass. At $2$ bits it closes $38\%$ and $86\%$ of the gap on Llama-3.2-1B and 3B, and on Llama-3-8B and Llama-2-13B it even outperforms the sequential method. We see two plausible reasons why our denoising outperforms the sequential method at 2 bits. First, QuIP\#'s randomized incoherence transform may already spread errors well enough that sequential correction helps less: at $3$B--$13$B it recovers only $0$--$16\%$ of the perplexity lost by parallel QuIP\#, compared to $18$--$28\%$ for GPTQ with QuaRot rotations, whereas our denoising recovers a similar or larger fraction under QuIP\#. Second, the sequential method is the asymmetric correction of \citet{li2025gptaq} applied to QuIP\#. While it is a state-of-the-art baseline, it is not necessarily optimal: it is a special case ($\lambda{=}1$) of our objective with the sequential error (\cref{sec:relation}), which fixes how much upstream error the layer is assumed to receive. Our framework is more flexible and treats $\lambda$ as a free parameter, and tuning it improves on the asymmetric correction even with the sequential error (\cref{app:sweep-range}). Lastly, on zero-shot accuracy, most differences are within about one point, so we view our denoising as comparable to parallel and sequential quantization on accuracy; only at $2$ bits does the sequential asymmetric correction fall behind, which our denoising avoids.

\label{sec:proxy-noise}
\begin{wrapfigure}{r}{0.4\textwidth}
  \vspace{-1.1cm}
  \centering
  %% Creator: Matplotlib, PGF backend
%%
%% To include the figure in your LaTeX document, write
%%   \input{<filename>.pgf}
%%
%% Make sure the required packages are loaded in your preamble
%%   \usepackage{pgf}
%%
%% Also ensure that all the required font packages are loaded; for instance,
%% the lmodern package is sometimes necessary when using math font.
%%   \usepackage{lmodern}
%%
%% Figures using additional raster images can only be included by \input if
%% they are in the same directory as the main LaTeX file. For loading figures
%% from other directories you can use the `import` package
%%   \usepackage{import}
%%
%% and then include the figures with
%%   \import{<path to file>}{<filename>.pgf}
%%
%% Matplotlib used the following preamble
%%   
%%   \usepackage[utf8]{inputenc}
%%   \usepackage[T1]{fontenc}
%%   \usepackage{times}
%%   \usepackage{amsmath}
%%   \newcommand*{\mat}[1]{\boldsymbol{#1}}
%%   
%%   \makeatletter\@ifpackageloaded{underscore}{}{\usepackage[strings]{underscore}}\makeatother
%%
\begingroup%
\makeatletter%
\begin{pgfpicture}%
\pgfpathrectangle{\pgfpointorigin}{\pgfqpoint{2.075550in}{0.500000in}}%
\pgfusepath{use as bounding box, clip}%
\begin{pgfscope}%
\pgfsetbuttcap%
\pgfsetmiterjoin%
\definecolor{currentfill}{rgb}{1.000000,1.000000,1.000000}%
\pgfsetfillcolor{currentfill}%
\pgfsetlinewidth{0.000000pt}%
\definecolor{currentstroke}{rgb}{1.000000,1.000000,1.000000}%
\pgfsetstrokecolor{currentstroke}%
\pgfsetdash{}{0pt}%
\pgfpathmoveto{\pgfqpoint{0.000000in}{0.000000in}}%
\pgfpathlineto{\pgfqpoint{2.075550in}{0.000000in}}%
\pgfpathlineto{\pgfqpoint{2.075550in}{0.500000in}}%
\pgfpathlineto{\pgfqpoint{0.000000in}{0.500000in}}%
\pgfpathlineto{\pgfqpoint{0.000000in}{0.000000in}}%
\pgfpathclose%
\pgfusepath{fill}%
\end{pgfscope}%
\begin{pgfscope}%
\pgfpathrectangle{\pgfqpoint{0.000000in}{0.000000in}}{\pgfqpoint{2.075550in}{0.500000in}}%
\pgfusepath{clip}%
\pgfsetbuttcap%
\pgfsetroundjoin%
\definecolor{currentfill}{rgb}{0.121569,0.466667,0.705882}%
\pgfsetfillcolor{currentfill}%
\pgfsetfillopacity{0.000000}%
\pgfsetlinewidth{1.003750pt}%
\definecolor{currentstroke}{rgb}{0.121569,0.466667,0.705882}%
\pgfsetstrokecolor{currentstroke}%
\pgfsetstrokeopacity{0.000000}%
\pgfsetdash{}{0pt}%
\pgfsys@defobject{currentmarker}{\pgfqpoint{-0.006944in}{-0.006944in}}{\pgfqpoint{0.006944in}{0.006944in}}{%
\pgfpathmoveto{\pgfqpoint{0.000000in}{-0.006944in}}%
\pgfpathcurveto{\pgfqpoint{0.001842in}{-0.006944in}}{\pgfqpoint{0.003608in}{-0.006213in}}{\pgfqpoint{0.004910in}{-0.004910in}}%
\pgfpathcurveto{\pgfqpoint{0.006213in}{-0.003608in}}{\pgfqpoint{0.006944in}{-0.001842in}}{\pgfqpoint{0.006944in}{0.000000in}}%
\pgfpathcurveto{\pgfqpoint{0.006944in}{0.001842in}}{\pgfqpoint{0.006213in}{0.003608in}}{\pgfqpoint{0.004910in}{0.004910in}}%
\pgfpathcurveto{\pgfqpoint{0.003608in}{0.006213in}}{\pgfqpoint{0.001842in}{0.006944in}}{\pgfqpoint{0.000000in}{0.006944in}}%
\pgfpathcurveto{\pgfqpoint{-0.001842in}{0.006944in}}{\pgfqpoint{-0.003608in}{0.006213in}}{\pgfqpoint{-0.004910in}{0.004910in}}%
\pgfpathcurveto{\pgfqpoint{-0.006213in}{0.003608in}}{\pgfqpoint{-0.006944in}{0.001842in}}{\pgfqpoint{-0.006944in}{0.000000in}}%
\pgfpathcurveto{\pgfqpoint{-0.006944in}{-0.001842in}}{\pgfqpoint{-0.006213in}{-0.003608in}}{\pgfqpoint{-0.004910in}{-0.004910in}}%
\pgfpathcurveto{\pgfqpoint{-0.003608in}{-0.006213in}}{\pgfqpoint{-0.001842in}{-0.006944in}}{\pgfqpoint{0.000000in}{-0.006944in}}%
\pgfpathlineto{\pgfqpoint{0.000000in}{-0.006944in}}%
\pgfpathclose%
\pgfusepath{stroke,fill}%
}%
\begin{pgfscope}%
\pgfsys@transformshift{0.000000in}{0.250000in}%
\pgfsys@useobject{currentmarker}{}%
\end{pgfscope}%
\begin{pgfscope}%
\pgfsys@transformshift{2.075550in}{0.250000in}%
\pgfsys@useobject{currentmarker}{}%
\end{pgfscope}%
\end{pgfscope}%
\begin{pgfscope}%
\pgfsetbuttcap%
\pgfsetmiterjoin%
\definecolor{currentfill}{rgb}{0.287302,0.757178,0.678992}%
\pgfsetfillcolor{currentfill}%
\pgfsetlinewidth{1.003750pt}%
\definecolor{currentstroke}{rgb}{0.287302,0.757178,0.678992}%
\pgfsetstrokecolor{currentstroke}%
\pgfsetdash{}{0pt}%
\pgfpathmoveto{\pgfqpoint{0.496151in}{0.339561in}}%
\pgfpathlineto{\pgfqpoint{0.596151in}{0.339561in}}%
\pgfpathlineto{\pgfqpoint{0.596151in}{0.397895in}}%
\pgfpathlineto{\pgfqpoint{0.496151in}{0.397895in}}%
\pgfpathlineto{\pgfqpoint{0.496151in}{0.339561in}}%
\pgfpathclose%
\pgfusepath{stroke,fill}%
\end{pgfscope}%
\begin{pgfscope}%
\definecolor{textcolor}{rgb}{0.150000,0.150000,0.150000}%
\pgfsetstrokecolor{textcolor}%
\pgfsetfillcolor{textcolor}%
\pgftext[x=0.629484in,y=0.339561in,left,base]{\color{textcolor}\rmfamily\fontsize{6.000000}{7.200000}\selectfont Parallel quantization (QuIP\#, 8 GPUs)}%
\end{pgfscope}%
\begin{pgfscope}%
\pgfsetbuttcap%
\pgfsetmiterjoin%
\definecolor{currentfill}{rgb}{0.207229,0.477632,0.637116}%
\pgfsetfillcolor{currentfill}%
\pgfsetlinewidth{1.003750pt}%
\definecolor{currentstroke}{rgb}{0.207229,0.477632,0.637116}%
\pgfsetstrokecolor{currentstroke}%
\pgfsetdash{}{0pt}%
\pgfpathmoveto{\pgfqpoint{0.496151in}{0.229854in}}%
\pgfpathlineto{\pgfqpoint{0.596151in}{0.229854in}}%
\pgfpathlineto{\pgfqpoint{0.596151in}{0.288187in}}%
\pgfpathlineto{\pgfqpoint{0.496151in}{0.288187in}}%
\pgfpathlineto{\pgfqpoint{0.496151in}{0.229854in}}%
\pgfpathclose%
\pgfusepath{stroke,fill}%
\end{pgfscope}%
\begin{pgfscope}%
\definecolor{textcolor}{rgb}{0.150000,0.150000,0.150000}%
\pgfsetstrokecolor{textcolor}%
\pgfsetfillcolor{textcolor}%
\pgftext[x=0.629484in,y=0.229854in,left,base]{\color{textcolor}\rmfamily\fontsize{6.000000}{7.200000}\selectfont \textbf{Our parallel denoising} (QuIP\#, 8 GPUs)}%
\end{pgfscope}%
\begin{pgfscope}%
\pgfsetbuttcap%
\pgfsetmiterjoin%
\definecolor{currentfill}{rgb}{0.242862,0.203456,0.412848}%
\pgfsetfillcolor{currentfill}%
\pgfsetlinewidth{1.003750pt}%
\definecolor{currentstroke}{rgb}{0.242862,0.203456,0.412848}%
\pgfsetstrokecolor{currentstroke}%
\pgfsetdash{}{0pt}%
\pgfpathmoveto{\pgfqpoint{0.496151in}{0.120146in}}%
\pgfpathlineto{\pgfqpoint{0.596151in}{0.120146in}}%
\pgfpathlineto{\pgfqpoint{0.596151in}{0.178480in}}%
\pgfpathlineto{\pgfqpoint{0.496151in}{0.178480in}}%
\pgfpathlineto{\pgfqpoint{0.496151in}{0.120146in}}%
\pgfpathclose%
\pgfusepath{stroke,fill}%
\end{pgfscope}%
\begin{pgfscope}%
\definecolor{textcolor}{rgb}{0.150000,0.150000,0.150000}%
\pgfsetstrokecolor{textcolor}%
\pgfsetfillcolor{textcolor}%
\pgftext[x=0.629484in,y=0.120146in,left,base]{\color{textcolor}\rmfamily\fontsize{6.000000}{7.200000}\selectfont Sequential asymmetric (QuIP\#, 1 GPU)}%
\end{pgfscope}%
\end{pgfpicture}%
\makeatother%
\endgroup%
\\[-14pt]
  \input{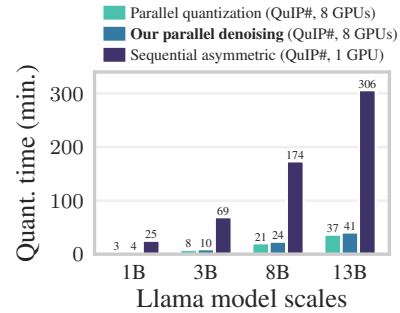}
  \vspace{-0.35cm}
  \caption{Projected multi-GPU quantization time (A100, QuIP\#).}
  \vspace{-0.35cm}
  \label{fig:quant-cost}
\end{wrapfigure}
Our denoising achieves the quality improvements while remaining fully parallel, which is significantly more efficient at scale (\autoref{fig:quant-cost}). Beyond the clean calibration forward pass that every method requires, our denoising needs only one additional forward pass on noised inputs. This noisy pass is slightly slower than the clean one (about $1.3\times$), since it additionally samples noise and accumulates noise and cross-covariance statistics. Yet for QuIP\#, the cost of both forward passes remains negligible compared to the rounding that solves the underlying weight-optimization problem. On Llama-2-13B the two passes take about $96$ seconds combined, under $1\%$ of the quantization time. 
Given the collected statistics, each layer's pretransform and rounding are independent and can be distributed across GPUs. Using measured single-GPU times for QuIP\#, we project a $7.1$--$7.5\times$ quantization time reduction on eight A100s relative to the single-GPU sequential baseline, assuming balanced work allocation and negligible communication overhead (cf.~\Cref{app:runtimes}).

\vspace{-0.15cm}
\subsection{Parallel denoising composes effectively with orthogonal LLM rotations}
\vspace{-0.05cm}

\textbf{Robustness regularization complements rotations by lowering the weight norm.}
\label{sec:wiener-shrink}
Orthogonal rotations reduce weight outliers by redistributing magnitude across channels, but preserve the Frobenius norm, $\|\mathbf{W}\mathbf{R}\|_F=\|\mathbf{W}\|_F$, and the singular values of $\mathbf{W}$. In contrast, by \autoref{thm:inject-lipschitz}, local noise injection and the proximity ridge induce weighted column-wise $L_2$ penalties toward down-scaled targets. The resulting filter $\mathbf{F}$ changes each layer's mapping and can reduce $\|\mathbf{W}\mathbf{F}\|_F$. We demonstrate this on the whole-matrix Frobenius norm in \autoref{fig:wiener-shrink} (left), which plots the Frobenius-norm ratio $\|\mathbf{W}\mathbf{F}\|_F/\|\mathbf{W}\|_F$ across depth and module type for Llama-3.2-1B, comparing denoised and original weights before rounding. The ratio is below $1$ for all $112$ linears, with the strongest reduction in the attention output projection. This verifies that denoising reduces a quantity that orthogonal rotations cannot change, rather than merely redistributing weight magnitude. The following ablation evaluates the quantization improvements when both are combined.

\begin{wrapfigure}[22]{r}{0.5\textwidth}
  \vspace{-0.6cm}
  \centering
  %% Creator: Matplotlib, PGF backend
%%
%% To include the figure in your LaTeX document, write
%%   \input{<filename>.pgf}
%%
%% Make sure the required packages are loaded in your preamble
%%   \usepackage{pgf}
%%
%% Also ensure that all the required font packages are loaded; for instance,
%% the lmodern package is sometimes necessary when using math font.
%%   \usepackage{lmodern}
%%
%% Figures using additional raster images can only be included by \input if
%% they are in the same directory as the main LaTeX file. For loading figures
%% from other directories you can use the `import` package
%%   \usepackage{import}
%%
%% and then include the figures with
%%   \import{<path to file>}{<filename>.pgf}
%%
%% Matplotlib used the following preamble
%%   
%%   \usepackage[utf8]{inputenc}
%%   \usepackage[T1]{fontenc}
%%   \usepackage{times}
%%   \usepackage{amsmath}
%%   \newcommand*{\mat}[1]{\boldsymbol{#1}}
%%   
%%   \makeatletter\@ifpackageloaded{underscore}{}{\usepackage[strings]{underscore}}\makeatother
%%
\begingroup%
\makeatletter%
\begin{pgfpicture}%
\pgfpathrectangle{\pgfpointorigin}{\pgfqpoint{2.784378in}{0.479123in}}%
\pgfusepath{use as bounding box, clip}%
\begin{pgfscope}%
\pgfsetbuttcap%
\pgfsetmiterjoin%
\definecolor{currentfill}{rgb}{1.000000,1.000000,1.000000}%
\pgfsetfillcolor{currentfill}%
\pgfsetlinewidth{0.000000pt}%
\definecolor{currentstroke}{rgb}{1.000000,1.000000,1.000000}%
\pgfsetstrokecolor{currentstroke}%
\pgfsetdash{}{0pt}%
\pgfpathmoveto{\pgfqpoint{0.000000in}{0.000000in}}%
\pgfpathlineto{\pgfqpoint{2.784378in}{0.000000in}}%
\pgfpathlineto{\pgfqpoint{2.784378in}{0.479123in}}%
\pgfpathlineto{\pgfqpoint{0.000000in}{0.479123in}}%
\pgfpathlineto{\pgfqpoint{0.000000in}{0.000000in}}%
\pgfpathclose%
\pgfusepath{fill}%
\end{pgfscope}%
\begin{pgfscope}%
\pgfpathrectangle{\pgfqpoint{0.050000in}{0.104561in}}{\pgfqpoint{2.684378in}{0.270000in}}%
\pgfusepath{clip}%
\pgfsetbuttcap%
\pgfsetroundjoin%
\definecolor{currentfill}{rgb}{0.121569,0.466667,0.705882}%
\pgfsetfillcolor{currentfill}%
\pgfsetfillopacity{0.000000}%
\pgfsetlinewidth{1.003750pt}%
\definecolor{currentstroke}{rgb}{0.121569,0.466667,0.705882}%
\pgfsetstrokecolor{currentstroke}%
\pgfsetstrokeopacity{0.000000}%
\pgfsetdash{}{0pt}%
\pgfsys@defobject{currentmarker}{\pgfqpoint{-0.006944in}{-0.006944in}}{\pgfqpoint{0.006944in}{0.006944in}}{%
\pgfpathmoveto{\pgfqpoint{0.000000in}{-0.006944in}}%
\pgfpathcurveto{\pgfqpoint{0.001842in}{-0.006944in}}{\pgfqpoint{0.003608in}{-0.006213in}}{\pgfqpoint{0.004910in}{-0.004910in}}%
\pgfpathcurveto{\pgfqpoint{0.006213in}{-0.003608in}}{\pgfqpoint{0.006944in}{-0.001842in}}{\pgfqpoint{0.006944in}{0.000000in}}%
\pgfpathcurveto{\pgfqpoint{0.006944in}{0.001842in}}{\pgfqpoint{0.006213in}{0.003608in}}{\pgfqpoint{0.004910in}{0.004910in}}%
\pgfpathcurveto{\pgfqpoint{0.003608in}{0.006213in}}{\pgfqpoint{0.001842in}{0.006944in}}{\pgfqpoint{0.000000in}{0.006944in}}%
\pgfpathcurveto{\pgfqpoint{-0.001842in}{0.006944in}}{\pgfqpoint{-0.003608in}{0.006213in}}{\pgfqpoint{-0.004910in}{0.004910in}}%
\pgfpathcurveto{\pgfqpoint{-0.006213in}{0.003608in}}{\pgfqpoint{-0.006944in}{0.001842in}}{\pgfqpoint{-0.006944in}{0.000000in}}%
\pgfpathcurveto{\pgfqpoint{-0.006944in}{-0.001842in}}{\pgfqpoint{-0.006213in}{-0.003608in}}{\pgfqpoint{-0.004910in}{-0.004910in}}%
\pgfpathcurveto{\pgfqpoint{-0.003608in}{-0.006213in}}{\pgfqpoint{-0.001842in}{-0.006944in}}{\pgfqpoint{0.000000in}{-0.006944in}}%
\pgfpathlineto{\pgfqpoint{0.000000in}{-0.006944in}}%
\pgfpathclose%
\pgfusepath{stroke,fill}%
}%
\begin{pgfscope}%
\pgfsys@transformshift{0.050000in}{0.239561in}%
\pgfsys@useobject{currentmarker}{}%
\end{pgfscope}%
\begin{pgfscope}%
\pgfsys@transformshift{2.734378in}{0.239561in}%
\pgfsys@useobject{currentmarker}{}%
\end{pgfscope}%
\end{pgfscope}%
\begin{pgfscope}%
\pgfsetbuttcap%
\pgfsetmiterjoin%
\definecolor{currentfill}{rgb}{1.000000,1.000000,1.000000}%
\pgfsetfillcolor{currentfill}%
\pgfsetlinewidth{1.003750pt}%
\definecolor{currentstroke}{rgb}{0.600000,0.600000,0.600000}%
\pgfsetstrokecolor{currentstroke}%
\pgfsetdash{}{0pt}%
\pgfpathmoveto{\pgfqpoint{0.527027in}{0.337456in}}%
\pgfpathlineto{\pgfqpoint{0.693694in}{0.337456in}}%
\pgfpathlineto{\pgfqpoint{0.693694in}{0.395790in}}%
\pgfpathlineto{\pgfqpoint{0.527027in}{0.395790in}}%
\pgfpathlineto{\pgfqpoint{0.527027in}{0.337456in}}%
\pgfpathclose%
\pgfusepath{stroke,fill}%
\end{pgfscope}%
\begin{pgfscope}%
\definecolor{textcolor}{rgb}{0.150000,0.150000,0.150000}%
\pgfsetstrokecolor{textcolor}%
\pgfsetfillcolor{textcolor}%
\pgftext[x=0.760361in,y=0.337456in,left,base]{\color{textcolor}\rmfamily\fontsize{6.000000}{7.200000}\selectfont no noise injection, random rotation (\(\displaystyle \star\): learned rotation)}%
\end{pgfscope}%
\begin{pgfscope}%
\pgfsetbuttcap%
\pgfsetmiterjoin%
\definecolor{currentfill}{rgb}{0.930000,0.930000,0.930000}%
\pgfsetfillcolor{currentfill}%
\pgfsetlinewidth{1.003750pt}%
\definecolor{currentstroke}{rgb}{0.600000,0.600000,0.600000}%
\pgfsetstrokecolor{currentstroke}%
\pgfsetdash{}{0pt}%
\pgfpathmoveto{\pgfqpoint{0.527027in}{0.219415in}}%
\pgfpathlineto{\pgfqpoint{0.693694in}{0.219415in}}%
\pgfpathlineto{\pgfqpoint{0.693694in}{0.277749in}}%
\pgfpathlineto{\pgfqpoint{0.527027in}{0.277749in}}%
\pgfpathlineto{\pgfqpoint{0.527027in}{0.219415in}}%
\pgfpathclose%
\pgfusepath{stroke,fill}%
\end{pgfscope}%
\begin{pgfscope}%
\definecolor{textcolor}{rgb}{0.150000,0.150000,0.150000}%
\pgfsetstrokecolor{textcolor}%
\pgfsetfillcolor{textcolor}%
\pgftext[x=0.760361in,y=0.219415in,left,base]{\color{textcolor}\rmfamily\fontsize{6.000000}{7.200000}\selectfont random noise injection, random Hadamard rotation}%
\end{pgfscope}%
\begin{pgfscope}%
\pgfsetbuttcap%
\pgfsetmiterjoin%
\definecolor{currentfill}{rgb}{0.850000,0.850000,0.850000}%
\pgfsetfillcolor{currentfill}%
\pgfsetlinewidth{1.003750pt}%
\definecolor{currentstroke}{rgb}{0.600000,0.600000,0.600000}%
\pgfsetstrokecolor{currentstroke}%
\pgfsetdash{}{0pt}%
\pgfpathmoveto{\pgfqpoint{0.527027in}{0.101374in}}%
\pgfpathlineto{\pgfqpoint{0.693694in}{0.101374in}}%
\pgfpathlineto{\pgfqpoint{0.693694in}{0.159708in}}%
\pgfpathlineto{\pgfqpoint{0.527027in}{0.159708in}}%
\pgfpathlineto{\pgfqpoint{0.527027in}{0.101374in}}%
\pgfpathclose%
\pgfusepath{stroke,fill}%
\end{pgfscope}%
\begin{pgfscope}%
\definecolor{textcolor}{rgb}{0.150000,0.150000,0.150000}%
\pgfsetstrokecolor{textcolor}%
\pgfsetfillcolor{textcolor}%
\pgftext[x=0.760361in,y=0.101374in,left,base]{\color{textcolor}\rmfamily\fontsize{6.000000}{7.200000}\selectfont random noise injection, fixed learned rotation}%
\end{pgfscope}%
\end{pgfpicture}%
\makeatother%
\endgroup%
\\[-5pt]
  \input{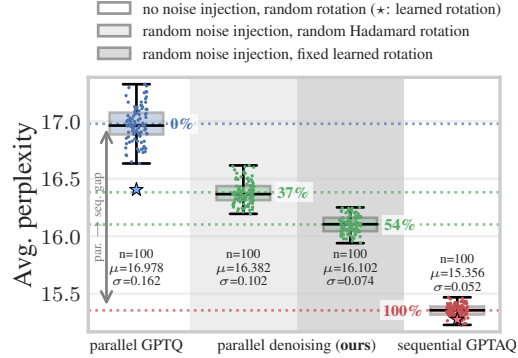}
  \vspace{-0.65cm}
  \caption{\textbf{Random rotation and noise ablation} (W3 Llama-3.2-1B). Avg.\ perplexity over $100$ draws: parallel GPTQ and sequential GPTAQ resample the Hadamard rotation (stars: learned-rotation), while for our denoising we resample noise either for random or learned rotations (gray backgrounds). Notably, our parallel denoising significantly reduces the parallel--sequential gap and is insensitive to the random noise draw.}
  \label{fig:randomness}
\end{wrapfigure}
\textbf{Parallel denoising significantly improves under model rotations.}
We perform an ablation study to understand the role of randomization in our method and the baselines. We consider two sources of randomization: the random Hadamard rotation used to make weights easier to quantize, and the injected Gaussian noise we use for regularization. We quantize Llama-3.2-1B at 3-bit precision $100$ times for each method, resampling rotations for the baselines and the noise injection for our method (\autoref{fig:randomness}).
We find that rotations matter greatly for parallel GPTQ: its perplexity varies widely across random Hadamard rotations (large standard deviation, $\sigma{=}0.16$) and a learned SpinQuant rotation improves a lot ($16.98{\to}16.41$). At the same time, we observe that sequential GPTAQ is quite robust under random rotations (smaller $\sigma{=}0.05$) and barely improved by learning them ($15.36\to15.28$). In contrast, our parallel denoising composes effectively with rotations, significantly reducing the~gap. 

Together, these results show that rotation and denoising are complementary: rotations reduce outliers while preserving the Frobenius norm, whereas denoising lowers this norm and improves quantization under both random and learned rotations.

\vspace{-0.3cm}
\subsection{Two efficient alternatives to full sequential quantization}
\label{sec:cheap-alternatives}
\vspace{-0.1cm}

Our denoising already recovers much of the gap in a single parallel pass. In this section we explore two approaches to further close the remaining gap without quantizing every layer sequentially.

\textbf{Trading quantization time: a short sequential prefix suffices.} First, we compute sequential GPTAQ for the first $K$ of $L$ blocks and continue quantizing the remaining blocks with parallel-GPTQ (\autoref{fig:splice-gapfrac}, middle). We find that the benefit of sequential quantization is strongly front-loaded across all models: the first quarter of the blocks already closes $40$--$60\%$ of the gap, whereas closing the last few percent requires sequential quantization to span most of the~depth. We see two complementary explanations for this front-loading: First, quantization error compounds through depth, and errors introduced at early blocks propagate through every downstream block, so correcting them early may have a disproportionately large effect. Second, later layers may contribute less to the output, consistent with evidence that language models do not use their depth efficiently and that deeper layers can often be removed with little effect \citep{csordas2025depth,men2025shortgpt}, leaving less error to correct. We leave attempts to disentangle these effects to future work.

\textbf{Trading inference cost: a few $4$-bit modules suffice.} Second, we explore promoting a small prefix of modules to $4$ bits at inference time, while keeping the rest of the network at $3$ bits. We find that the value projections $v_{\text{proj}}$ dominate: promoting only them to $4$ bits in a single parallel pass matches full sequential-GPTAQ quality at an average width of just $3.03$ bits, i.e.\ the whole gap for $0.03$ bits over W3. Because $v_{\text{proj}}$ carries so few parameters under grouped-query attention, this undercuts every finer group size (g$32$ needs a full $3.5$ bits for the same quality) and any other prefix in the comparison. When a small inference-time cost is acceptable, this is a cheap and effective way to close the remaining gap without any sequential quantization.

\begin{figure}[t]
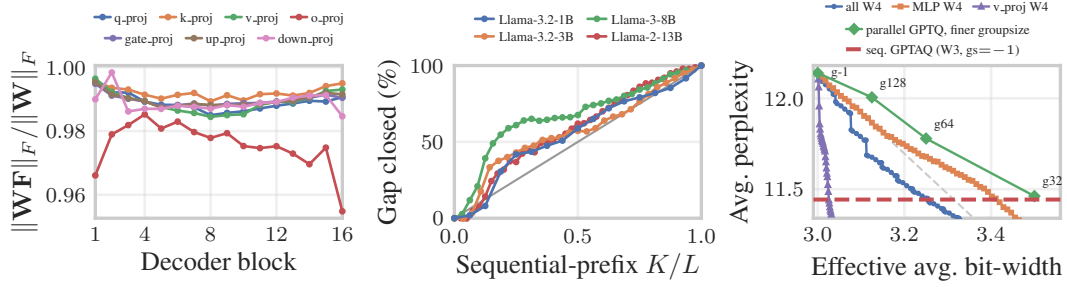

  \vspace{-0.4cm}
  \centering
  \begin{minipage}[b]{0.32\textwidth}
    \centering
    \input{figures/47_wiener_shrink/shrink_small_legend.pgf}\\[-14pt]
    \input{figures/47_wiener_shrink/shrink_small.pgf}
  \end{minipage}\hfill
  \begin{minipage}[b]{0.32\textwidth}
    \centering
    \input{figures/8_splice_scaling/gapfrac_legend.pgf}\\[-14pt]
    \input{figures/8_splice_scaling/gapfrac.pgf}
  \end{minipage}\hfill
  \begin{minipage}[b]{0.32\textwidth}
    \centering
    %% Creator: Matplotlib, PGF backend
%%
%% To include the figure in your LaTeX document, write
%%   \input{<filename>.pgf}
%%
%% Make sure the required packages are loaded in your preamble
%%   \usepackage{pgf}
%%
%% Also ensure that all the required font packages are loaded; for instance,
%% the lmodern package is sometimes necessary when using math font.
%%   \usepackage{lmodern}
%%
%% Figures using additional raster images can only be included by \input if
%% they are in the same directory as the main LaTeX file. For loading figures
%% from other directories you can use the `import` package
%%   \usepackage{import}
%%
%% and then include the figures with
%%   \import{<path to file>}{<filename>.pgf}
%%
%% Matplotlib used the following preamble
%%   
%%   \usepackage[utf8]{inputenc}
%%   \usepackage[T1]{fontenc}
%%   \usepackage{times}
%%   \usepackage{amsmath}
%%   \newcommand*{\mat}[1]{\boldsymbol{#1}}
%%   
%%   \makeatletter\@ifpackageloaded{underscore}{}{\usepackage[strings]{underscore}}\makeatother
%%
\begingroup%
\makeatletter%
\begin{pgfpicture}%
\pgfpathrectangle{\pgfpointorigin}{\pgfqpoint{1.826484in}{0.500000in}}%
\pgfusepath{use as bounding box, clip}%
\begin{pgfscope}%
\pgfsetbuttcap%
\pgfsetmiterjoin%
\definecolor{currentfill}{rgb}{1.000000,1.000000,1.000000}%
\pgfsetfillcolor{currentfill}%
\pgfsetlinewidth{0.000000pt}%
\definecolor{currentstroke}{rgb}{1.000000,1.000000,1.000000}%
\pgfsetstrokecolor{currentstroke}%
\pgfsetdash{}{0pt}%
\pgfpathmoveto{\pgfqpoint{0.000000in}{0.000000in}}%
\pgfpathlineto{\pgfqpoint{1.826484in}{0.000000in}}%
\pgfpathlineto{\pgfqpoint{1.826484in}{0.500000in}}%
\pgfpathlineto{\pgfqpoint{0.000000in}{0.500000in}}%
\pgfpathlineto{\pgfqpoint{0.000000in}{0.000000in}}%
\pgfpathclose%
\pgfusepath{fill}%
\end{pgfscope}%
\begin{pgfscope}%
\pgfsetroundcap%
\pgfsetroundjoin%
\pgfsetlinewidth{0.903375pt}%
\definecolor{currentstroke}{rgb}{0.298039,0.447059,0.690196}%
\pgfsetstrokecolor{currentstroke}%
\pgfsetdash{}{0pt}%
\pgfpathmoveto{\pgfqpoint{0.525470in}{0.367517in}}%
\pgfpathlineto{\pgfqpoint{0.567137in}{0.367517in}}%
\pgfpathlineto{\pgfqpoint{0.608804in}{0.367517in}}%
\pgfusepath{stroke}%
\end{pgfscope}%
\begin{pgfscope}%
\pgfsetbuttcap%
\pgfsetroundjoin%
\definecolor{currentfill}{rgb}{0.298039,0.447059,0.690196}%
\pgfsetfillcolor{currentfill}%
\pgfsetlinewidth{1.003750pt}%
\definecolor{currentstroke}{rgb}{0.298039,0.447059,0.690196}%
\pgfsetstrokecolor{currentstroke}%
\pgfsetdash{}{0pt}%
\pgfsys@defobject{currentmarker}{\pgfqpoint{-0.006944in}{-0.006944in}}{\pgfqpoint{0.006944in}{0.006944in}}{%
\pgfpathmoveto{\pgfqpoint{0.000000in}{-0.006944in}}%
\pgfpathcurveto{\pgfqpoint{0.001842in}{-0.006944in}}{\pgfqpoint{0.003608in}{-0.006213in}}{\pgfqpoint{0.004910in}{-0.004910in}}%
\pgfpathcurveto{\pgfqpoint{0.006213in}{-0.003608in}}{\pgfqpoint{0.006944in}{-0.001842in}}{\pgfqpoint{0.006944in}{0.000000in}}%
\pgfpathcurveto{\pgfqpoint{0.006944in}{0.001842in}}{\pgfqpoint{0.006213in}{0.003608in}}{\pgfqpoint{0.004910in}{0.004910in}}%
\pgfpathcurveto{\pgfqpoint{0.003608in}{0.006213in}}{\pgfqpoint{0.001842in}{0.006944in}}{\pgfqpoint{0.000000in}{0.006944in}}%
\pgfpathcurveto{\pgfqpoint{-0.001842in}{0.006944in}}{\pgfqpoint{-0.003608in}{0.006213in}}{\pgfqpoint{-0.004910in}{0.004910in}}%
\pgfpathcurveto{\pgfqpoint{-0.006213in}{0.003608in}}{\pgfqpoint{-0.006944in}{0.001842in}}{\pgfqpoint{-0.006944in}{0.000000in}}%
\pgfpathcurveto{\pgfqpoint{-0.006944in}{-0.001842in}}{\pgfqpoint{-0.006213in}{-0.003608in}}{\pgfqpoint{-0.004910in}{-0.004910in}}%
\pgfpathcurveto{\pgfqpoint{-0.003608in}{-0.006213in}}{\pgfqpoint{-0.001842in}{-0.006944in}}{\pgfqpoint{0.000000in}{-0.006944in}}%
\pgfpathlineto{\pgfqpoint{0.000000in}{-0.006944in}}%
\pgfpathclose%
\pgfusepath{stroke,fill}%
}%
\begin{pgfscope}%
\pgfsys@transformshift{0.567137in}{0.367517in}%
\pgfsys@useobject{currentmarker}{}%
\end{pgfscope}%
\end{pgfscope}%
\begin{pgfscope}%
\definecolor{textcolor}{rgb}{0.150000,0.150000,0.150000}%
\pgfsetstrokecolor{textcolor}%
\pgfsetfillcolor{textcolor}%
\pgftext[x=0.629637in,y=0.343212in,left,base]{\color{textcolor}\rmfamily\fontsize{5.000000}{6.000000}\selectfont all W4}%
\end{pgfscope}%
\begin{pgfscope}%
\pgfsetroundcap%
\pgfsetroundjoin%
\pgfsetlinewidth{0.903375pt}%
\definecolor{currentstroke}{rgb}{0.866667,0.517647,0.321569}%
\pgfsetstrokecolor{currentstroke}%
\pgfsetdash{}{0pt}%
\pgfpathmoveto{\pgfqpoint{0.872273in}{0.367517in}}%
\pgfpathlineto{\pgfqpoint{0.913940in}{0.367517in}}%
\pgfpathlineto{\pgfqpoint{0.955607in}{0.367517in}}%
\pgfusepath{stroke}%
\end{pgfscope}%
\begin{pgfscope}%
\pgfsetbuttcap%
\pgfsetmiterjoin%
\definecolor{currentfill}{rgb}{0.866667,0.517647,0.321569}%
\pgfsetfillcolor{currentfill}%
\pgfsetlinewidth{1.003750pt}%
\definecolor{currentstroke}{rgb}{0.866667,0.517647,0.321569}%
\pgfsetstrokecolor{currentstroke}%
\pgfsetdash{}{0pt}%
\pgfsys@defobject{currentmarker}{\pgfqpoint{-0.008333in}{-0.008333in}}{\pgfqpoint{0.008333in}{0.008333in}}{%
\pgfpathmoveto{\pgfqpoint{-0.008333in}{-0.008333in}}%
\pgfpathlineto{\pgfqpoint{0.008333in}{-0.008333in}}%
\pgfpathlineto{\pgfqpoint{0.008333in}{0.008333in}}%
\pgfpathlineto{\pgfqpoint{-0.008333in}{0.008333in}}%
\pgfpathlineto{\pgfqpoint{-0.008333in}{-0.008333in}}%
\pgfpathclose%
\pgfusepath{stroke,fill}%
}%
\begin{pgfscope}%
\pgfsys@transformshift{0.913940in}{0.367517in}%
\pgfsys@useobject{currentmarker}{}%
\end{pgfscope}%
\end{pgfscope}%
\begin{pgfscope}%
\definecolor{textcolor}{rgb}{0.150000,0.150000,0.150000}%
\pgfsetstrokecolor{textcolor}%
\pgfsetfillcolor{textcolor}%
\pgftext[x=0.976440in,y=0.343212in,left,base]{\color{textcolor}\rmfamily\fontsize{5.000000}{6.000000}\selectfont MLP W4}%
\end{pgfscope}%
\begin{pgfscope}%
\pgfsetroundcap%
\pgfsetroundjoin%
\pgfsetlinewidth{0.903375pt}%
\definecolor{currentstroke}{rgb}{0.505882,0.447059,0.701961}%
\pgfsetstrokecolor{currentstroke}%
\pgfsetdash{}{0pt}%
\pgfpathmoveto{\pgfqpoint{1.292411in}{0.367517in}}%
\pgfpathlineto{\pgfqpoint{1.334077in}{0.367517in}}%
\pgfpathlineto{\pgfqpoint{1.375744in}{0.367517in}}%
\pgfusepath{stroke}%
\end{pgfscope}%
\begin{pgfscope}%
\pgfsetbuttcap%
\pgfsetmiterjoin%
\definecolor{currentfill}{rgb}{0.505882,0.447059,0.701961}%
\pgfsetfillcolor{currentfill}%
\pgfsetlinewidth{1.003750pt}%
\definecolor{currentstroke}{rgb}{0.505882,0.447059,0.701961}%
\pgfsetstrokecolor{currentstroke}%
\pgfsetdash{}{0pt}%
\pgfsys@defobject{currentmarker}{\pgfqpoint{-0.009722in}{-0.009722in}}{\pgfqpoint{0.009722in}{0.009722in}}{%
\pgfpathmoveto{\pgfqpoint{0.000000in}{0.009722in}}%
\pgfpathlineto{\pgfqpoint{-0.009722in}{-0.009722in}}%
\pgfpathlineto{\pgfqpoint{0.009722in}{-0.009722in}}%
\pgfpathlineto{\pgfqpoint{0.000000in}{0.009722in}}%
\pgfpathclose%
\pgfusepath{stroke,fill}%
}%
\begin{pgfscope}%
\pgfsys@transformshift{1.334077in}{0.367517in}%
\pgfsys@useobject{currentmarker}{}%
\end{pgfscope}%
\end{pgfscope}%
\begin{pgfscope}%
\definecolor{textcolor}{rgb}{0.150000,0.150000,0.150000}%
\pgfsetstrokecolor{textcolor}%
\pgfsetfillcolor{textcolor}%
\pgftext[x=1.396577in,y=0.343212in,left,base]{\color{textcolor}\rmfamily\fontsize{5.000000}{6.000000}\selectfont v\_proj W4}%
\end{pgfscope}%
\begin{pgfscope}%
\pgfsetroundcap%
\pgfsetroundjoin%
\pgfsetlinewidth{0.903375pt}%
\definecolor{currentstroke}{rgb}{0.333333,0.658824,0.407843}%
\pgfsetstrokecolor{currentstroke}%
\pgfsetdash{}{0pt}%
\pgfpathmoveto{\pgfqpoint{0.630264in}{0.250173in}}%
\pgfpathlineto{\pgfqpoint{0.671930in}{0.250173in}}%
\pgfpathlineto{\pgfqpoint{0.713597in}{0.250173in}}%
\pgfusepath{stroke}%
\end{pgfscope}%
\begin{pgfscope}%
\pgfsetbuttcap%
\pgfsetmiterjoin%
\definecolor{currentfill}{rgb}{0.333333,0.658824,0.407843}%
\pgfsetfillcolor{currentfill}%
\pgfsetlinewidth{1.003750pt}%
\definecolor{currentstroke}{rgb}{0.333333,0.658824,0.407843}%
\pgfsetstrokecolor{currentstroke}%
\pgfsetdash{}{0pt}%
\pgfsys@defobject{currentmarker}{\pgfqpoint{-0.023570in}{-0.023570in}}{\pgfqpoint{0.023570in}{0.023570in}}{%
\pgfpathmoveto{\pgfqpoint{-0.000000in}{-0.023570in}}%
\pgfpathlineto{\pgfqpoint{0.023570in}{0.000000in}}%
\pgfpathlineto{\pgfqpoint{-0.000000in}{0.023570in}}%
\pgfpathlineto{\pgfqpoint{-0.023570in}{0.000000in}}%
\pgfpathlineto{\pgfqpoint{-0.000000in}{-0.023570in}}%
\pgfpathclose%
\pgfusepath{stroke,fill}%
}%
\begin{pgfscope}%
\pgfsys@transformshift{0.671930in}{0.250173in}%
\pgfsys@useobject{currentmarker}{}%
\end{pgfscope}%
\end{pgfscope}%
\begin{pgfscope}%
\definecolor{textcolor}{rgb}{0.150000,0.150000,0.150000}%
\pgfsetstrokecolor{textcolor}%
\pgfsetfillcolor{textcolor}%
\pgftext[x=0.734430in,y=0.225867in,left,base]{\color{textcolor}\rmfamily\fontsize{5.000000}{6.000000}\selectfont parallel GPTQ, finer groupsize}%
\end{pgfscope}%
\begin{pgfscope}%
\pgfsetbuttcap%
\pgfsetroundjoin%
\pgfsetlinewidth{1.505625pt}%
\definecolor{currentstroke}{rgb}{0.768627,0.305882,0.321569}%
\pgfsetstrokecolor{currentstroke}%
\pgfsetdash{{5.550000pt}{2.400000pt}}{0.000000pt}%
\pgfpathmoveto{\pgfqpoint{0.630264in}{0.144861in}}%
\pgfpathlineto{\pgfqpoint{0.671930in}{0.144861in}}%
\pgfpathlineto{\pgfqpoint{0.713597in}{0.144861in}}%
\pgfusepath{stroke}%
\end{pgfscope}%
\begin{pgfscope}%
\definecolor{textcolor}{rgb}{0.150000,0.150000,0.150000}%
\pgfsetstrokecolor{textcolor}%
\pgfsetfillcolor{textcolor}%
\pgftext[x=0.734430in,y=0.120556in,left,base]{\color{textcolor}\rmfamily\fontsize{5.000000}{6.000000}\selectfont seq.\ GPTAQ (W3, gs\(\displaystyle {=}{-}1\))}%
\end{pgfscope}%
\end{pgfpicture}%
\makeatother%
\endgroup%
\\[-14pt]
    \input{figures/10_w4prefix_module/w4prefix_zoom_3b.pgf}
  \end{minipage}
  \caption{\textbf{Denoising reduces weight norms, and two alternatives to full sequential quantization.} \emph{Left} (Frobenius-norm reduction on Llama-3.2-1B): ratio $\|\mathbf{W}\mathbf{F}\|_F/\|\mathbf{W}\|_F$ of denoised to original weights before rounding, across depth and module type. The ratio is below $1$ for all $112$ linears, most strongly in the attention output projection. \emph{Middle} (quantization-time trade-off): quantizing the first $K$ blocks sequentially and the remaining $L{-}K$ blocks in parallel, we find the sequential benefit is front-loaded: the first quarter of blocks closes $40$--$60\%$ of the gap, but the last few percent need a prefix spanning most of the depth. \emph{Right} (inference-time trade-off on Llama-3.2-3B): avg.\ perplexity vs.\ avg.\ effective bit-width, promoting a prefix of all/MLP-only/$v_{\text{proj}}$-only modules to $4$ bits (green diamonds: finer group-size; red dashed: seq.\ GPTAQ). We find promoting only $v_{\text{proj}}$ to 4 bits reaches sequential quality while raising the avg.\ bit-width by only ${\approx}0.03$ bits.}
  \label{fig:tradeoffs}
  \label{fig:wiener-shrink}
  \label{fig:splice-gapfrac}
  \label{fig:w4prefix-zoom}
\end{figure}

\vspace{-0.05cm}
\section{Discussion}
\vspace{-0.05cm}
\label{sec:discussion}
\label{sec:limitations}

\textbf{Future work.} We see several promising directions for future work: extending our activation denoising to activation and KV-cache quantization, combining it with stronger rounding schemes \citep{tseng2025modelpreserving,lee2026kronq}, and jointly learning the noise injections and rotations. Recovery fine-tuning \citep{tseng2024quipsharp,malinovskii2024pvtuning} is a further, orthogonal future direction.

\textbf{Closing the remaining gap.}
In a single parallel pass, our denoising already recovers a large fraction of the sequential benefit, though a gap to fully sequential quantization remains. We see the robustness perspective as a promising way to close it further, e.g.\ by using noise models that carry signal-correlated structure. We demonstrate this in \cref{sec:proxy-model}: the propagated error of an already parallel-quantized model provides a signal-correlated noise model, so an existing quantized model can be improved by denoising against its measured error, closing more of the gap without any sequential quantization. Moreover, because all rounding steps stay independent, the efficiency advantage over sequential quantization should grow at larger scale, which we already demonstrate up to $13$B-parameter models. Lastly, where a small trade-off is acceptable, the remaining gap can already be closed with a short sequential prefix or by promoting a few modules to $4$ bits (\cref{sec:cheap-alternatives}). 

\textbf{Robustness regularization.}
Our results suggest that much of the benefit of sequential quantization does not require the exact sequential upstream error: making each layer robust to propagated perturbations already recovers a significant part of the gap. While our regularization introduces additional hyperparameters that require tuning (noise scales~$\lambda,\rho$; injection magnitude~$\gamma$), we believe this can be overcome, since their optimal values are similar across models (\autoref{app:additional-results}). More broadly, our regularization is similar to the equivalence between training with input noise and Tikhonov regularization \citep{bishop1995training}, here it arises at quantization time. This corresponds to average-case robustness against random perturbations; its worst-case, adversarial counterpart \citep{szegedy2014intriguing,goodfellow2015explaining,madry2018towards} is a promising direction for future work.

\vspace{-0.05cm}
\section{Conclusion}
\vspace{-0.05cm}
\label{sec:conclusion}

We introduce activation denoising, which recovers much of the benefit of sequential post-training quantization while keeping quantization fully parallel. Our central take-away is that the upstream quantization errors, which sequential methods handle through expensive layer-by-layer re-calibration, can instead be addressed by regularizing every layer to be robust to noise, in a single parallel pass. Treating quantization errors as a robustness problem thus offers a principled and more efficient alternative to sequential quantization, providing a promising direction for efficient and accurate LLM compression at scale.

\clearpage
\bibliographystyle{conference}
\bibliography{references}

\newpage
\appendix

\crefalias{section}{appendix}
\crefalias{subsection}{appendix}
\crefalias{subsubsection}{appendix}
\raggedbottom
\providecommand{\corollaryautorefname}{Corollary}
\providecommand{\propositionautorefname}{Proposition}

\section*{Appendix overview}

This appendix augments our main findings with further experiments and experimental setup details, additional ablations, and full proofs and derivations. We organize the appendix as follows:
\begin{itemize}[itemsep=1pt,topsep=-3pt,leftmargin=0.5cm]
\item \textbf{\Cref{app:additional-results}} reports additional experimental results and ablations that extend our main findings.
\item \textbf{\Cref{app:setup}} describes the full experimental setup in detail including hyperparameter details.
\item \textbf{\Cref{sec:relation}} relates our denoising objective to the sequential corrections of GPTAQ and QEP.
\item \textbf{\Cref{app:proofs}} provides the full proofs and derivations of our theoretical results.
\end{itemize}

\section{Additional experimental results and ablations}
\label{app:additional-results}

In this section we provide additional experimental results and ablations extending our main findings.

\subsection{Noise-model mechanism ablation: local vs.\ propagated noise injection}
\label{app:mechanism}

\Cref{sec:random-noise} attributes the effect of the injected noise to two mechanisms: a local column-wise regularization from the injection $\boldsymbol{\xi}$ at each layer (effect~(i), formalized by \autoref{thm:inject-lipschitz}), and a depth-compounding regularization from propagating the injected noise downstream (effect~(ii)). \Cref{fig:mechanism-ablation} isolates the two by ablating the noise model on the $5{\times}5$ $(\rho,\delta)$ grid for Llama-3.2-1B at W3A16 (Gaussian injection, random Hadamard rotation, ridge $\alpha{=}0.01$), under a shared color scale.

Turning off all injected noise reduces our objective to parallel GPTQ: the noise moments vanish ($\mathbf{N}{=}\mathbf{C}{=}\mathbf{0}$), so $\rho$ and $\delta$ have no effect and the grid is flat at the parallel-GPTQ value (here $17.07$ at the fixed ridge). Injecting the noise locally without propagation, i.e., on top of the clean forward pass (so each layer sees only its own newly sampled $\boldsymbol{\xi}$), already closes $11\%$ of the parallel--sequential gap, confirming that effect~(i) works as a standalone regularizer. With propagation, so that the injected noise compounds across depth as in our full method, the gain roughly doubles to $28\%$. The optimum sits at the same cell $(\rho{=}0.45,\delta{=}0.0375)$ in both cases. Thus both effects contribute, with the propagated, depth-compounding term (ii) providing the majority of the benefit.

\begin{figure}[H]
\centering
\input{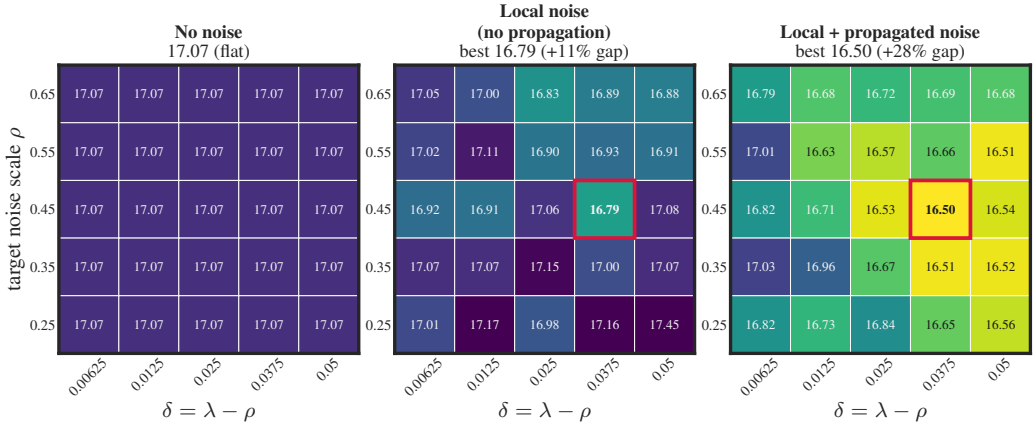}
\caption{\textbf{Noise-model mechanism ablation} (Llama-3.2-1B W3A16, Gaussian denoising, random Hadamard rotation, $\alpha{=}0.01$). Average perplexity over the $5{\times}5$ $(\rho,\delta)$ grid. Each title reports the best value and the \% of the parallel--sequential gap it closes. \emph{Left}: no injected noise, so the objective collapses to parallel GPTQ and $\rho,\delta$ have no effect. \emph{Middle}: local injection only (noise enters the moments but does not propagate), closing $11\%$ of the gap. \emph{Right}: the full method with propagated noise, closing $28\%$. Local injection alone already improves over parallel GPTQ, while propagation accounts for the majority of the improvement.}
\label{fig:mechanism-ablation}
\end{figure}

\newpage
\subsection{Noise-scale landscape: partial denoising and generalizing sequential quantization}
\label{app:sweep-range}

We explore the full range of both noise scales, $\rho,\lambda\in[0,1]$, on a $9{\times}9$ grid for Llama-3.2-1B W3A16:

\begin{figure}[H]
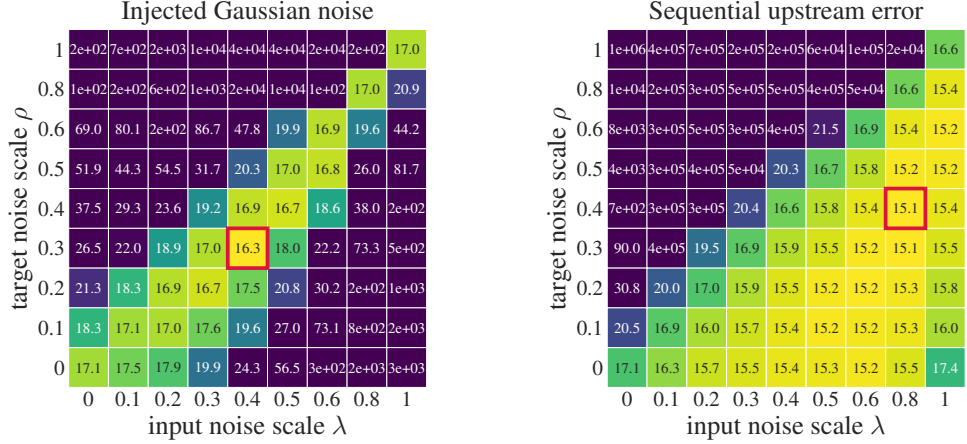

\centering
\begin{minipage}[b]{0.49\textwidth}
\centering
\input{figures/17_rho_lambda_landscape/538.pgf}
\end{minipage}\hfill
\begin{minipage}[b]{0.49\textwidth}
\centering
\input{figures/17_rho_lambda_landscape/401.pgf}
\end{minipage}
\caption{\textbf{Full $\rho\times\lambda$ sweep}, Llama-3.2-1B W3A16 ($9{\times}9$ grid over $[0,1]^2$). \emph{Left}: injected Gaussian noise (best avg.\ perplexity $16.31$ at $\lambda{=}0.4,\rho{=}0.3$). \emph{Right}: the real sequential upstream error (best $15.12$ at $\lambda{=}0.8,\rho{=}0.4$). In both, $\lambda$ is the input noise scale and $\rho$ the target scale; the low-perplexity region is below the diagonal ($\rho{<}\lambda$).}
\label{fig:rho-lambda-landscape}
\end{figure}

\textbf{Improving upon sequential asymmetric correction.} We compare the injected-noise proxy against the real sequential upstream error $\boldsymbol{\eta}_{\mathrm{seq}}$ (\autoref{fig:rho-lambda-landscape}, right). With this error, the asymmetric corrections of GPTAQ~\citep{li2025gptaq} and QEP~\citep{arai2025qep} are special cases of our objective: both fix the input noise scale to $\lambda{=}1$ and only tune the correction strength $1{-}\rho$ (\cref{sec:relation}), which restricts them to the rightmost column of the landscape. Our framework instead treats $\lambda$ as a free parameter. Along the column $\lambda{=}1$, the best configuration reaches an average perplexity of $15.2$ ($\rho\in\{0.5,0.6\}$), whereas the optimum of our sweep at $\lambda{=}0.8,\rho{=}0.4$ reaches $15.12$. Scaling down the upstream error ($\lambda{<}1$) thus improves upon the best asymmetric correction, even though both use the same sequential error. The flexibility of our objective therefore benefits not only parallel quantization with a noise proxy, but also sequential quantization.

\textbf{Why the optimum lies below the diagonal.}   We find the best configurations sit below the diagonal ($\rho{<}\lambda$). This shape follows from the objective in \autoref{eq:wiener-objective}. On the diagonal $\rho{=}\lambda$, the first factor of $\mathbf{F}$ in \autoref{eq:wiener-filter} equals $\mathbf{G}$, so $\mathbf{F}{=}\mathbf{I}$ and $\tilde{\mathbf{W}}{=}\mathbf{W}$: the noise only changes the rounding metric $\mathbf{G}$ and does not correct the weights. Off the diagonal, \autoref{thm:inject-lipschitz} shows that injected noise pulls each weight column toward $s_c\mathbf{W}_{:,c}$, where $s_c{<}1$ if and only if $\rho{<}\lambda$. Below the diagonal the objective therefore shrinks the weights, which makes the layer more robust to input perturbations. Above the diagonal ($\rho{>}\lambda$) we instead have $s_c{>}1$: the target carries more noise than the input, so the layer is asked to amplify the error and its weights grow. This is consistent with the poor performance above the diagonal and with the divergence with large $\rho$ and small $\lambda$.

\newpage
\subsection{Full result tables with hyperparameter sweep heatmaps}
\label{sec:proxy-model}

We provide the full tables of our W2/W3/W4 weight quantization results with detailed hyperparameter heatmaps of the sweeps used to optimize hyperparameters for each model and bit-width. 

\textbf{Proxy-model quantization.}
The random noise models a signal-independent input perturbation, providing a principled regularization approach to input activation denoising. Here we provide additional results for a second noise model: we propose to estimate upstream error from a model quantized in parallel, which is particularly efficient when a previously quantized model is already available and can be reused to further improve the quantization quality. 
Concretely, we quantize an LLM once in a single parallel pass (for example with parallel GPTQ at the deployed bit-width), and estimate the noise from the difference of input activations between the full-precision and proxy-quantized models. 
Instead of expensive sequential quantization, the cost of this proxy model corresponds to an additional parallel quantization pass to obtain the proxy quantized model, which can then be used for denoising the input activations in a second parallel quantization pass.
In practice one may quantize a dedicated proxy or just reuse an existing quantized model at no additional cost. In either case, the proxy model can provide a more accurate estimate of the deployed error than random noise injection alone. The following results therefore include both random noise injection and proxy-model denoising results, which we find to be working best when combined. The results suggest that any existing parallel-GPTQ quantized model can be used as a proxy to improve the quantization quality in a second parallel pass with our activation denoising.

The tables and figures are organized as follows:

\begin{itemize}[itemsep=1pt,topsep=-3pt,leftmargin=0.5cm]
\item \autoref{tab:w3-full}: W3A16 summary across all four models and both rotations.
\item \autoref{tab:w3-1b}: Llama-3.2-1B at W3A16 (GPTQ), sweeps in \autoref{fig:w3-1b-sweeps-random} and \autoref{fig:w3-1b-sweeps-learned}.
\item \autoref{tab:w3-3b}: Llama-3.2-3B at W3A16 (GPTQ), sweeps in \autoref{fig:w3-3b-sweeps-random} and \autoref{fig:w3-3b-sweeps-learned}.
\item \autoref{tab:w3-8b}: Llama-3-8B at W3A16 (GPTQ), sweeps in \autoref{fig:w3-8b-sweeps-random} and \autoref{fig:w3-8b-sweeps-learned}.
\item \autoref{tab:w3-13b}: Llama-2-13B at W3A16 (GPTQ), sweeps in \autoref{fig:w3-13b-sweeps-random} and \autoref{fig:w3-13b-sweeps-learned}.
\item \autoref{tab:w4-1b}: Llama-3.2-1B at W4A16 (GPTQ), sweeps in \autoref{fig:w4-1b-sweeps-random} and \autoref{fig:w4-1b-sweeps-learned}.
\item \autoref{tab:w2-1b}: Llama-3.2-1B at W2A16 (QuIP\#), sweeps in \autoref{fig:w2-1b-sweeps}.
\item \autoref{tab:w2-3b}: Llama-3.2-3B at W2A16 (QuIP\#), sweeps in \autoref{fig:w2-3b-sweeps}.
\item \autoref{tab:w2-8b}: Llama-3-8B at W2A16 (QuIP\#), sweeps in \autoref{fig:w2-8b-sweeps}.
\item \autoref{tab:w2-13b}: Llama-2-13B at W2A16 (QuIP\#), sweeps in \autoref{fig:w2-13b-sweeps}.
\end{itemize}

\subsubsection{W3A16 GPTQ summary across models}

\begin{table}[H]
\centering
\caption{\textbf{W3A16 GPTQ summary across Llama-3.2-1B/3B, Llama-3-8B, and Llama-2-13B.} Same quantization objectives as \autoref{tab:main}, here reported for the fixed random-Hadamard and the learned SpinQuant rotation side by side, including the GPTQ-W3 proxy noise.}
\label{tab:w3-full}
{\scriptsize\setlength{\tabcolsep}{3pt}\begin{tabular}{@{}lll ccc ccc ccc ccc@{}}
\toprule
& & Quantization objective & \multicolumn{3}{c}{Llama\,1B} & \multicolumn{3}{c}{Llama\,3B} & \multicolumn{3}{c}{Llama\,8B} & \multicolumn{3}{c}{Llama\,13B} \\
\cmidrule(lr){4-6}\cmidrule(lr){7-9}\cmidrule(lr){10-12}\cmidrule(lr){13-15}
& & & Acc\,$\uparrow$ & Avg\,$\downarrow$ & Gap\,$\uparrow$ & Acc\,$\uparrow$ & Avg\,$\downarrow$ & Gap\,$\uparrow$ & Acc\,$\uparrow$ & Avg\,$\downarrow$ & Gap\,$\uparrow$ & Acc\,$\uparrow$ & Avg\,$\downarrow$ & Gap\,$\uparrow$ \\
\midrule
& & FP16 & 54.8 & 11.33 & -- & 62.6 & 9.17 & -- & 67.1 & 7.57 & -- & 66.5 & 5.73 & -- \\
\midrule
\multirow{5}{*}{\rotatebox[origin=c]{90}{\scriptsize Random rot.}} & \multirow{3}{*}{\rotatebox[origin=c]{90}{\scriptsize Parallel}} & symmetric correction & 49.4 & 17.07 & $0\%$ & 58.5 & 12.11 & $0\%$ & 63.9 & 9.94 & $0\%$ & 64.2 & 6.40 & $0\%$ \\
 & & Denoising reg.\ (\textbf{ours}) & 49.3 & 16.51 & $31\%$ & 57.1 & 11.97 & $21\%$ & 64.9 & 9.81 & $30\%$ & 64.5 & 6.37 & $18\%$ \\
 & & \quad + GPTQ-W3 proxy & 50.2 & 16.30 & $\mathbf{43\%}$ & 59.1 & 11.88 & $\mathbf{34\%}$ & 64.6 & 9.78 & $\mathbf{36\%}$ & 64.2 & 6.33 & $\mathbf{36\%}$ \\
\cmidrule{2-15}
 & \multirow{2}{*}{\rotatebox[origin=c]{90}{\scriptsize Seq.}} & symmetric correction & 49.8 & 16.53 & $31\%$ & 59.0 & 11.99 & $17\%$ & 63.5 & 9.81 & $29\%$ & 63.7 & 6.38 & $13\%$ \\
 & & asymmetric correction & 50.3 & 15.29 & $100\%$ & 58.4 & 11.41 & $100\%$ & 63.8 & 9.51 & $100\%$ & 64.5 & 6.21 & $100\%$ \\
\midrule[\heavyrulewidth]
\multirow{5}{*}{\rotatebox[origin=c]{90}{\scriptsize Learned rot.}} & \multirow{3}{*}{\rotatebox[origin=c]{90}{\scriptsize Parallel}} & symmetric correction & 49.3 & 16.59 & $0\%$ & 58.3 & 11.88 & $0\%$ & 63.3 & 9.73 & $0\%$ & 64.8 & 6.34 & $0\%$ \\
 & & Denoising reg.\ (\textbf{ours}) & 49.1 & 16.02 & $41\%$ & 58.9 & 11.83 & $11\%$ & 63.7 & 9.75 & $-7\%$ & 65.1 & 6.31 & $22\%$ \\
 & & \quad + GPTQ-W3 proxy & 49.9 & 15.97 & $\mathbf{44\%}$ & 58.6 & 11.80 & $\mathbf{15\%}$ & 63.9 & 9.65 & $\mathbf{28\%}$ & 64.9 & 6.30 & $\mathbf{26\%}$ \\
\cmidrule{2-15}
 & \multirow{2}{*}{\rotatebox[origin=c]{90}{\scriptsize Seq.}} & symmetric correction & 49.0 & 16.23 & $26\%$ & 59.2 & 11.86 & $5\%$ & 63.1 & 9.72 & $3\%$ & 64.8 & 6.32 & $17\%$ \\
 & & asymmetric correction & 50.2 & 15.17 & $100\%$ & 58.0 & 11.37 & $100\%$ & 63.8 & 9.43 & $100\%$ & 64.9 & 6.20 & $100\%$ \\
\bottomrule
\end{tabular}
}
\end{table}

\subsubsection{Llama-3.2-1B (W3A16 GPTQ)}

\begin{table}[H]
\centering
\caption{\textbf{Llama-3.2-1B at W3A16 (GPTQ).} The upper block uses the fixed random-Hadamard rotation, the lower block the SpinQuant-learned rotation; rows marked \emph{(ours)} are our denoising objective. Gauss.\ denotes injected Gaussian noise, W3 corresponds to the additional GPTQ-W3 proxy noise, and W3+Gauss combines both.}
\label{tab:w3-1b}
\small
\setlength{\tabcolsep}{3pt}
\begin{tabular}{@{}lll cccc ccc cc@{}}
\toprule
& & Method & Bits & Grid & Setting & Acc\,$\uparrow$ & Wiki\,$\downarrow$ & C4\,$\downarrow$ & FW\,$\downarrow$ & Avg\,$\downarrow$ & Gap\,$\uparrow$ \\
\midrule
& & FP16 & 16 & -- & -- & 54.8 & 9.76 & 11.54 & 12.70 & 11.33 & -- \\
\midrule
\multirow{6}{*}{\rotatebox[origin=c]{90}{\scriptsize Random rot.}} & \multirow{4}{*}{\rotatebox[origin=c]{90}{\scriptsize Parallel}} & GPTQ & 3.00 & 25 & $\mathrm{damp}{=}0.002$ & 49.4 & 13.21 & 18.80 & 19.20 & 17.07 & $0\%$ \\
\cmidrule(l){3-12}
 & & Gauss.\,\textit{(ours)} & 3.00 & 25 & $\rho{=}0.55,\,\delta{=}0.05$ & 49.3 & 12.91 & 18.06 & 18.57 & 16.51 & $31\%$ \\
 & & W3\,\textit{(ours)} & 3.00 & 25 & $\rho{=}0.65,\,\delta{=}0.05$ & 50.1 & 13.01 & 18.04 & 18.53 & 16.53 & $30\%$ \\
 & & W3+Gauss\,\textit{(ours)} & 3.00 & 25 & $\rho{=}0.25,\,\delta{=}0.05$ & 50.2 & 12.76 & 17.79 & 18.36 & 16.30 & $43\%$ \\
\cmidrule{2-12}
 & \multirow{2}{*}{\rotatebox[origin=c]{90}{\scriptsize Seq.}} & GPTQ & 3.00 & 25 & $\mathrm{damp}{=}0.05$ & 49.8 & 12.99 & 18.03 & 18.56 & 16.53 & $31\%$ \\
 & & GPTAQ & 3.00 & 25 & $\beta{=}0.5833$ & 50.3 & 11.87 & 16.77 & 17.22 & 15.29 & $100\%$ \\
\midrule[\heavyrulewidth]
\multirow{6}{*}{\rotatebox[origin=c]{90}{\scriptsize Learned rot.}} & \multirow{4}{*}{\rotatebox[origin=c]{90}{\scriptsize Parallel}} & GPTQ & 3.00 & 25 & $\mathrm{damp}{=}0.02$ & 49.3 & 12.93 & 18.13 & 18.72 & 16.59 & $0\%$ \\
\cmidrule(l){3-12}
 & & Gauss.\,\textit{(ours)} & 3.00 & 25 & $\rho{=}0.55,\,\delta{=}0.05$ & 49.1 & 12.58 & 17.39 & 18.08 & 16.02 & $41\%$ \\
 & & W3\,\textit{(ours)} & 3.00 & 25 & $\rho{=}0.65,\,\delta{=}0.05$ & 49.0 & 12.74 & 17.84 & 18.28 & 16.29 & $22\%$ \\
 & & W3+Gauss\,\textit{(ours)} & 3.00 & 25 & $\rho{=}0.45,\,\delta{=}0.05$ & 49.9 & 12.57 & 17.32 & 18.02 & 15.97 & $44\%$ \\
\cmidrule{2-12}
 & \multirow{2}{*}{\rotatebox[origin=c]{90}{\scriptsize Seq.}} & GPTQ & 3.00 & 25 & $\mathrm{damp}{=}0.002$ & 49.0 & 12.76 & 17.65 & 18.28 & 16.23 & $26\%$ \\
 & & GPTAQ & 3.00 & 25 & $\beta{=}0.4167$ & 50.2 & 11.86 & 16.56 & 17.11 & 15.17 & $100\%$ \\
\bottomrule
\end{tabular}
\end{table}

\begin{figure}[h]
\centering
\input{figures/16_sweeps/w3_1b_random.pgf}
\caption{\textbf{Llama-3.2-1B W3A16 hyperparameter sweeps, fixed random-Hadamard rotation.} \emph{Top}: the one-dimensional baseline sweeps, avg.\ perplexity against the swept knob, with the selected optimum starred. \emph{Bottom}: our denoising objective over its $5{\times}5$ grid of the two swept knobs $\rho$ and $\delta$ (which sets the noise scale through $\lambda{=}\rho{+}\delta$); one heatmap per upstream-error proxy (Gaussian noise, the GPTQ-W3 proxy, and their combination), with the best (minimum-perplexity) cell boxed.}
\label{fig:w3-1b-sweeps-random}
\end{figure}

\begin{figure}[h]
\centering
\input{figures/16_sweeps/w3_1b_learned.pgf}
\caption{\textbf{Llama-3.2-1B W3A16 hyperparameter sweeps, learned SpinQuant rotation.} Panels as in Fig.~\ref{fig:w3-1b-sweeps-random}.}
\label{fig:w3-1b-sweeps-learned}
\end{figure}

\subsubsection{Llama-3.2-3B (W3A16 GPTQ)}

\begin{table}[H]
\centering
\caption{\textbf{Llama-3.2-3B at W3A16 (GPTQ).} Columns and rotation blocks as in \autoref{tab:w3-1b}. Gauss.\ denotes injected Gaussian noise, W3 corresponds to the additional GPTQ-W3 proxy noise, and W3+Gauss combines both.}
\label{tab:w3-3b}
\small
\setlength{\tabcolsep}{3pt}
\begin{tabular}{@{}lll cccc ccc cc@{}}
\toprule
& & Method & Bits & Grid & Setting & Acc\,$\uparrow$ & Wiki\,$\downarrow$ & C4\,$\downarrow$ & FW\,$\downarrow$ & Avg\,$\downarrow$ & Gap\,$\uparrow$ \\
\midrule
& & FP16 & 16 & -- & -- & 62.6 & 7.82 & 9.29 & 10.41 & 9.17 & -- \\
\midrule
\multirow{6}{*}{\rotatebox[origin=c]{90}{\scriptsize Random rot.}} & \multirow{4}{*}{\rotatebox[origin=c]{90}{\scriptsize Parallel}} & GPTQ & 3.00 & 25 & $\mathrm{damp}{=}0.001$ & 58.5 & 9.61 & 13.14 & 13.60 & 12.11 & $0\%$ \\
\cmidrule(l){3-12}
 & & Gauss.\,\textit{(ours)} & 3.00 & 25 & $\rho{=}0.35,\,\delta{=}0.0375$ & 57.1 & 9.47 & 12.96 & 13.47 & 11.97 & $21\%$ \\
 & & W3\,\textit{(ours)} & 3.00 & 25 & $\rho{=}0.35,\,\delta{=}0.05$ & 58.7 & 9.46 & 13.02 & 13.45 & 11.98 & $19\%$ \\
 & & W3+Gauss\,\textit{(ours)} & 3.00 & 25 & $\rho{=}0.35,\,\delta{=}0.05$ & 59.1 & 9.41 & 12.88 & 13.34 & 11.88 & $34\%$ \\
\cmidrule{2-12}
 & \multirow{2}{*}{\rotatebox[origin=c]{90}{\scriptsize Seq.}} & GPTQ & 3.00 & 25 & $\mathrm{damp}{=}0.002$ & 59.0 & 9.51 & 13.03 & 13.45 & 11.99 & $17\%$ \\
 & & GPTAQ & 3.00 & 25 & $\beta{=}0.4583$ & 58.4 & 8.94 & 12.44 & 12.85 & 11.41 & $100\%$ \\
\midrule[\heavyrulewidth]
\multirow{6}{*}{\rotatebox[origin=c]{90}{\scriptsize Learned rot.}} & \multirow{4}{*}{\rotatebox[origin=c]{90}{\scriptsize Parallel}} & GPTQ & 3.00 & 25 & $\mathrm{damp}{=}0.4$ & 58.3 & 9.41 & 12.86 & 13.38 & 11.88 & $0\%$ \\
\cmidrule(l){3-12}
 & & Gauss.\,\textit{(ours)} & 3.00 & 25 & $\rho{=}0.65,\,\delta{=}0.0375$ & 58.9 & 9.29 & 12.82 & 13.37 & 11.83 & $11\%$ \\
 & & W3\,\textit{(ours)} & 3.00 & 25 & $\rho{=}0.55,\,\delta{=}0.05$ & 58.7 & 9.34 & 12.89 & 13.41 & 11.88 & $0\%$ \\
 & & W3+Gauss\,\textit{(ours)} & 3.00 & 25 & $\rho{=}0.45,\,\delta{=}0.05$ & 58.6 & 9.27 & 12.81 & 13.33 & 11.80 & $15\%$ \\
\cmidrule{2-12}
 & \multirow{2}{*}{\rotatebox[origin=c]{90}{\scriptsize Seq.}} & GPTQ & 3.00 & 25 & $\mathrm{damp}{=}0.02$ & 59.2 & 9.34 & 12.86 & 13.37 & 11.86 & $5\%$ \\
 & & GPTAQ & 3.00 & 25 & $\beta{=}0.5417$ & 58.0 & 8.91 & 12.39 & 12.80 & 11.37 & $100\%$ \\
\bottomrule
\end{tabular}
\end{table}

\begin{figure}[H]
\centering
\input{figures/16_sweeps/w3_3b_random.pgf}
\caption{\textbf{Llama-3.2-3B W3A16 hyperparameter sweeps, fixed random-Hadamard rotation.} \emph{Top}: the one-dimensional baseline sweeps, avg.\ perplexity against the swept knob, with the selected optimum starred. \emph{Bottom}: our denoising objective over its $5{\times}5$ grid of the two swept knobs $\rho$ and $\delta$ (which sets the noise scale through $\lambda{=}\rho{+}\delta$); one heatmap per upstream-error proxy (Gaussian noise, the GPTQ-W3 proxy, and their combination), with the best (minimum-perplexity) cell boxed.}
\label{fig:w3-3b-sweeps-random}
\end{figure}

\begin{figure}[H]
\centering
\input{figures/16_sweeps/w3_3b_learned.pgf}
\caption{\textbf{Llama-3.2-3B W3A16 hyperparameter sweeps, learned SpinQuant rotation.} Panels as in Fig.~\ref{fig:w3-3b-sweeps-random}.}
\label{fig:w3-3b-sweeps-learned}
\end{figure}

\subsubsection{Llama-3-8B (W3A16 GPTQ)}

\begin{table}[H]
\centering
\caption{\textbf{Llama-3-8B at W3A16 (GPTQ).} Columns and rotation blocks as in \autoref{tab:w3-1b}. Gauss.\ denotes injected Gaussian noise, W3 corresponds to the additional GPTQ-W3 proxy noise, and W3+Gauss combines both.}
\label{tab:w3-8b}
\small
\setlength{\tabcolsep}{3pt}
\begin{tabular}{@{}lll cccc ccc cc@{}}
\toprule
& & Method & Bits & Grid & Setting & Acc\,$\uparrow$ & Wiki\,$\downarrow$ & C4\,$\downarrow$ & FW\,$\downarrow$ & Avg\,$\downarrow$ & Gap\,$\uparrow$ \\
\midrule
& & FP16 & 16 & -- & -- & 67.1 & 6.14 & 7.78 & 8.78 & 7.57 & -- \\
\midrule
\multirow{6}{*}{\rotatebox[origin=c]{90}{\scriptsize Random rot.}} & \multirow{4}{*}{\rotatebox[origin=c]{90}{\scriptsize Parallel}} & GPTQ & 3.00 & 25 & $\mathrm{damp}{=}0.005$ & 63.9 & 7.60 & 10.90 & 11.30 & 9.94 & $0\%$ \\
\cmidrule(l){3-12}
 & & Gauss.\,\textit{(ours)} & 3.00 & 25 & $\rho{=}0.55,\,\delta{=}0.0375$ & 64.9 & 7.50 & 10.75 & 11.17 & 9.81 & $30\%$ \\
 & & W3\,\textit{(ours)} & 3.00 & 25 & $\rho{=}0.65,\,\delta{=}0.05$ & 64.1 & 7.51 & 10.83 & 11.22 & 9.85 & $19\%$ \\
 & & W3+Gauss\,\textit{(ours)} & 3.00 & 25 & $\rho{=}0.35,\,\delta{=}0.025$ & 64.6 & 7.51 & 10.70 & 11.14 & 9.78 & $36\%$ \\
\cmidrule{2-12}
 & \multirow{2}{*}{\rotatebox[origin=c]{90}{\scriptsize Seq.}} & GPTQ & 3.00 & 25 & $\mathrm{damp}{=}0.02$ & 63.5 & 7.56 & 10.72 & 11.16 & 9.81 & $29\%$ \\
 & & GPTAQ & 3.00 & 25 & $\beta{=}0.375$ & 63.8 & 7.22 & 10.48 & 10.82 & 9.51 & $100\%$ \\
\midrule[\heavyrulewidth]
\multirow{6}{*}{\rotatebox[origin=c]{90}{\scriptsize Learned rot.}} & \multirow{4}{*}{\rotatebox[origin=c]{90}{\scriptsize Parallel}} & GPTQ & 3.00 & 25 & $\mathrm{damp}{=}0.2$ & 63.3 & 7.42 & 10.67 & 11.09 & 9.73 & $0\%$ \\
\cmidrule(l){3-12}
 & & Gauss.\,\textit{(ours)} & 3.00 & 25 & $\rho{=}0.55,\,\delta{=}0.05$ & 63.7 & 7.37 & 10.73 & 11.15 & 9.75 & $-7\%$ \\
 & & W3\,\textit{(ours)} & 3.00 & 25 & $\rho{=}0.45,\,\delta{=}0.025$ & 63.9 & 7.40 & 10.74 & 11.13 & 9.76 & $-10\%$ \\
 & & W3+Gauss\,\textit{(ours)} & 3.00 & 25 & $\rho{=}0.45,\,\delta{=}0.025$ & 63.9 & 7.37 & 10.56 & 11.01 & 9.65 & $28\%$ \\
\cmidrule{2-12}
 & \multirow{2}{*}{\rotatebox[origin=c]{90}{\scriptsize Seq.}} & GPTQ & 3.00 & 25 & $\mathrm{damp}{=}0.01$ & 63.1 & 7.39 & 10.68 & 11.09 & 9.72 & $3\%$ \\
 & & GPTAQ & 3.00 & 25 & $\beta{=}0.3333$ & 63.8 & 7.19 & 10.35 & 10.75 & 9.43 & $100\%$ \\
\bottomrule
\end{tabular}
\end{table}

\begin{figure}[H]
\centering
\input{figures/16_sweeps/w3_8b_random.pgf}
\caption{\textbf{Llama-3-8B W3A16 hyperparameter sweeps, fixed random-Hadamard rotation.} \emph{Top}: the one-dimensional baseline sweeps, avg.\ perplexity against the swept knob, with the selected optimum starred. \emph{Bottom}: our denoising objective over its $5{\times}5$ grid of the two swept knobs $\rho$ and $\delta$ (which sets the noise scale through $\lambda{=}\rho{+}\delta$); one heatmap per upstream-error proxy (Gaussian noise, the GPTQ-W3 proxy, and their combination), with the best (minimum-perplexity) cell boxed.}
\label{fig:w3-8b-sweeps-random}
\end{figure}

\begin{figure}[H]
\centering
\input{figures/16_sweeps/w3_8b_learned.pgf}
\caption{\textbf{Llama-3-8B W3A16 hyperparameter sweeps, learned SpinQuant rotation.} Panels as in Fig.~\ref{fig:w3-8b-sweeps-random}.}
\label{fig:w3-8b-sweeps-learned}
\end{figure}

\subsubsection{Llama-2-13B (W3A16 GPTQ)}

\begin{table}[H]
\centering
\caption{\textbf{Llama-2-13B at W3A16 (GPTQ).} Columns and rotation blocks as in \autoref{tab:w3-1b}. Gauss.\ denotes injected Gaussian noise, W3 corresponds to the additional GPTQ-W3 proxy noise, and W3+Gauss combines both.}
\label{tab:w3-13b}
\small
\setlength{\tabcolsep}{3pt}
\begin{tabular}{@{}lll cccc ccc cc@{}}
\toprule
& & Method & Bits & Grid & Setting & Acc\,$\uparrow$ & Wiki\,$\downarrow$ & C4\,$\downarrow$ & FW\,$\downarrow$ & Avg\,$\downarrow$ & Gap\,$\uparrow$ \\
\midrule
& & FP16 & 16 & -- & -- & 66.5 & 4.88 & 5.80 & 6.49 & 5.73 & -- \\
\midrule
\multirow{6}{*}{\rotatebox[origin=c]{90}{\scriptsize Random rot.}} & \multirow{4}{*}{\rotatebox[origin=c]{90}{\scriptsize Parallel}} & GPTQ & 3.00 & 25 & $\mathrm{damp}{=}0.005$ & 64.2 & 5.37 & 6.66 & 7.17 & 6.40 & $0\%$ \\
\cmidrule(l){3-12}
 & & Gauss.\,\textit{(ours)} & 3.00 & 25 & $\rho{=}0.35,\,\delta{=}0.0375$ & 64.5 & 5.31 & 6.64 & 7.15 & 6.37 & $18\%$ \\
 & & W3\,\textit{(ours)} & 3.00 & 25 & $\rho{=}0.55,\,\delta{=}0.05$ & 63.7 & 5.33 & 6.64 & 7.15 & 6.37 & $15\%$ \\
 & & W3+Gauss\,\textit{(ours)} & 3.00 & 25 & $\rho{=}0.35,\,\delta{=}0.0375$ & 64.2 & 5.29 & 6.60 & 7.11 & 6.33 & $36\%$ \\
\cmidrule{2-12}
 & \multirow{2}{*}{\rotatebox[origin=c]{90}{\scriptsize Seq.}} & GPTQ & 3.00 & 25 & $\mathrm{damp}{=}0.002$ & 63.7 & 5.35 & 6.64 & 7.15 & 6.38 & $13\%$ \\
 & & GPTAQ & 3.00 & 25 & $\beta{=}0.5833$ & 64.5 & 5.20 & 6.44 & 7.00 & 6.21 & $100\%$ \\
\midrule[\heavyrulewidth]
\multirow{6}{*}{\rotatebox[origin=c]{90}{\scriptsize Learned rot.}} & \multirow{4}{*}{\rotatebox[origin=c]{90}{\scriptsize Parallel}} & GPTQ & 3.00 & 25 & $\mathrm{damp}{=}0.4$ & 64.8 & 5.27 & 6.61 & 7.14 & 6.34 & $0\%$ \\
\cmidrule(l){3-12}
 & & Gauss.\,\textit{(ours)} & 3.00 & 25 & $\rho{=}0.35,\,\delta{=}0.025$ & 65.1 & 5.25 & 6.58 & 7.10 & 6.31 & $22\%$ \\
 & & W3\,\textit{(ours)} & 3.00 & 25 & $\rho{=}0.25,\,\delta{=}0.05$ & 64.2 & 5.26 & 6.58 & 7.10 & 6.31 & $20\%$ \\
 & & W3+Gauss\,\textit{(ours)} & 3.00 & 25 & $\rho{=}0.45,\,\delta{=}0.025$ & 64.9 & 5.24 & 6.57 & 7.10 & 6.30 & $26\%$ \\
\cmidrule{2-12}
 & \multirow{2}{*}{\rotatebox[origin=c]{90}{\scriptsize Seq.}} & GPTQ & 3.00 & 25 & $\mathrm{damp}{=}0.05$ & 64.8 & 5.27 & 6.57 & 7.11 & 6.32 & $17\%$ \\
 & & GPTAQ & 3.00 & 25 & $\beta{=}0.5$ & 64.9 & 5.18 & 6.43 & 6.98 & 6.20 & $100\%$ \\
\bottomrule
\end{tabular}
\end{table}

\begin{figure}[H]
\centering
\input{figures/16_sweeps/w3_13b_random.pgf}
\caption{\textbf{Llama-2-13B W3A16 hyperparameter sweeps, fixed random-Hadamard rotation.} \emph{Top}: the one-dimensional baseline sweeps, avg.\ perplexity against the swept knob, with the selected optimum starred. \emph{Bottom}: our denoising objective over its $5{\times}5$ grid of the two swept knobs $\rho$ and $\delta$ (which sets the noise scale through $\lambda{=}\rho{+}\delta$); one heatmap per upstream-error proxy (Gaussian noise, the GPTQ-W3 proxy, and their combination), with the best (minimum-perplexity) cell boxed.}
\label{fig:w3-13b-sweeps-random}
\end{figure}

\begin{figure}[H]
\centering
\input{figures/16_sweeps/w3_13b_learned.pgf}
\caption{\textbf{Llama-2-13B W3A16 hyperparameter sweeps, learned SpinQuant rotation.} Panels as in Fig.~\ref{fig:w3-13b-sweeps-random}.}
\label{fig:w3-13b-sweeps-learned}
\end{figure}

\subsubsection{Llama-3.2-1B (W4A16 GPTQ)}

\begin{table}[H]
\centering
\caption{\textbf{Llama-3.2-1B at W4A16 (GPTQ).} Columns and rotation blocks as in \autoref{tab:w3-1b}. Gauss.\ denotes injected Gaussian noise, W4 corresponds to the additional GPTQ-W4 proxy noise, and W4+Gauss combines both. At $4$ bits the parallel--sequential gap is small ($\approx 0.26$ perplexity).}
\label{tab:w4-1b}
\small
\setlength{\tabcolsep}{3pt}
\begin{tabular}{@{}lll cccc ccc cc@{}}
\toprule
& & Method & Bits & Grid & Setting & Acc\,$\uparrow$ & Wiki\,$\downarrow$ & C4\,$\downarrow$ & FW\,$\downarrow$ & Avg\,$\downarrow$ & Gap\,$\uparrow$ \\
\midrule
& & FP16 & 16 & -- & -- & 54.8 & 9.76 & 11.54 & 12.70 & 11.33 & -- \\
\midrule
\multirow{6}{*}{\rotatebox[origin=c]{90}{\scriptsize Random rot.}} & \multirow{4}{*}{\rotatebox[origin=c]{90}{\scriptsize Parallel}} & GPTQ & 4.00 & 25 & $\mathrm{damp}{=}0.02$ & 53.1 & 10.48 & 13.30 & 14.11 & 12.63 & $0\%$ \\
\cmidrule(l){3-12}
 & & Gauss.\,\textit{(ours)} & 4.00 & 25 & $\rho{=}0.55,\,\delta{=}0.05$ & 53.7 & 10.41 & 13.29 & 14.15 & 12.61 & $7\%$ \\
 & & W4\,\textit{(ours)} & 4.00 & 25 & $\rho{=}0.65,\,\delta{=}0.025$ & 52.8 & 10.45 & 13.28 & 14.10 & 12.61 & $8\%$ \\
 & & W4+Gauss\,\textit{(ours)} & 4.00 & 25 & $\rho{=}0.45,\,\delta{=}0.05$ & 54.0 & 10.41 & 13.28 & 14.13 & 12.61 & $9\%$ \\
\cmidrule{2-12}
 & \multirow{2}{*}{\rotatebox[origin=c]{90}{\scriptsize Seq.}} & GPTQ & 4.00 & 25 & $\mathrm{damp}{=}0.001$ & 53.4 & 10.47 & 13.27 & 14.09 & 12.61 & $8\%$ \\
 & & GPTAQ & 4.00 & 25 & $\beta{=}0.4583$ & 54.2 & 10.25 & 13.01 & 13.84 & 12.37 & $100\%$ \\
\midrule[\heavyrulewidth]
\multirow{6}{*}{\rotatebox[origin=c]{90}{\scriptsize Learned rot.}} & \multirow{4}{*}{\rotatebox[origin=c]{90}{\scriptsize Parallel}} & GPTQ & 4.00 & 25 & $\mathrm{damp}{=}0.25$ & 53.6 & 10.42 & 13.20 & 14.08 & 12.57 & $0\%$ \\
\cmidrule(l){3-12}
 & & Gauss.\,\textit{(ours)} & 4.00 & 25 & $\rho{=}0.55,\,\delta{=}0.0375$ & 52.9 & 10.36 & 13.24 & 14.09 & 12.56 & $2\%$ \\
 & & W4\,\textit{(ours)} & 4.00 & 25 & $\rho{=}0.65,\,\delta{=}0.05$ & 53.0 & 10.40 & 13.25 & 14.07 & 12.57 & $-4\%$ \\
 & & W4+Gauss\,\textit{(ours)} & 4.00 & 25 & $\rho{=}0.65,\,\delta{=}0.025$ & 53.0 & 10.35 & 13.18 & 14.05 & 12.53 & $21\%$ \\
\cmidrule{2-12}
 & \multirow{2}{*}{\rotatebox[origin=c]{90}{\scriptsize Seq.}} & GPTQ & 4.00 & 25 & $\mathrm{damp}{=}0.05$ & 53.4 & 10.41 & 13.20 & 14.06 & 12.56 & $5\%$ \\
 & & GPTAQ & 4.00 & 25 & $\beta{=}0.5$ & 53.4 & 10.26 & 13.02 & 13.86 & 12.38 & $100\%$ \\
\bottomrule
\end{tabular}
\end{table}

\begin{figure}[H]
\centering
\input{figures/16_sweeps/w4_1b_random.pgf}
\caption{\textbf{Llama-3.2-1B W4A16 hyperparameter sweeps, fixed random-Hadamard rotation.} \emph{Top}: the one-dimensional baseline sweeps, avg.\ perplexity against the swept knob, with the selected optimum starred. \emph{Bottom}: our denoising objective over its $5{\times}5$ grid of the two swept knobs $\rho$ and $\delta$ (which sets the noise scale through $\lambda{=}\rho{+}\delta$); one heatmap per upstream-error proxy (Gaussian noise, the GPTQ-W4 proxy, and their combination), with the best (minimum-perplexity) cell boxed.}
\label{fig:w4-1b-sweeps-random}
\end{figure}

\begin{figure}[H]
\centering
\input{figures/16_sweeps/w4_1b_learned.pgf}
\caption{\textbf{Llama-3.2-1B W4A16 hyperparameter sweeps, learned SpinQuant rotation.} Panels as in Fig.~\ref{fig:w4-1b-sweeps-random}.}
\label{fig:w4-1b-sweeps-learned}
\end{figure}

\newpage
\subsubsection{Llama-3.2-1B (W2A16 QuIP\#)}

\begin{table}[H]
\centering
\caption{\textbf{QuIP\# at W2A16 (2-bit), Llama-3.2-1B.} Parallel: symmetric correction quantizes each layer from clean activations, Denoising reg. is our objective (Gaussian noise injection, $\rho$/$\delta$ grid). Sequential: symmetric correction uses the symmetric objective on the sequential input, i.e.\ the outputs of the already-quantized upstream layers; asymmetric correction additionally targets the clean full-precision output (GPTAQ-correction applied to QuIP\#). Both then apply E8 rounding. Acc: zero-shot average accuracy (\%). Gap: $0\%{=}$parallel symmetric, $100\%{=}$sequential asymmetric.}
\label{tab:w2-1b}
\small
\setlength{\tabcolsep}{5pt}
\begin{tabular}{@{}ll ccc ccc cc@{}}
\toprule
& Method & Grid & Setting & Acc\,$\uparrow$ & Wiki\,$\downarrow$ & C4\,$\downarrow$ & FW\,$\downarrow$ & Avg\,$\downarrow$ & Gap\,$\uparrow$ \\
\midrule
& FP16 & -- & -- & 54.8 & 9.76 & 11.54 & 12.70 & 11.33 & -- \\
\midrule
\multirow{2}{*}{\rotatebox[origin=c]{90}{\scriptsize Parallel}} & symmetric correction & 25 & $\mathrm{damp}{=}0.55$ & 44.3 & 19.13 & 30.36 & 31.66 & 27.05 & $0\%$ \\
 & Denoising reg.\ (\textbf{ours}) & 25 & $\rho{=}0.25,\,\delta{=}0.0125$ & 45.1 & 18.01 & 28.46 & 30.07 & 25.51 & $38\%$ \\
\cmidrule(l){2-10}
\multirow{2}{*}{\rotatebox[origin=c]{90}{\scriptsize Seq.}} & symmetric correction & 25 & $\mathrm{damp}{=}0.3$ & 43.8 & 19.05 & 29.14 & 30.88 & 26.36 & $17\%$ \\
 & asymmetric correction & 25 & $\beta{=}0.3333$ & 45.1 & 16.30 & 25.67 & 27.15 & 23.04 & $100\%$ \\
\bottomrule
\end{tabular}

\end{table}

\begin{figure}[H]
\centering
\input{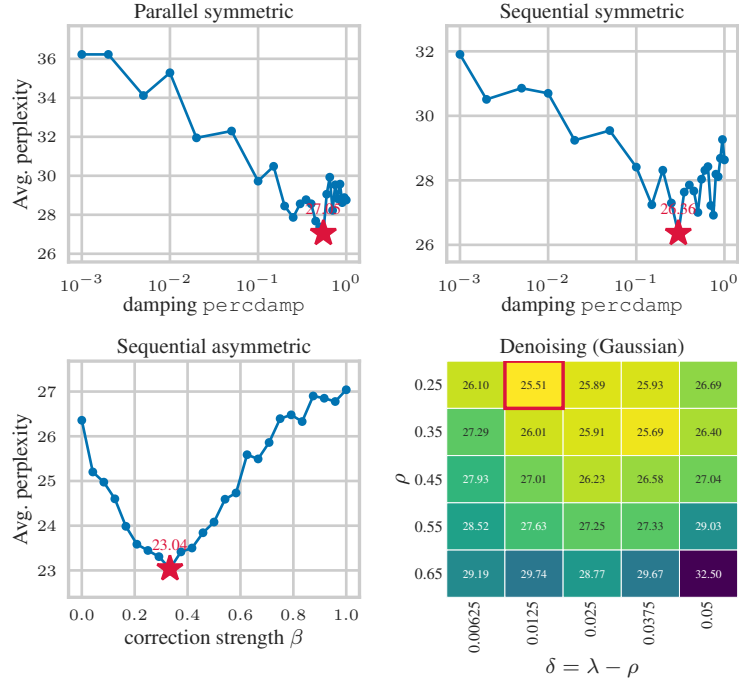}
\caption{\textbf{QuIP\# W2A16 hyperparameter sweeps, Llama-3.2-1B.} The three baseline sweeps (parallel symmetric, sequential symmetric, sequential asymmetric), avg.\ perplexity against the swept knob with the selected optimum starred, and our denoising objective over its $5{\times}5$ grid of $\rho$ and $\delta$ (which sets the noise scale through $\lambda{=}\rho{+}\delta$), best cell boxed.}
\label{fig:w2-1b-sweeps}
\end{figure}

\subsubsection{Llama-3.2-3B (W2A16 QuIP\#)}
\begin{table}[H]
\centering
\caption{\textbf{QuIP\# at W2A16 (2-bit), Llama-3.2-3B.} Columns and blocks as in \autoref{tab:w2-1b}.}
\label{tab:w2-3b}
\small
\setlength{\tabcolsep}{5pt}
\begin{tabular}{@{}ll ccc ccc cc@{}}
\toprule
& Method & Grid & Setting & Acc\,$\uparrow$ & Wiki\,$\downarrow$ & C4\,$\downarrow$ & FW\,$\downarrow$ & Avg\,$\downarrow$ & Gap\,$\uparrow$ \\
\midrule
& FP16 & -- & -- & 62.6 & 7.82 & 9.29 & 10.41 & 9.17 & -- \\
\midrule
\multirow{2}{*}{\rotatebox[origin=c]{90}{\scriptsize Parallel}} & symmetric correction & 25 & $\mathrm{damp}{=}0.25$ & 52.7 & 12.25 & 18.47 & 19.15 & 16.62 & $0\%$ \\
 & Denoising reg.\ (\textbf{ours}) & 25 & $\rho{=}0.35,\,\delta{=}0.025$ & 53.9 & 11.61 & 17.21 & 17.90 & 15.57 & $86\%$ \\
\cmidrule(l){2-10}
\multirow{2}{*}{\rotatebox[origin=c]{90}{\scriptsize Seq.}} & symmetric correction & 25 & $\mathrm{damp}{=}0.25$ & 52.6 & 12.28 & 18.44 & 19.19 & 16.64 & $-1\%$ \\
 & asymmetric correction & 25 & $\beta{=}0.3333$ & 51.3 & 11.23 & 17.12 & 17.84 & 15.40 & $100\%$ \\
\bottomrule
\end{tabular}

\end{table}

\begin{figure}[H]
\centering
\input{figures/16_sweeps/w2_3b.pgf}
\caption{\textbf{QuIP\# W2A16 hyperparameter sweeps, Llama-3.2-3B.} Panels as in \autoref{fig:w2-1b-sweeps}.}
\label{fig:w2-3b-sweeps}
\end{figure}

\clearpage
\subsubsection{Llama-3-8B (W2A16 QuIP\#)}
\begin{table}[H]
\centering
\caption{\textbf{QuIP\# at W2A16 (2-bit), Llama-3-8B.} Columns and blocks as in \autoref{tab:w2-1b}.}
\label{tab:w2-8b}
\small
\setlength{\tabcolsep}{5pt}
\begin{tabular}{@{}ll ccc ccc cc@{}}
\toprule
& Method & Grid & Setting & Acc\,$\uparrow$ & Wiki\,$\downarrow$ & C4\,$\downarrow$ & FW\,$\downarrow$ & Avg\,$\downarrow$ & Gap\,$\uparrow$ \\
\midrule
& FP16 & -- & -- & 67.1 & 6.14 & 7.78 & 8.78 & 7.57 & -- \\
\midrule
\multirow{2}{*}{\rotatebox[origin=c]{90}{\scriptsize Parallel}} & symmetric correction & 25 & $\mathrm{damp}{=}0.3$ & 57.7 & 9.67 & 15.54 & 16.06 & 13.76 & $0\%$ \\
 & Denoising reg.\ (\textbf{ours}) & 25 & $\rho{=}0.25,\,\delta{=}0.0125$ & 58.8 & 9.11 & 14.25 & 14.75 & 12.70 & $119\%$ \\
\cmidrule(l){2-10}
\multirow{2}{*}{\rotatebox[origin=c]{90}{\scriptsize Seq.}} & symmetric correction & 25 & $\mathrm{damp}{=}0.5$ & 58.0 & 9.98 & 15.66 & 16.19 & 13.94 & $-20\%$ \\
 & asymmetric correction & 25 & $\beta{=}0.3333$ & 56.2 & 9.05 & 14.55 & 15.01 & 12.87 & $100\%$ \\
\bottomrule
\end{tabular}

\end{table}

\begin{figure}[H]
\centering
\input{figures/16_sweeps/w2_8b.pgf}
\caption{\textbf{QuIP\# W2A16 hyperparameter sweeps, Llama-3-8B.} Panels as in \autoref{fig:w2-1b-sweeps}.}
\label{fig:w2-8b-sweeps}
\end{figure}

\clearpage
\subsubsection{Llama-2-13B (W2A16 QuIP\#)}
\begin{table}[H]
\centering
\caption{\textbf{QuIP\# at W2A16 (2-bit), Llama-2-13B.} Columns and blocks as in \autoref{tab:w2-1b}. Here parallel-symmetric and sequential-asymmetric nearly coincide.}
\label{tab:w2-13b}
\small
\setlength{\tabcolsep}{5pt}
\begin{tabular}{@{}ll ccc ccc cc@{}}
\toprule
& Method & Grid & Setting & Acc\,$\uparrow$ & Wiki\,$\downarrow$ & C4\,$\downarrow$ & FW\,$\downarrow$ & Avg\,$\downarrow$ & Gap\,$\uparrow$ \\
\midrule
& FP16 & -- & -- & 66.5 & 4.88 & 5.80 & 6.49 & 5.73 & -- \\
\midrule
\multirow{2}{*}{\rotatebox[origin=c]{90}{\scriptsize Parallel}} & symmetric correction & 25 & $\mathrm{damp}{=}0.6$ & 61.0 & 5.90 & 7.81 & 8.12 & 7.28 & -- \\
 & Denoising reg.\ (\textbf{ours}) & 25 & $\rho{=}0.25,\,\delta{=}0.0125$ & 61.0 & 5.81 & 7.72 & 8.05 & 7.19 & -- \\
\cmidrule(l){2-10}
\multirow{2}{*}{\rotatebox[origin=c]{90}{\scriptsize Seq.}} & symmetric correction & 25 & $\mathrm{damp}{=}0.2$ & 60.6 & 5.98 & 8.10 & 8.37 & 7.48 & -- \\
 & asymmetric correction & 25 & $\beta{=}0.375$ & 59.3 & 5.81 & 7.89 & 8.16 & 7.28 & -- \\
\bottomrule
\end{tabular}

\end{table}

\begin{figure}[H]
\centering
\input{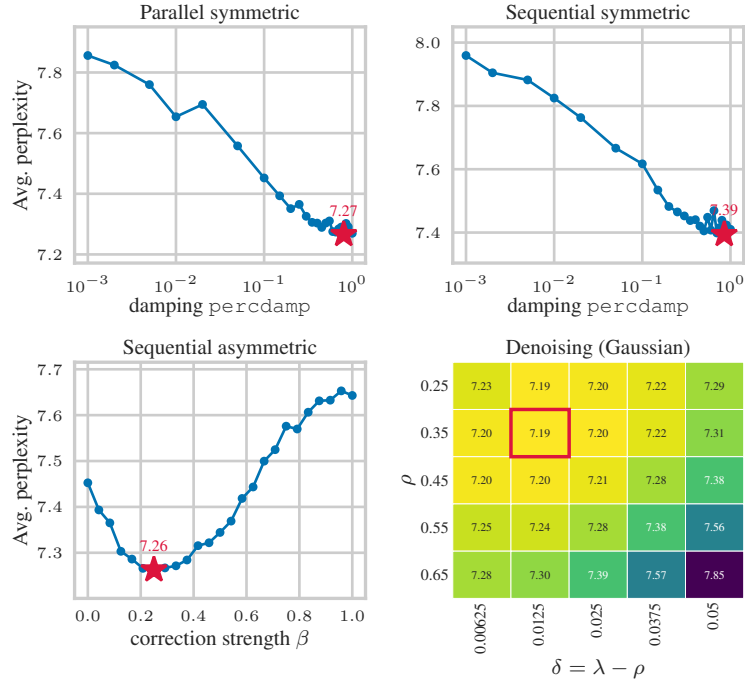}
\caption{\textbf{QuIP\# W2A16 hyperparameter sweeps, Llama-2-13B.} Panels as in \autoref{fig:w2-1b-sweeps}.}
\label{fig:w2-13b-sweeps}
\end{figure}

\clearpage

\newpage
\subsection{Injected-noise scale ablation}
\label{app:gamma-sweep}

We fix the injected-noise scale to $\gamma{=}0.05$ throughout (\cref{app:table-sweeps}). \Cref{fig:gamma-sweep} sweeps $\gamma$ for our objective on Llama-3.2-1B W3A16 (at $\lambda{=}0.4$, $\rho{=}0.3$, random Hadamard rotation), over ten log-spaced values in $[0.01,0.2]$. The best average perplexity ($16.59$) is reached at $\gamma{=}0.053$, and nearby values perform similarly. For $\gamma$ above $0.07$, the perplexity increases quickly.

\begin{figure}[H]
\centering
\input{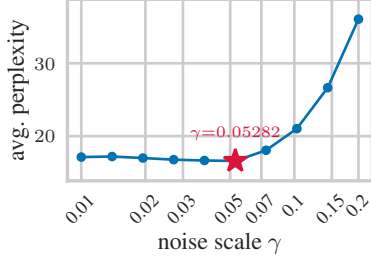}
\caption{\textbf{Injected-noise scale $\gamma$ sweep} (Llama-3.2-1B W3A16, $\lambda{=}0.4$, $\rho{=}0.3$, random Hadamard rotation). Average perplexity against $\gamma$ over ten log-spaced values in $[0.01,0.2]$.}
\label{fig:gamma-sweep}
\end{figure}

\subsection{Random noise model under forward pass modifications}
\label{app:noise-model}

For a fair comparison, we tune every method with equally sized hyperparameter sweeps (\cref{app:table-sweeps}) and use a noise scale of $\gamma{=}0.05$ throughout. \Cref{fig:iso-vs-channel} compares two noise models on Llama-3.2-1B for W3 GPTQ and W2 QuIP\#: module-wise isotropic noise $\boldsymbol{\Sigma}_m{=}\sigma_m^2\mathbf{I}$ with $\sigma_m{=}\gamma\,\mathbb{E}_t\|\mathbf{x}_{m,t}\|_\infty$, and per-channel noise $\boldsymbol{\Sigma}_m{=}\operatorname{diag}(\sigma_{m,c}^2)$ with $\sigma_{m,c}{=}\gamma\max_t|x_{m,t,c}|$.

Per-channel noise improves over the parallel baseline in both settings. Isotropic noise also works for GPTQ, but for QuIP\# it stays above the parallel baseline for every $\rho$ and every ridge $\alpha$ we tried. We attribute this to the forward pass used to estimate the noise moments: QuaRot fuses its rotations into the model before we estimate the moments, so the activations in this forward pass are already rotated. QuIP\#, in contrast, applies its incoherence transform only to the weights after we estimate moments, so the activations are not rotated and potentially contain more outliers. Per-channel noise scales with the magnitude of each channel and therefore regularizes these outlier channels more strongly. Since QuaRot already spreads such outliers, we use uniform noise as the default for W3 GPTQ, where \autoref{thm:inject-lipschitz} reduces to a uniform down-scaling of all weight columns. For W2 QuIP\#, we use per-channel noise as discussed.

\begin{figure}[H]
\centering
\input{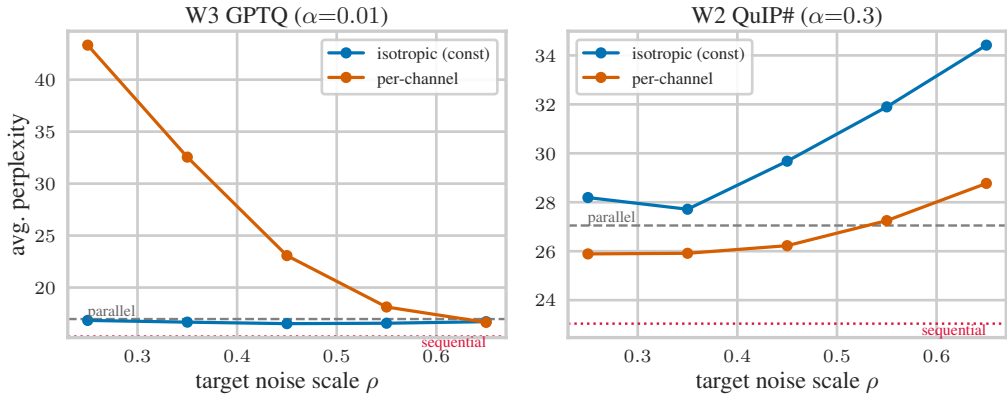}
\caption{\textbf{Isotropic vs.\ per-channel Gaussian noise} (Llama-3.2-1B, $\delta{=}0.025$, random Hadamard rotation, $\gamma{=}0.05$). Average perplexity against the target noise scale $\rho$; grey dashed marks the parallel baseline ($0\%$ gap) and red dotted the sequential baseline ($100\%$). Left (W3 GPTQ, $\alpha{=}0.01$): both isotropic and per-channel drop below the parallel baseline. Right (W2 QuIP\#, $\alpha{=}0.3$): only per-channel improves over the baseline; isotropic noise stays above it throughout.}
\label{fig:iso-vs-channel}
\end{figure}

\subsection{Generalization ablation: held-out evaluation}
\label{app:generalization}

Throughout the paper, every method quantizes models using a calibration set ($128$ WikiText-2 ``train'' sequences, \cref{app:setup}), which is separate from the test data we evaluate on.  In the main tables, for every method (baselines and our denoising objective alike) we select the hyperparameter configuration with the lowest WikiText-2 ``test'' perplexity, and report its perplexity averaged over three test sets: WikiText-2, C4, and FineWeb. This additional validation step is not part of our method and only used to report the best result each method can achieve, so that our method and the baselines are compared on equal terms, while the average perplexity can be compared to existing reports in the literature. We also report all results without this step separately in \cref{sec:proxy-model}. 

Nevertheless, we verify that reporting the best setting in this way does not change our conclusions. To this end, we report the avg.\ perplexity on C4 and FineWeb only (excluding the WikiText-2 test perplexity from the reported average). This way, calibration and validation use disjoint splits of WikiText-2, and evaluation uses separate datasets. Importantly, we find the qualitative conclusions remain the same.

In detail, \Cref{tab:main-wikisel} repeats \autoref{tab:main} under this protocol. The parallel-to-sequential gap that our denoising closes is preserved across the board: at $3$ bits it closes $34/21/32/12\%$ of the gap on Llama-3.2-1B/3B, Llama-3-8B, and Llama-2-13B (versus $18$--$31\%$ in the main table), and at $2$ bits it again closes a large fraction on the smaller models and matches or surpasses the sequential baseline on the larger ones. The qualitative conclusions are identical to \autoref{tab:main}. %

\begin{table}[h]
\centering
\caption{\textbf{Held-out generalization ablation.} Same objectives, models, and sweeps as \autoref{tab:main}, every configuration is selected by WikiText-2 test perplexity, but the ``Avg'' column reports the average of the two held-out sets (C4 and FineWeb), which are unseen by the selection. The conclusions match \autoref{tab:main}, showing that the qualitative improvements persist under held-out evaluation.}
\label{tab:main-wikisel}
{\scriptsize\setlength{\tabcolsep}{2.2pt}\begin{tabular}{@{}lll ccc ccc ccc ccc@{}}
\toprule
& & Quantization objective & \multicolumn{3}{c}{Llama\,1B} & \multicolumn{3}{c}{Llama\,3B} & \multicolumn{3}{c}{Llama\,8B} & \multicolumn{3}{c}{Llama\,13B} \\
\cmidrule(lr){4-6}\cmidrule(lr){7-9}\cmidrule(lr){10-12}\cmidrule(lr){13-15}
& & & Acc\,$\uparrow$ & Avg\,$\downarrow$ & Gap\,$\uparrow$ & Acc\,$\uparrow$ & Avg\,$\downarrow$ & Gap\,$\uparrow$ & Acc\,$\uparrow$ & Avg\,$\downarrow$ & Gap\,$\uparrow$ & Acc\,$\uparrow$ & Avg\,$\downarrow$ & Gap\,$\uparrow$ \\
\midrule
FP16 & & & 54.8 & 12.12 & -- & 62.6 & 9.85 & -- & 67.1 & 8.28 & -- & 66.5 & 6.15 & -- \\
\midrule
\multirow{4}{*}{\rotatebox[origin=c]{90}{\scriptsize \shortstack{\textbf{3-bit}\\\textbf{GPTQ}}}} & \multirow{2}{*}{\rotatebox[origin=c]{90}{\scriptsize Par.}} & symmetric correction & 49.4 & 19.00 & $0\%$ & 58.5 & 13.37 & $0\%$ & 63.9 & 11.10 & $0\%$ & 64.2 & 6.92 & $0\%$ \\
 & & Denoising reg.\ (\textbf{ours}) & 49.3 & 18.32 & $\mathbf{34\%}$ & 57.1 & 13.22 & $\mathbf{21\%}$ & 64.9 & 10.96 & $\mathbf{32\%}$ & 64.5 & 6.90 & $\mathbf{12\%}$ \\
\cmidrule{2-15}
 & \multirow{2}{*}{\rotatebox[origin=c]{90}{\scriptsize Seq.}} & symmetric correction & 49.8 & 18.30 & $35\%$ & 59.0 & 13.24 & $18\%$ & 63.5 & 10.94 & $36\%$ & 63.7 & 6.89 & $13\%$ \\
 & & asymmetric correction & 50.3 & 17.00 & $100\%$ & 58.4 & 12.65 & $100\%$ & 63.8 & 10.65 & $100\%$ & 64.5 & 6.72 & $100\%$ \\
\midrule[\heavyrulewidth]
\multirow{4}{*}{\rotatebox[origin=c]{90}{\scriptsize \shortstack{\textbf{2-bit}\\\textbf{QuIP\#}}}} & \multirow{2}{*}{\rotatebox[origin=c]{90}{\scriptsize Par.}} & symmetric correction & 44.3 & 31.01 & $0\%$ & 52.7 & 18.81 & $0\%$ & 57.7 & 15.80 & $0\%$ & 61.0 & 7.96 & -- \\
 & & Denoising reg.\ (\textbf{ours}) & 45.1 & 29.27 & $\mathbf{38\%}$ & 53.9 & 17.55 & $\mathbf{95\%}$ & 58.8 & 14.50 & $\mathbf{127\%}$ & 61.0 & 7.89 & -- \\
\cmidrule{2-15}
 & \multirow{2}{*}{\rotatebox[origin=c]{90}{\scriptsize Seq.}} & symmetric correction & 43.8 & 30.01 & $22\%$ & 52.6 & 18.81 & $0\%$ & 58.0 & 15.92 & $-12\%$ & 60.6 & 8.23 & -- \\
 & & asymmetric correction & 45.1 & 26.41 & $100\%$ & 51.3 & 17.48 & $100\%$ & 56.2 & 14.78 & $100\%$ & 59.3 & 8.02 & -- \\
\bottomrule
\end{tabular}
}
\end{table}

\newpage
\subsection{Sequential-prefix and 4-bit-prefix trade-offs across models}

We provide additional results for the two trade-offs shown in \autoref{fig:tradeoffs} (middle and right). \autoref{fig:splice-grid} extends the quantization-time trade-off, where the first $K$ blocks are quantized sequentially and the remaining ones in parallel, to all four models, and additionally shows the per-block signal-to-noise ratio (SNR) of the residual stream. The SNR measures how close the quantized model's residual stream stays to the full-precision one after each block, $\mathrm{SNR}{=}{-}20\log_{10}\big(\mathbb{E}\,\|\mathbf{h}_q{-}\mathbf{h}\|/\|\mathbf{h}\|\big)$ in dB, with $\mathbf{h}$ and $\mathbf{h}_q$ the full-precision and quantized residual streams (higher is better). \autoref{fig:w4prefix-bits} extends the inference-time trade-off, where a prefix of modules is promoted to $4$ bits, to Llama-3.2-1B/3B and Llama-3-8B. In both cases, the conclusions of the main text hold across all models.

\begin{figure}[h]
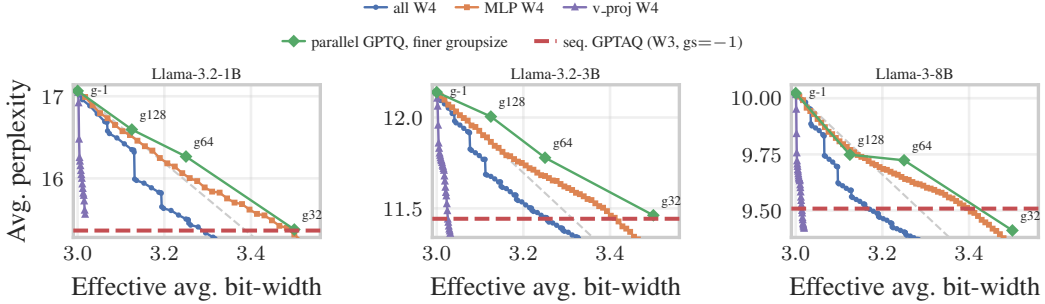

  \centering
  %% Creator: Matplotlib, PGF backend
%%
%% To include the figure in your LaTeX document, write
%%   \input{<filename>.pgf}
%%
%% Make sure the required packages are loaded in your preamble
%%   \usepackage{pgf}
%%
%% Also ensure that all the required font packages are loaded; for instance,
%% the lmodern package is sometimes necessary when using math font.
%%   \usepackage{lmodern}
%%
%% Figures using additional raster images can only be included by \input if
%% they are in the same directory as the main LaTeX file. For loading figures
%% from other directories you can use the `import` package
%%   \usepackage{import}
%%
%% and then include the figures with
%%   \import{<path to file>}{<filename>.pgf}
%%
%% Matplotlib used the following preamble
%%   
%%   \usepackage[utf8]{inputenc}
%%   \usepackage[T1]{fontenc}
%%   \usepackage{times}
%%   \usepackage{amsmath}
%%   \newcommand*{\mat}[1]{\boldsymbol{#1}}
%%   
%%   \makeatletter\@ifpackageloaded{underscore}{}{\usepackage[strings]{underscore}}\makeatother
%%
\begingroup%
\makeatletter%
\begin{pgfpicture}%
\pgfpathrectangle{\pgfpointorigin}{\pgfqpoint{5.500000in}{0.420000in}}%
\pgfusepath{use as bounding box, clip}%
\begin{pgfscope}%
\pgfsetbuttcap%
\pgfsetmiterjoin%
\definecolor{currentfill}{rgb}{1.000000,1.000000,1.000000}%
\pgfsetfillcolor{currentfill}%
\pgfsetlinewidth{0.000000pt}%
\definecolor{currentstroke}{rgb}{1.000000,1.000000,1.000000}%
\pgfsetstrokecolor{currentstroke}%
\pgfsetdash{}{0pt}%
\pgfpathmoveto{\pgfqpoint{0.000000in}{0.000000in}}%
\pgfpathlineto{\pgfqpoint{5.500000in}{0.000000in}}%
\pgfpathlineto{\pgfqpoint{5.500000in}{0.420000in}}%
\pgfpathlineto{\pgfqpoint{0.000000in}{0.420000in}}%
\pgfpathlineto{\pgfqpoint{0.000000in}{0.000000in}}%
\pgfpathclose%
\pgfusepath{fill}%
\end{pgfscope}%
\begin{pgfscope}%
\pgfsetroundcap%
\pgfsetroundjoin%
\pgfsetlinewidth{0.903375pt}%
\definecolor{currentstroke}{rgb}{0.298039,0.447059,0.690196}%
\pgfsetstrokecolor{currentstroke}%
\pgfsetdash{}{0pt}%
\pgfpathmoveto{\pgfqpoint{1.977129in}{0.303020in}}%
\pgfpathlineto{\pgfqpoint{2.035462in}{0.303020in}}%
\pgfpathlineto{\pgfqpoint{2.093795in}{0.303020in}}%
\pgfusepath{stroke}%
\end{pgfscope}%
\begin{pgfscope}%
\pgfsetbuttcap%
\pgfsetroundjoin%
\definecolor{currentfill}{rgb}{0.298039,0.447059,0.690196}%
\pgfsetfillcolor{currentfill}%
\pgfsetlinewidth{1.003750pt}%
\definecolor{currentstroke}{rgb}{0.298039,0.447059,0.690196}%
\pgfsetstrokecolor{currentstroke}%
\pgfsetdash{}{0pt}%
\pgfsys@defobject{currentmarker}{\pgfqpoint{-0.006944in}{-0.006944in}}{\pgfqpoint{0.006944in}{0.006944in}}{%
\pgfpathmoveto{\pgfqpoint{0.000000in}{-0.006944in}}%
\pgfpathcurveto{\pgfqpoint{0.001842in}{-0.006944in}}{\pgfqpoint{0.003608in}{-0.006213in}}{\pgfqpoint{0.004910in}{-0.004910in}}%
\pgfpathcurveto{\pgfqpoint{0.006213in}{-0.003608in}}{\pgfqpoint{0.006944in}{-0.001842in}}{\pgfqpoint{0.006944in}{0.000000in}}%
\pgfpathcurveto{\pgfqpoint{0.006944in}{0.001842in}}{\pgfqpoint{0.006213in}{0.003608in}}{\pgfqpoint{0.004910in}{0.004910in}}%
\pgfpathcurveto{\pgfqpoint{0.003608in}{0.006213in}}{\pgfqpoint{0.001842in}{0.006944in}}{\pgfqpoint{0.000000in}{0.006944in}}%
\pgfpathcurveto{\pgfqpoint{-0.001842in}{0.006944in}}{\pgfqpoint{-0.003608in}{0.006213in}}{\pgfqpoint{-0.004910in}{0.004910in}}%
\pgfpathcurveto{\pgfqpoint{-0.006213in}{0.003608in}}{\pgfqpoint{-0.006944in}{0.001842in}}{\pgfqpoint{-0.006944in}{0.000000in}}%
\pgfpathcurveto{\pgfqpoint{-0.006944in}{-0.001842in}}{\pgfqpoint{-0.006213in}{-0.003608in}}{\pgfqpoint{-0.004910in}{-0.004910in}}%
\pgfpathcurveto{\pgfqpoint{-0.003608in}{-0.006213in}}{\pgfqpoint{-0.001842in}{-0.006944in}}{\pgfqpoint{0.000000in}{-0.006944in}}%
\pgfpathlineto{\pgfqpoint{0.000000in}{-0.006944in}}%
\pgfpathclose%
\pgfusepath{stroke,fill}%
}%
\begin{pgfscope}%
\pgfsys@transformshift{2.035462in}{0.303020in}%
\pgfsys@useobject{currentmarker}{}%
\end{pgfscope}%
\end{pgfscope}%
\begin{pgfscope}%
\definecolor{textcolor}{rgb}{0.150000,0.150000,0.150000}%
\pgfsetstrokecolor{textcolor}%
\pgfsetfillcolor{textcolor}%
\pgftext[x=2.127129in,y=0.273854in,left,base]{\color{textcolor}\rmfamily\fontsize{6.000000}{7.200000}\selectfont all W4}%
\end{pgfscope}%
\begin{pgfscope}%
\pgfsetroundcap%
\pgfsetroundjoin%
\pgfsetlinewidth{0.903375pt}%
\definecolor{currentstroke}{rgb}{0.866667,0.517647,0.321569}%
\pgfsetstrokecolor{currentstroke}%
\pgfsetdash{}{0pt}%
\pgfpathmoveto{\pgfqpoint{2.451626in}{0.303020in}}%
\pgfpathlineto{\pgfqpoint{2.509959in}{0.303020in}}%
\pgfpathlineto{\pgfqpoint{2.568292in}{0.303020in}}%
\pgfusepath{stroke}%
\end{pgfscope}%
\begin{pgfscope}%
\pgfsetbuttcap%
\pgfsetmiterjoin%
\definecolor{currentfill}{rgb}{0.866667,0.517647,0.321569}%
\pgfsetfillcolor{currentfill}%
\pgfsetlinewidth{1.003750pt}%
\definecolor{currentstroke}{rgb}{0.866667,0.517647,0.321569}%
\pgfsetstrokecolor{currentstroke}%
\pgfsetdash{}{0pt}%
\pgfsys@defobject{currentmarker}{\pgfqpoint{-0.008333in}{-0.008333in}}{\pgfqpoint{0.008333in}{0.008333in}}{%
\pgfpathmoveto{\pgfqpoint{-0.008333in}{-0.008333in}}%
\pgfpathlineto{\pgfqpoint{0.008333in}{-0.008333in}}%
\pgfpathlineto{\pgfqpoint{0.008333in}{0.008333in}}%
\pgfpathlineto{\pgfqpoint{-0.008333in}{0.008333in}}%
\pgfpathlineto{\pgfqpoint{-0.008333in}{-0.008333in}}%
\pgfpathclose%
\pgfusepath{stroke,fill}%
}%
\begin{pgfscope}%
\pgfsys@transformshift{2.509959in}{0.303020in}%
\pgfsys@useobject{currentmarker}{}%
\end{pgfscope}%
\end{pgfscope}%
\begin{pgfscope}%
\definecolor{textcolor}{rgb}{0.150000,0.150000,0.150000}%
\pgfsetstrokecolor{textcolor}%
\pgfsetfillcolor{textcolor}%
\pgftext[x=2.601626in,y=0.273854in,left,base]{\color{textcolor}\rmfamily\fontsize{6.000000}{7.200000}\selectfont MLP W4}%
\end{pgfscope}%
\begin{pgfscope}%
\pgfsetroundcap%
\pgfsetroundjoin%
\pgfsetlinewidth{0.903375pt}%
\definecolor{currentstroke}{rgb}{0.505882,0.447059,0.701961}%
\pgfsetstrokecolor{currentstroke}%
\pgfsetdash{}{0pt}%
\pgfpathmoveto{\pgfqpoint{3.014123in}{0.303020in}}%
\pgfpathlineto{\pgfqpoint{3.072457in}{0.303020in}}%
\pgfpathlineto{\pgfqpoint{3.130790in}{0.303020in}}%
\pgfusepath{stroke}%
\end{pgfscope}%
\begin{pgfscope}%
\pgfsetbuttcap%
\pgfsetmiterjoin%
\definecolor{currentfill}{rgb}{0.505882,0.447059,0.701961}%
\pgfsetfillcolor{currentfill}%
\pgfsetlinewidth{1.003750pt}%
\definecolor{currentstroke}{rgb}{0.505882,0.447059,0.701961}%
\pgfsetstrokecolor{currentstroke}%
\pgfsetdash{}{0pt}%
\pgfsys@defobject{currentmarker}{\pgfqpoint{-0.009722in}{-0.009722in}}{\pgfqpoint{0.009722in}{0.009722in}}{%
\pgfpathmoveto{\pgfqpoint{0.000000in}{0.009722in}}%
\pgfpathlineto{\pgfqpoint{-0.009722in}{-0.009722in}}%
\pgfpathlineto{\pgfqpoint{0.009722in}{-0.009722in}}%
\pgfpathlineto{\pgfqpoint{0.000000in}{0.009722in}}%
\pgfpathclose%
\pgfusepath{stroke,fill}%
}%
\begin{pgfscope}%
\pgfsys@transformshift{3.072457in}{0.303020in}%
\pgfsys@useobject{currentmarker}{}%
\end{pgfscope}%
\end{pgfscope}%
\begin{pgfscope}%
\definecolor{textcolor}{rgb}{0.150000,0.150000,0.150000}%
\pgfsetstrokecolor{textcolor}%
\pgfsetfillcolor{textcolor}%
\pgftext[x=3.164123in,y=0.273854in,left,base]{\color{textcolor}\rmfamily\fontsize{6.000000}{7.200000}\selectfont v\_proj W4}%
\end{pgfscope}%
\begin{pgfscope}%
\pgfsetroundcap%
\pgfsetroundjoin%
\pgfsetlinewidth{0.903375pt}%
\definecolor{currentstroke}{rgb}{0.333333,0.658824,0.407843}%
\pgfsetstrokecolor{currentstroke}%
\pgfsetdash{}{0pt}%
\pgfpathmoveto{\pgfqpoint{1.549707in}{0.126620in}}%
\pgfpathlineto{\pgfqpoint{1.608040in}{0.126620in}}%
\pgfpathlineto{\pgfqpoint{1.666374in}{0.126620in}}%
\pgfusepath{stroke}%
\end{pgfscope}%
\begin{pgfscope}%
\pgfsetbuttcap%
\pgfsetmiterjoin%
\definecolor{currentfill}{rgb}{0.333333,0.658824,0.407843}%
\pgfsetfillcolor{currentfill}%
\pgfsetlinewidth{1.003750pt}%
\definecolor{currentstroke}{rgb}{0.333333,0.658824,0.407843}%
\pgfsetstrokecolor{currentstroke}%
\pgfsetdash{}{0pt}%
\pgfsys@defobject{currentmarker}{\pgfqpoint{-0.023570in}{-0.023570in}}{\pgfqpoint{0.023570in}{0.023570in}}{%
\pgfpathmoveto{\pgfqpoint{-0.000000in}{-0.023570in}}%
\pgfpathlineto{\pgfqpoint{0.023570in}{0.000000in}}%
\pgfpathlineto{\pgfqpoint{-0.000000in}{0.023570in}}%
\pgfpathlineto{\pgfqpoint{-0.023570in}{0.000000in}}%
\pgfpathlineto{\pgfqpoint{-0.000000in}{-0.023570in}}%
\pgfpathclose%
\pgfusepath{stroke,fill}%
}%
\begin{pgfscope}%
\pgfsys@transformshift{1.608040in}{0.126620in}%
\pgfsys@useobject{currentmarker}{}%
\end{pgfscope}%
\end{pgfscope}%
\begin{pgfscope}%
\definecolor{textcolor}{rgb}{0.150000,0.150000,0.150000}%
\pgfsetstrokecolor{textcolor}%
\pgfsetfillcolor{textcolor}%
\pgftext[x=1.699707in,y=0.097454in,left,base]{\color{textcolor}\rmfamily\fontsize{6.000000}{7.200000}\selectfont parallel GPTQ, finer groupsize}%
\end{pgfscope}%
\begin{pgfscope}%
\pgfsetbuttcap%
\pgfsetroundjoin%
\pgfsetlinewidth{1.505625pt}%
\definecolor{currentstroke}{rgb}{0.768627,0.305882,0.321569}%
\pgfsetstrokecolor{currentstroke}%
\pgfsetdash{{5.550000pt}{2.400000pt}}{0.000000pt}%
\pgfpathmoveto{\pgfqpoint{2.860612in}{0.126620in}}%
\pgfpathlineto{\pgfqpoint{2.918945in}{0.126620in}}%
\pgfpathlineto{\pgfqpoint{2.977279in}{0.126620in}}%
\pgfusepath{stroke}%
\end{pgfscope}%
\begin{pgfscope}%
\definecolor{textcolor}{rgb}{0.150000,0.150000,0.150000}%
\pgfsetstrokecolor{textcolor}%
\pgfsetfillcolor{textcolor}%
\pgftext[x=3.010612in,y=0.097454in,left,base]{\color{textcolor}\rmfamily\fontsize{6.000000}{7.200000}\selectfont seq.\ GPTAQ (W3, gs\(\displaystyle {=}{-}1\))}%
\end{pgfscope}%
\end{pgfpicture}%
\makeatother%
\endgroup%
\\[-6pt]
  \begin{minipage}[b]{0.32\textwidth}
    \centering
    \input{figures/10_w4prefix_module/w4prefix_zoom_app_1b.pgf}
  \end{minipage}\hfill
  \begin{minipage}[b]{0.32\textwidth}
    \centering
    \input{figures/10_w4prefix_module/w4prefix_zoom_app_3b.pgf}
  \end{minipage}\hfill
  \begin{minipage}[b]{0.32\textwidth}
    \centering
    \input{figures/10_w4prefix_module/w4prefix_zoom_app_8b.pgf}
  \end{minipage}
  \caption{\textbf{Effective bit-width vs.\ perplexity} on Llama-3.2-1B/3B and Llama-3-8B (W3A16). Same plot as \autoref{fig:tradeoffs} (right) for Llama-3.2-3B, extended to Llama-3.2-1B and Llama-3-8B.}
  \label{fig:w4prefix-bits}
\end{figure}

\begin{figure}[h]
  \centering
  \includegraphics[width=\linewidth]{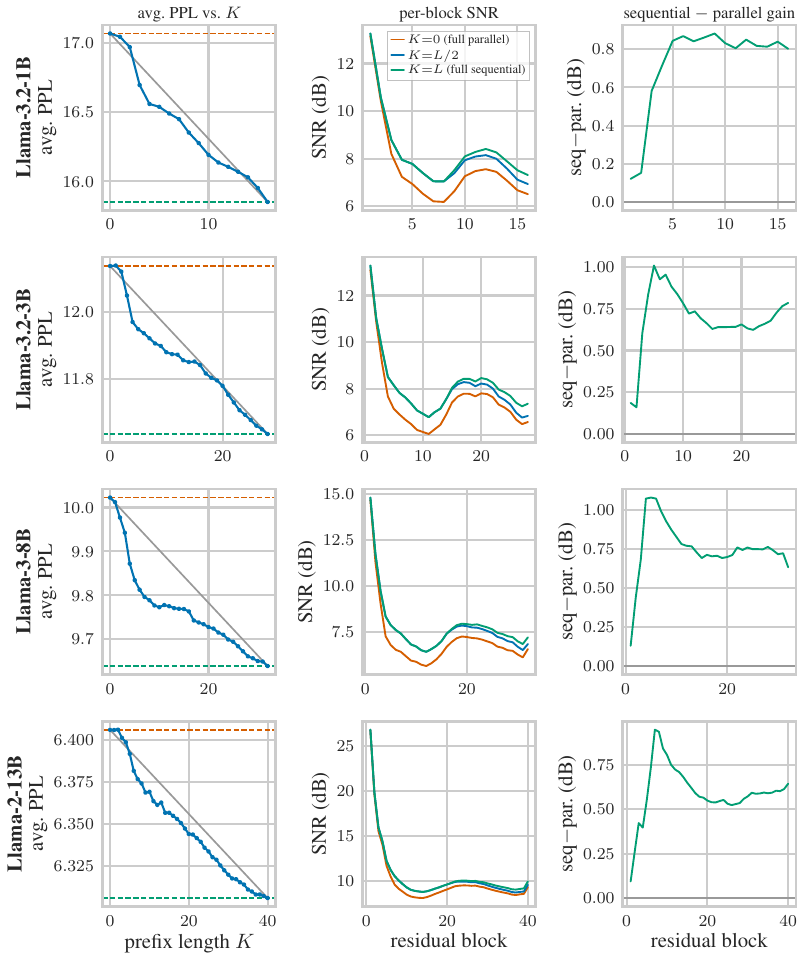}
  \caption{\textbf{Per-model sequential-prefix diagnostics} (W3A16), one row per model. The first $K$ blocks are quantized sequentially (GPTAQ, $\beta{=}0.1$) and the remaining ones in parallel (GPTQ), at fixed hyperparameters. Left: avg.\ perplexity vs.\ $K$; the gray line interpolates linearly between $K{=}0$ and $K{=}L$, so a curve below it means the gains are front-loaded. Middle: per-block residual-stream SNR for $K\in\{0,L/2,L\}$. Right: per-block SNR gain of sequential over parallel quantization.}
  \label{fig:splice-grid}
\end{figure}

\clearpage
\newpage
\subsection{Runtime measurements and projected multi-GPU time}
\label{app:runtimes}

\textbf{Measurement details.} All runtimes in this paper are wall-clock times measured on a single GPU. Our implementation runs parallel quantization on one GPU in two stages. First, it runs the calibration forward passes that collect the statistics of every layer (for our denoising, the clean pass and the noisy pass that accumulates the noise moments). Second, it computes each layer's pretransform and rounds its weights. In the parallel schedule, the second stage of a layer depends only on the statistics of that layer, so all layers could be processed at the same time. On a single GPU, however, we process them one after another, so the measured time is the sum over all layers. The sequential schedule instead alternates between the two stages: each layer is rounded only after all previous layers are quantized, since its statistics depend on their quantized outputs.

\textbf{Single-GPU runtimes.} \Cref{tab:runtimes} reports the mean per-run wall-clock time on a single NVIDIA A100 GPU, split into quantization (Quant) and perplexity evaluation (Eval), for each method row of the W3A16 tables. For our methods, Quant covers the full parallel pipeline (proxy quantization, moment estimation, the weight pretransform, and rounding). Learned-rotation runs load a precomputed rotation, and we report the time for learning rotations separately in \autoref{app:setup}. For Llama-2-13B, which is close to the memory of a single $40$\,GB A100, we use numerically exact CPU offloading of the residual activation streams for configurations that do not otherwise fit (the learned-rotation sequential GPTAQ), which adds a small overhead to the reported time. \Cref{tab:quipsharp-runtimes} reports the same breakdown for the QuIP\# W2A16 sweeps, on a single A100 for Llama-3.2-1B/3B and Llama-3-8B and on a single H200 for Llama-2-13B, so the $13$B column is not directly comparable to the others.

\textbf{Overhead of denoising depends on the solver and the noise model.} Gaussian-only denoising needs no preliminary quantization pass, but its noisy forward pass, moment estimation, and pretransform add cost on a single GPU. How much this matters depends on the rounding solver. With GPTQ, rounding is cheap, so this overhead dominates: on Llama-3-8B, our denoising takes $29.4$ minutes, compared to $11.6$ for parallel GPTQ and $24.1$ for sequential GPTAQ. With QuIP\#, rounding is roughly an order of magnitude more expensive (e.g.\ ${\approx}24$ vs.\ ${\approx}2.5$ minutes on Llama-3.2-1B), so the same overhead is small in relative terms (e.g.\ $183.2$ vs.\ $173.3$ minutes for the sequential baseline on Llama-3-8B). The proxy-based variants (W3 and W3+Gauss rows) additionally require a preliminary parallel quantization pass, unless an existing quantized model is reused (\cref{sec:proxy-model}). In all cases, the per-layer processing remains independent.

\textbf{Projected multi-GPU time.} We do not run multi-GPU quantization experiments. Instead, \autoref{fig:quant-cost} and \autoref{fig:quant-cost-gpus} report a projection computed from measured single-GPU QuIP\# times. For these figures, we re-measure the Llama-2-13B times on a single A100, so that all models are timed on the same GPU. We split the measured quantization time into the time of the forward passes, $T_{\mathrm{fwd}}$, and the time of the rounding, $T_{\mathrm{round}}$, and approximate the time of a parallel method on $N$ GPUs as
\begin{equation*}
T_{\mathrm{par}}(N)=T_{\mathrm{fwd}}+T_{\mathrm{round}}/N .
\end{equation*}
For the sequential schedule, we keep the single-GPU time, $T_{\mathrm{seq}}(N)=T_{\mathrm{seq}}(1)$, since its layers cannot be rounded concurrently. Note that the projection assumes that the layers are split evenly across GPUs and ignores communication overhead, so it is an upper bound on the achievable speed-up. Since the rounding of different layers is independent given the collected statistics, distributing it across GPUs does not change the quantized model.

\Cref{fig:quant-cost-gpus} shows the projection for Llama-2-13B (W2A16 QuIP\#). On a single GPU, all three methods take a comparable ${\approx}290$--$316$ minutes. The projected time of the parallel methods decreases roughly as $1/N$, so that on $8$ GPUs our denoising takes ${\approx}41$ minutes compared to ${\approx}306$ for the sequential baseline (a ${\approx}7.5\times$ reduction), at a small overhead over parallel QuIP\#. Across the models in \autoref{fig:quant-cost}, the projected reduction on $8$ GPUs ranges from ${\approx}7.1\times$ (Llama-3.2-1B) to ${\approx}7.5\times$ (Llama-2-13B).

\begin{figure}[h]
\centering
\input{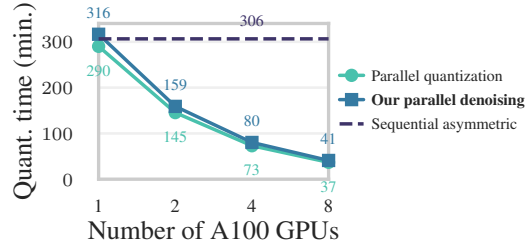}
\caption{\textbf{Projected quantization time vs.\ number of GPUs on Llama-2-13B (W2A16 QuIP\#).} Projection from measured single-GPU times. Parallel quantization and our parallel denoising distribute the per-block rounding across GPUs, whereas the sequential asymmetric correction cannot be parallelized.}
\label{fig:quant-cost-gpus}
\end{figure}

\begin{table}[h]
\centering
\caption{\textbf{Per-row runtime on a single A100} (minutes, mean over each sweep), split into quantization (Quant) and perplexity evaluation (Eval), for Llama-3.2-1B/3B, Llama-3-8B, and Llama-2-13B at W3A16. Rot.: whether the rotation is a learned SpinQuant rotation ($\times$ means random rotation).}
\footnotesize
\label{tab:runtimes}
\begin{tabular}{@{}ll cc cc cc cc@{}}
\toprule
Method & Rot. & \multicolumn{2}{c}{Llama\,1B} & \multicolumn{2}{c}{Llama\,3B} & \multicolumn{2}{c}{Llama\,8B} & \multicolumn{2}{c}{Llama\,13B} \\
\cmidrule(lr){3-4}\cmidrule(lr){5-6}\cmidrule(lr){7-8}\cmidrule(lr){9-10}
& & Quant & Eval & Quant & Eval & Quant & Eval & Quant & Eval \\
\midrule
\multirow{2}{*}{Parallel GPTQ} & $\times$ & 2.5 & 1.0 & 5.6 & 1.9 & 11.6 & 3.4 & 16.6 & 5.2 \\
 & \checkmark & 2.5 & 1.0 & 5.6 & 1.9 & 11.6 & 3.4 & 16.6 & 5.2 \\
\multirow{2}{*}{Gaussian \textit{(ours)}} & $\times$ & 5.7 & 0.9 & 13.1 & 1.9 & 29.4 & 3.5 & 43.2 & 5.2 \\
 & \checkmark & 5.8 & 1.0 & 13.1 & 1.9 & 29.4 & 3.5 & 43.3 & 5.2 \\
\multirow{2}{*}{W3 proxy \textit{(ours)}} & $\times$ & 8.2 & 1.0 & 19.0 & 1.9 & 42.4 & 3.5 & 61.2 & 5.2 \\
 & \checkmark & 8.2 & 1.0 & 19.0 & 1.9 & 42.5 & 3.5 & 61.0 & 5.2 \\
\multirow{2}{*}{W3 proxy + Gauss.\ \textit{(ours)}} & $\times$ & 8.2 & 1.0 & 19.0 & 1.9 & 42.4 & 3.5 & 60.8 & 5.2 \\
 & \checkmark & 8.3 & 1.0 & 19.0 & 1.9 & 41.7 & 3.5 & 60.5 & 5.2 \\
\midrule
\multirow{2}{*}{Sequential GPTQ} & $\times$ & 3.2 & 1.0 & 6.9 & 1.9 & 14.1 & 3.4 & 20.5 & 5.2 \\
 & \checkmark & 3.2 & 1.0 & 7.0 & 1.9 & 14.1 & 3.4 & 20.4 & 5.2 \\
\multirow{2}{*}{Sequential GPTAQ} & $\times$ & 5.2 & 1.0 & 11.4 & 1.9 & 24.1 & 3.5 & 34.7 & 5.2 \\
 & \checkmark & 5.2 & 1.0 & 11.5 & 1.9 & 24.1 & 3.5 & 36.9 & 5.2 \\
\bottomrule
\end{tabular}

\end{table}

\begin{table}[h]
\centering
\caption{\textbf{Per-schedule runtime for the QuIP\# W2A16 sweeps} (minutes, mean over each sweep). A100 for $1$B/$3$B/$8$B, H200 for $13$B. QuIP\#'s rounding is more expensive than GPTQ's (\cref{tab:runtimes}).}
\label{tab:quipsharp-runtimes}
\begin{tabular}{@{}l cc cc cc cc@{}}
\toprule
Schedule & \multicolumn{2}{c}{Llama\,1B} & \multicolumn{2}{c}{Llama\,3B} & \multicolumn{2}{c}{Llama\,8B} & \multicolumn{2}{c}{Llama\,13B} \\
\cmidrule(lr){2-3}\cmidrule(lr){4-5}\cmidrule(lr){6-7}\cmidrule(lr){8-9}
& Quant & Eval & Quant & Eval & Quant & Eval & Quant & Eval \\
\midrule
Parallel symmetric & 23.6 & 0.5 & 64.7 & 1.1 & 163.3 & 2.3 & 138.9 & 1.4 \\
Denoising reg.\ (ours) & 27.3 & 0.5 & 73.9 & 1.1 & 183.2 & 2.3 & 145.1 & 1.4 \\
\midrule
Sequential symmetric & 23.8 & 0.5 & 65.3 & 1.1 & 164.1 & 2.3 & 140.1 & 1.4 \\
Sequential asymmetric & 25.4 & 0.5 & 69.2 & 1.1 & 173.3 & 2.3 & 142.1 & 1.4 \\
\bottomrule
\end{tabular}

\end{table}

\textbf{Forward-pass cost of the noise injection is negligible.} Beyond the per-layer rounding, every parallel method runs one clean full-precision forward pass over the calibration set to build its moments. Our denoising objective needs one additional forward pass with the injected (per-channel Gaussian) noise to estimate the noise statistics. \Cref{tab:forward-times} reports both passes per model on a single A100 ($n{=}128$ calibration sequences, seqlen $2048$). The noisy pass is only $1.25$--$1.40\times$ the clean one (it additionally samples the noise and accumulates the noise and cross-covariance statistics). The overhead of our extra pass is a few seconds and negligible next to the per-layer rounding that dominates the quantization time in \autoref{tab:runtimes} and \autoref{tab:quipsharp-runtimes}. 

\begin{table}[h]
\centering
\caption{\textbf{Clean vs.\ noisy calibration forward-pass time on a single A100} (seconds, $n{=}128$ sequences, seqlen $2048$). Clean is the full-precision forward pass every parallel method runs to build its moments. Noisy is the additional pass with injected per-channel Gaussian noise that our denoising objective needs. Overhead $=$ Noisy $-$ Clean is the extra cost of our second pass.}
\label{tab:forward-times}
\begin{tabular}{@{}l cccc@{}}
\toprule
Model & Clean (s) & Noisy (s) & Overhead (s) & Noisy/Clean \\
\midrule
Llama-3.2-1B & 5.1 & 7.1 & 2.0 & 1.40 \\
Llama-3.2-3B & 12.3 & 16.6 & 4.3 & 1.35 \\
Llama-3-8B & 25.7 & 33.3 & 7.6 & 1.29 \\
Llama-2-13B & 42.9 & 53.5 & 10.6 & 1.25 \\
\bottomrule
\end{tabular}

\end{table}

\clearpage
\newpage
\section{Detailed experimental setup and hyperparameters}
\label{app:setup}

\subsection{Fixed setting}
\label{app:fixed-setting}
Unless a hyperparameter is explicitly swept, all experiments use the following fixed setting.

\textbf{Weight grid and activations.} Symmetric per-channel weights (groupsize $-1$) with MSE clipping; $16$-bit activations and KV cache throughout (weight-only quantization).

\textbf{Rounding solvers and their ridges.} We use two rounding solvers, both regularized with the standard damping convention, i.e.\ a ridge proportional to the mean Hessian diagonal $\mathrm{tr}(\mathbf{H})/d$:
\begin{itemize}
  \item \textbf{GPTQ} (W3/W4): OBS-style column rounding with Hessian damping $\texttt{percdamp}\cdot\mathrm{tr}(\mathbf{H})/d$, block size $128$, no activation-order permutation. Default $\texttt{percdamp}{=}0.01$.
  \item \textbf{QuIP\#} (W2): randomized-Hadamard incoherence with $E_8$-lattice block-LDLQ (block $8$) and $10$ refinement iterations. The trace-normalized LDL metric is ridged by $\sigma_{\mathrm{reg}}$. We do not use recovery fine-tuning.
\end{itemize}
Our denoising objective (\autoref{prop:pretransform-exact}) reuses these same solvers as a drop-in rounder on the pretransformed weights $\tilde{\mathbf{W}}{=}\mathbf{W}\mathbf{F}$ in the metric $\mathbf{G}$. Its proximity ridge $\alpha$ follows the same convention as \texttt{percdamp} ($\alpha\cdot\mathrm{tr}(\cdot)/d$ of the metric being inverted; default $\alpha{=}0.01$), so the two are directly comparable. Because $\mathbf{G}$ already carries the $\alpha\mathbf{I}$ ridge, both solvers round in exactly $\mathbf{G}$ and add no further damping, avoiding a double ridge: the GPTQ sweep inverts $\mathbf{G}$ directly rather than re-damping with \texttt{percdamp}, and the QuIP\# solver sets $\sigma_{\mathrm{reg}}{=}10^{-6}$ instead of its default \texttt{percdamp}.

\textbf{Scheduling.} Sequential methods quantize modules in the order q/k/v\,$\to$\,o\,$\to$\,up/gate\,$\to$\,down. Parallel methods quantize every block from full-precision activations in a single pass.

\textbf{Noise model.} Injected Gaussian noise at scale $\gamma{=}0.05$: isotropic for GPTQ and per-channel for QuIP\#. For an ablation of the noise model, see \cref{app:noise-model}.

\textbf{Calibration.} $128$ WikiText-2 sequences of $2048$ tokens; no BOS token.

\textbf{Evaluation.} Perplexity on WikiText-2, C4, and FineWeb ($2048$-token sequences). Zero-shot accuracy on ARC-easy and ARC-challenge \citep{clark2018arc}, BoolQ \citep{clark2019boolq}, HellaSwag \citep{zellers2019hellaswag}, OBQA \citep{mihaylov2018CanAS}, PIQA \citep{bisk2020piqa}, SIQA \citep{sap2019socialiqa}, and WinoGrande \citep{sakaguchi2021winogrande}, following the setup of \citet{liu2025spinquant}.

\textbf{Determinism and hardware.} We use deterministic kernels and one single GPU per run, a single $40$\,GB NVIDIA A100 unless stated differently. 13B QuIP\# runs on a single NVIDIA H200 (consistently for all experiments). We fix the random seed for all experiments, including the QuaRot random rotations, so that all runs are deterministic and reproducible.

\subsection{QuaRot rotations and learned SpinQuant rotations}

The rotations described here apply to the GPTQ solver (W3/W4). The QuIP\# solver (W2) carries its own randomized-Hadamard incoherence transform (\cref{app:fixed-setting}) and is therefore run without an external rotation, using neither the QuaRot random rotation nor the learned SpinQuant rotation.

\paragraph{Fixed random rotations.} By default we apply the QuaRot randomized-Hadamard rotation ($R_1$, $R_2$, $R_4$) before weight quantization. We fix the random seed used to draw the Hadamard matrices, so that the same random rotations are used across all experiments (we ablate the randomization separately in \autoref{fig:randomness}, see \cref{app:randomness} for details).

\paragraph{Learned rotation.} We also replace the fixed random-Hadamard $R_1$ (and per-head $R_2$) by a learned SpinQuant-style rotation, optimized on the Stiefel manifold with Cayley~SGD against a straight-through weight-fake-quant loss. We use $100$ steps, learning rate $0.3$ with momentum $0.9$, a cosine schedule with $10\%$ linear warm-up, MSE weight clipping (matching the deployment grid), and a training sequence length of $1024$. Only $R_1$ and $R_2$ are learned. The online Hadamard rotation $R_4$ remains a fixed random Hadamard. The learned rotation depends only on the model and its weight grid/bit-width (not on the downstream quantizer or the parallel/sequential schedule), so we compute it once per model and reuse it across all experiments. The $1$B and $3$B rotations are learned on a single A100. For Llama-3-8B and Llama-2-13B we learn it on a single NVIDIA H200 ($141$\,GB) using gradient-checkpointing.

We learn rotations for each model and bit-width once and use the resulting rotation matrices as the fixed rotation. For W3, learning the rotation took $36$ minutes for Llama-3.2-1B (A100), $83$ minutes for Llama-3.2-3B (A100), $73$ minutes for Llama-3-8B (H200), and $122$ minutes for Llama-2-13B (H200). For W4, it took $33$ minutes for Llama-3.2-1B (A100), $77$ minutes for Llama-3.2-3B (A100), $76$ minutes for Llama-3-8B (H200), and $117$ minutes for Llama-2-13B (H200).

\subsection{Main tables hyperparameter sweeps}
\label{app:table-sweeps}

This section gives the hyperparameter sweeps behind the detailed per-model tables:
\begin{itemize}[leftmargin=0.5cm]
  \item \autoref{tab:w3-1b}, \autoref{tab:w3-3b}, \autoref{tab:w3-8b}, \autoref{tab:w3-13b}: Llama-3.2-1B/3B, Llama-3-8B, and Llama-2-13B at W3A16.
  \item \autoref{tab:w4-1b}: Llama-3.2-1B at W4A16.
  \item \autoref{tab:w2-1b}, \autoref{tab:w2-3b}, \autoref{tab:w2-8b}, \autoref{tab:w2-13b}: Llama-3.2-1B/3B, Llama-3-8B, Llama-2-13B at W2A16.
\end{itemize}

The rounding solver depends on the bit-width: the W3/W4 tables use the GPTQ \citep{frantar2023optq} solver (the same solver as the GPTQ baselines) and the W2 tables use the QuIP\# ($E_8$) solver \citep{tseng2024quipsharp}. For W3/W4 (GPTQ) we additionally repeat every configuration once with the fixed random rotation and once with the fixed SpinQuant-learned rotation.

\textbf{Baselines (W3/W4, GPTQ).}
\begin{itemize}[leftmargin=0.5cm]
  \item \textbf{Parallel / Sequential GPTQ (\texttt{percdamp} sweep).} GPTQ damping \texttt{percdamp} over $25$ values, $\{0.001,0.002,0.005,0.01,0.02,0.05\}\cup\{0.1,0.15,0.2,\dots,1.0\}$.
  \item \textbf{Sequential GPTAQ ($\beta$ sweep).} Correction strength $\beta$ over $25$ values evenly spaced in $[0,1]$, at \texttt{percdamp}${=}0.01$ ($\beta$ is called \texttt{alpha} in the GPTAQ code; see \cref{sec:relation}).
\end{itemize}

\textbf{Baselines (W2, QuIP\#).} The same schedules with the QuIP\# solver, where \texttt{percdamp} plays the role of the LDL ridge $\sigma_{\mathrm{reg}}$:
\begin{itemize}[leftmargin=0.5cm]
  \item \textbf{Parallel / Sequential QuIP\# (\texttt{percdamp} sweep).} The same $25$-value \texttt{percdamp} grid; the sequential run keeps the asymmetric correction off.
  \item \textbf{Sequential asymmetric QuIP\# ($\beta$ sweep).} Correction strength $\beta$ over $25$ values evenly spaced in $[0,1]$ at \texttt{percdamp}${=}0.3$ (the symmetric-sequential optimum).
\end{itemize}

\textbf{Our denoising objective.} We run the same $5\times5$ grid over the two knobs of our objective, $\rho\in\{0.25,0.35,0.45,0.55,0.65\}$ and $\delta\in\{0.00625,0.0125,0.025,0.0375,0.05\}$ with $\lambda=\rho+\delta$ ($25$ configurations per proxy). For W3/W4 we use the proximity ridge $\alpha{=}0.01$ and the following parallel proxies for the upstream error $\boldsymbol{\eta}$:
\begin{itemize}[leftmargin=0.5cm]
  \item \textbf{Gaussian noise injection.} Injected, propagated isotropic Gaussian noise at $\gamma{=}0.05$ (see ablation in \cref{app:noise-model}).
  \item \textbf{Parallel-GPTQ (W3A16) upstream error.} Upstream weight-quantization error of a parallel-GPTQ W3A16 model, no injected noise.
  \item \textbf{Both.} The parallel-GPTQ W3A16 upstream error together with injected Gaussian noise at $\gamma{=}0.05$.
\end{itemize}
For W2 we use $\alpha{=}0.3$ and the per-channel Gaussian proxy. We reuse the same fixed random seed for the injected Gaussian noise across all experiments (as for the fixed rotations), and ablate this randomness separately in \autoref{fig:randomness}, see \cref{app:randomness} for details.

\subsection{Randomness ablation}
\label{app:randomness}

We ablate the randomization of each method on Llama-3.2-1B at W3A16. For the two baselines the only source of randomness is the Hadamard rotation, which we resample over $100$ seeds. We additionally report a single run with a fixed learned rotation. For our denoising objective, we resample the injected noise over $100$ seeds, either together with resampled Hadamard rotations or at the fixed learned rotation.

\begin{itemize}[leftmargin=0.5cm]
  \item \textbf{Parallel GPTQ.} At fixed \texttt{percdamp}${=}0.1$, we run $100$ independent runs with different random seeds for the random Hadamard rotation.
  \item \textbf{Sequential GPTAQ.} At fixed \texttt{percdamp}${=}0.01$ and $\beta=0.5$, we run $100$ independent runs with different random seeds for the random Hadamard rotation. Note that GPTAQ is a special case of our denoising objective with $\boldsymbol{\eta}_{seq}$ the sequential upstream error and $\lambda=1$.
  \item \textbf{Our denoising objective with Gaussian noise injection (random rotation, random noise).} At fixed $\alpha{=}0.01$, $\rho=0.5$, $\delta=0.05$, we run $100$ independent runs with different random seeds for the injected Gaussian noise ($\gamma{=}0.05$; no quantized proxy, i.e.\ Gaussian noise only), and with different random seeds for the random Hadamard rotation.
  \item \textbf{Our denoising objective with Gaussian noise injection (learned rotation, random noise).} At fixed $\alpha{=}0.01$, $\rho=0.5$, $\delta=0.05$, and the fixed learned rotation, we run $100$ independent runs with different random seeds for the injected Gaussian noise ($\gamma{=}0.05$; no quantized proxy, i.e.\ Gaussian noise only).
\end{itemize}

\subsection{Frobenius-norm analysis setup}
\label{app:weight-outlier}

\Cref{sec:wiener-shrink} analyzes how our pretransform $\tilde{\mathbf{W}}=\mathbf{W}\mathbf{F}$ reshapes the weights, comparing the Frobenius norm before and after the correction (\autoref{fig:wiener-shrink}, left). We use the fixed setting of \cref{app:setup} on Llama-3.2-1B, with $\rho=0.5$, $\delta=0.05$ (so $\lambda=\rho+\delta=0.55$).

\clearpage
\section{GPTAQ and QEP as special cases}
\label{sec:relation}

In this section, we relate our denoising objective to the two recent works that aim to correct for upstream quantization noise using sequential quantization: GPTAQ \citep{li2025gptaq} and QEP \citep{arai2025qep}. We show that both methods can be expressed as special cases of our denoising objective when passing the sequential upstream error, and we discuss the differences in how they motivate the correction strength $\beta$. We show the closed forms of both methods in \autoref{prop:gptaq-qep}, and then show that they are both special cases of our denoising objective in \autoref{cor:gptaq}.

\subsection{Background on GPTAQ and QEP}

\textbf{GPTAQ} \citep{li2025gptaq} is a sequential method that separates the two roles the input plays in the objective, $\min_{\hat{\mW}} \mathbb{E}\|\mW\vx - \hat{\mW}\tilde{\vx}\|^2$: the target $\mW\vx$ uses the clean activations $\vx$ of the original model, while the quantized weights $\hat{\mW}$ act on the perturbed activations $\tilde{\vx}$ propagated from the previously quantized layers. Each layer thus corrects for the upstream quantization noise it will actually receive, at the price of a sequential pass to measure $\tilde{\vx}$.

\textbf{QEP} \citep{arai2025qep} shares this goal but realizes the correction differently: rather than folding it into GPTQ's column sweep, it first applies a closed-form (ordinary least-squares) pre-transform to the weights and then runs a standard GPTQ round on the transformed weights.

GPTAQ and QEP solve the same problem. Both keep GPTQ's damping ridge and additionally introduce a ``correction strength'' $\beta\in[0,1]$ scaling the asymmetric correction: GPTAQ's \texttt{alpha}, fixed at $0.25$ in its released code, and QEP's \texttt{perccorr}, $0.5$ by default.

\begin{proposition}[Closed forms of GPTAQ and QEP]
\label{prop:gptaq-qep}
Let $\tilde{\mathbf{x}}$ be the layer input once all previous layers are quantized, $\boldsymbol{\eta}\triangleq\tilde{\mathbf{x}}-\mathbf{x}$ the sequential upstream error, assume $\mathbf{D}\triangleq\mathbb{E}[\tilde{\mathbf{x}}\tilde{\mathbf{x}}^\top]$ is invertible, and let $\beta\in[0,1]$ be the correction strength. Both GPTAQ and QEP approximately solve the rounding problem $\min_{\hat{\mathbf{W}}\in\mathcal{G}_b}\operatorname{tr}\big[(\hat{\mathbf{W}}-\mathbf{W}^\star)\,\mathbf{D}\,(\hat{\mathbf{W}}-\mathbf{W}^\star)^\top\big]$ (\autoref{eq:gptq-operator} with metric $\mathbf{D}$) with target
\begin{equation*}
\begin{aligned}
\mathbf{W}^\star
&=\mathbf{W}+\beta\,\mathbf{W}\,\mathbb{E}\big[(\mathbf{x}-\tilde{\mathbf{x}})\tilde{\mathbf{x}}^\top\big]\,\mathbf{D}^{-1}\\
&=(1-\beta)\,\mathbf{W}+\beta\,\mathbf{W}\,\mathbb{E}\big[\mathbf{x}\tilde{\mathbf{x}}^\top\big]\,\mathbf{D}^{-1}.
\end{aligned}
\end{equation*}
GPTAQ realizes this correction inside GPTQ's column sweep, whereas QEP applies it as an explicit weight pretransform followed by a standard GPTQ round.
\end{proposition}
\begin{proof}
At $\beta=1$, the target $\mathbf{W}^\star$ is the continuous minimizer of the asymmetric objective $\mathbb{E}\|\hat{\mathbf{W}}\tilde{\mathbf{x}}-\mathbf{W}\mathbf{x}\|^2$ \citep[Prop.~5.1]{arai2025qep}. With $\boldsymbol{\delta}=\mathbf{X}-\tilde{\mathbf{X}}$ and $\hat{\mathbf{H}}=\tilde{\mathbf{X}}\tilde{\mathbf{X}}^\top$ stacking the calibration tokens as columns, QEP writes the scaled target as $\mathbf{W}+\beta\,\mathbf{W}\boldsymbol{\delta}\tilde{\mathbf{X}}^\top\hat{\mathbf{H}}^{-1}$ \citep[Eq.~(6)]{arai2025qep} and rounds it in the metric $\hat{\mathbf{H}}$ \citep[Eq.~(5)]{arai2025qep}; the $1/n$ normalizations of $\mathbb{E}[(\mathbf{x}-\tilde{\mathbf{x}})\tilde{\mathbf{x}}^\top]$ and $\mathbf{D}$ cancel. GPTAQ scales the same correction term ($\Delta\mathbf{X}\tilde{\mathbf{X}}^\top$ in its notation) by $\beta$ within its column sweep, with $\mathbf{D}$ as its Hessian \citep{li2025gptaq}. The second form follows by linearity of the expectation and $\mathbb{E}[\tilde{\mathbf{x}}\tilde{\mathbf{x}}^\top]=\mathbf{D}$:
\begin{equation*}
\begin{aligned}
\mathbf{W}+\beta\,\mathbf{W}\,\mathbb{E}\big[(\mathbf{x}-\tilde{\mathbf{x}})\tilde{\mathbf{x}}^\top\big]\,\mathbf{D}^{-1}
&=\mathbf{W}+\beta\,\mathbf{W}\big(\mathbb{E}[\mathbf{x}\tilde{\mathbf{x}}^\top]-\mathbb{E}[\tilde{\mathbf{x}}\tilde{\mathbf{x}}^\top]\big)\mathbf{D}^{-1}\\
&=\mathbf{W}+\beta\,\mathbf{W}\big(\mathbb{E}[\mathbf{x}\tilde{\mathbf{x}}^\top]-\mathbf{D}\big)\mathbf{D}^{-1}\\
&=(1-\beta)\,\mathbf{W}+\beta\,\mathbf{W}\,\mathbb{E}[\mathbf{x}\tilde{\mathbf{x}}^\top]\,\mathbf{D}^{-1}.
\end{aligned}
\end{equation*}
\end{proof}
Here $\beta$ linearly interpolates in weight space between the full-precision weights ($\beta=0$, symmetric sequential GPTQ) and the fully corrected solution ($\beta=1$). 
But the two works motivate this parameter $\beta$ differently. QEP presents $\beta$ as a tunable propagation mechanism that prevents overfitting the correction to the calibration set and reports that a fixed $\beta=1/2$ works well in most cases.
GPTAQ does not present its $\beta$ in the paper at all; it appears only in the released code, justified in a GitHub issue as an empirical choice.\footnote{\url{https://github.com/Intelligent-Computing-Lab-Panda/GPTAQ/issues/4}} We argue that neither provides a principled motivation for this parameter, and find that the mechanism of their empirical correction is better understood as noise scaling in our framework, which we show below.

\subsection{GPTAQ and QEP as special cases of activation denoising}

We show that both GPTAQ and QEP are special cases of our denoising objective (\autoref{eq:wiener-objective}) when passing the sequential upstream error $\boldsymbol{\eta}=\tilde{\mathbf{x}}-\mathbf{x}$, with $\lambda{=}1$ (and $\alpha{=}0$ for simplicity). The correction strength $\beta$ of GPTAQ and QEP is exactly $1{-}\rho$, where $\rho$ is the noise scaling parameter of our denoising objective.

\begin{corollary}[GPTAQ and QEP as special cases of activation denoising]
\label{cor:gptaq}
Let $\boldsymbol{\eta}=\tilde{\mathbf{x}}-\mathbf{x}$ be the sequential upstream error and assume $\mathbf{D}=\mathbb{E}[\tilde{\mathbf{x}}\tilde{\mathbf{x}}^\top]$ is invertible. Then our denoising objective (\autoref{eq:wiener-objective}) with noise $\boldsymbol{\eta}$ at $\lambda=1$ (and $\alpha=0$ for simplicity) satisfies, for every $\hat{\mathbf{W}}$,
\begin{equation*}
\mathbb{E}\big\|\hat{\mathbf{W}}\tilde{\mathbf{x}}-\mathbf{W}(\mathbf{x}+\rho\,\boldsymbol{\eta})\big\|^2
=\operatorname{tr}\!\big[(\hat{\mathbf{W}}-\mathbf{W}\mathbf{F})\,\mathbf{D}\,(\hat{\mathbf{W}}-\mathbf{W}\mathbf{F})^\top\big]+c,
\end{equation*}
where $c$ does not depend on $\hat{\mathbf{W}}$, so both objectives have the same minimizers over $\mathcal{G}_b$, and
\begin{equation*}
\begin{aligned}
\mathbf{W}\mathbf{F}
&=\mathbf{W}+(1-\rho)\,\mathbf{W}\,\mathbb{E}\big[(\mathbf{x}-\tilde{\mathbf{x}})\tilde{\mathbf{x}}^\top\big]\,\mathbf{D}^{-1}\\
&=\rho\,\mathbf{W}+(1-\rho)\,\mathbf{W}\,\mathbb{E}\big[\mathbf{x}\tilde{\mathbf{x}}^\top\big]\,\mathbf{D}^{-1}.
\end{aligned}
\end{equation*}
Comparing with \autoref{prop:gptaq-qep}, this is exactly the GPTAQ and QEP rounding problem with correction strength $\beta=1-\rho$.
\end{corollary}
\begin{proof}
Setting $\lambda=1$ and $\alpha=0$ in the metric of \autoref{prop:pretransform-exact} gives
\begin{equation*}
\mathbf{G}=\mathbf{H}+\alpha\mathbf{I}+\lambda^2\mathbf{N}+\lambda(\mathbf{C}+\mathbf{C}^\top)
=\mathbf{H}+\mathbf{C}+\mathbf{C}^\top+\mathbf{N}.
\end{equation*}
Expanding $\mathbf{D}$ with $\tilde{\mathbf{x}}=\mathbf{x}+\boldsymbol{\eta}$ and using linearity of the expectation,
\begin{equation*}
\mathbf{D}=\mathbb{E}\big[(\mathbf{x}+\boldsymbol{\eta})(\mathbf{x}+\boldsymbol{\eta})^\top\big]
=\mathbb{E}[\mathbf{x}\mathbf{x}^\top]+\mathbb{E}[\mathbf{x}\boldsymbol{\eta}^\top]+\mathbb{E}[\boldsymbol{\eta}\mathbf{x}^\top]+\mathbb{E}[\boldsymbol{\eta}\boldsymbol{\eta}^\top]
=\mathbf{H}+\mathbf{C}+\mathbf{C}^\top+\mathbf{N},
\end{equation*}
so $\mathbf{G}=\mathbf{D}$. 

Likewise, the filter becomes $\mathbf{F}=(\mathbf{H}+\alpha\mathbf{I}+\lambda\mathbf{C}+\rho\mathbf{C}^\top+\rho\lambda\mathbf{N})\mathbf{G}^{-1}=(\mathbf{H}+\mathbf{C}+\rho\,\mathbf{C}^\top+\rho\,\mathbf{N})\mathbf{D}^{-1}$. The stated identity is the completed square from the proof of \autoref{prop:pretransform-exact}, which only uses that $\mathbf{G}$ is symmetric and invertible; this holds for $\alpha=0$ since $\mathbf{G}=\mathbf{D}$ is invertible by assumption.

It remains to write $\mathbf{W}\mathbf{F}$ in terms of the moments of $\tilde{\mathbf{x}}$. We group the numerator of $\mathbf{F}$ into the terms without and with $\rho$, $\mathbf{F}=\big(\mathbf{H}+\mathbf{C}+\rho\,(\mathbf{C}^\top+\mathbf{N})\big)\mathbf{D}^{-1}$. The first group is $\mathbb{E}[\mathbf{x}\tilde{\mathbf{x}}^\top]=\mathbb{E}[\mathbf{x}(\mathbf{x}+\boldsymbol{\eta})^\top]=\mathbf{H}+\mathbf{C}$, and by the expansion of $\mathbf{D}$ above, the second group is $\mathbf{C}^\top+\mathbf{N}=\mathbf{D}-\mathbb{E}[\mathbf{x}\tilde{\mathbf{x}}^\top]$. Hence
\begin{equation*}
\begin{aligned}
\mathbf{W}\mathbf{F}
&=\mathbf{W}\big(\mathbb{E}[\mathbf{x}\tilde{\mathbf{x}}^\top]+\rho\,(\mathbf{D}-\mathbb{E}[\mathbf{x}\tilde{\mathbf{x}}^\top])\big)\mathbf{D}^{-1}\\
&=\rho\,\mathbf{W}+(1-\rho)\,\mathbf{W}\,\mathbb{E}[\mathbf{x}\tilde{\mathbf{x}}^\top]\,\mathbf{D}^{-1},
\end{aligned}
\end{equation*}
which is the second form. To obtain the first form, we split $\rho\,\mathbf{W}=\mathbf{W}-(1-\rho)\,\mathbf{W}$ and collect the terms with $(1-\rho)$:
\begin{equation*}
\begin{aligned}
\rho\,\mathbf{W}+(1-\rho)\,\mathbf{W}\,\mathbb{E}[\mathbf{x}\tilde{\mathbf{x}}^\top]\,\mathbf{D}^{-1}
&=\mathbf{W}-(1-\rho)\,\mathbf{W}+(1-\rho)\,\mathbf{W}\,\mathbb{E}[\mathbf{x}\tilde{\mathbf{x}}^\top]\,\mathbf{D}^{-1}\\
&=\mathbf{W}+(1-\rho)\,\mathbf{W}\big(\mathbb{E}[\mathbf{x}\tilde{\mathbf{x}}^\top]-\mathbf{D}\big)\mathbf{D}^{-1}\\
&=\mathbf{W}+(1-\rho)\,\mathbf{W}\,\mathbb{E}\big[(\mathbf{x}-\tilde{\mathbf{x}})\tilde{\mathbf{x}}^\top\big]\,\mathbf{D}^{-1},
\end{aligned}
\end{equation*}
where the second step uses $\mathbf{W}=\mathbf{W}\mathbf{D}\mathbf{D}^{-1}$ and the last step uses $\mathbf{D}=\mathbb{E}[\tilde{\mathbf{x}}\tilde{\mathbf{x}}^\top]$ and linearity of the expectation.
\end{proof}
The correction strength of GPTAQ and QEP therefore corresponds to our target noise scale: full correction $\beta=1$ is full denoising ($\rho{=}0$), and $\beta=0$ is symmetric sequential GPTQ ($\rho=\lambda=1$). We believe our noise scaling perspective is more principled than the empirical motivations of GPTAQ and QEP, and we find that it generalizes to other upstream error proxies beyond the sequential upstream error.

\clearpage\newpage
\section{Full proofs and derivations}
\label{app:proofs}

\providecommand{\lemmaautorefname}{Lemma}\providecommand{\corollaryautorefname}{Corollary}%
\newenvironment{restatedtheorem}[1]
  {\renewcommand{\thetheorem}{\ref{#1}}\renewcommand{\theHtheorem}{restated.#1}\begin{theorem}}
  {\end{theorem}\addtocounter{theorem}{-1}}

In this section we restate the theorems of the main text and provide their full proofs.

\subsection{Setup and notation}
\label{app:proofs-setup}

Let $\mathbf{W}\in\sR^{d_{\mathrm{out}}\times d}$ be the full-precision weights of a linear layer with input dimension $d$. Let $\mathbf{x}\in\sR^{d}$ be the clean input activation and $\boldsymbol{\eta}\in\sR^{d}$ the input noise, jointly distributed with finite second moments; all expectations are over this joint distribution (for empirical moments, the uniform distribution over the $n$ calibration tokens). We abbreviate the noisy input and the target input of the denoising objective in \autoref{eq:wiener-objective} as
\begin{equation*}
\mathbf{u}\;\triangleq\;\mathbf{x}+\lambda\boldsymbol{\eta},\qquad \mathbf{v}\;\triangleq\;\mathbf{x}+\rho\boldsymbol{\eta},
\end{equation*}
so that the objective reads
\begin{equation}
J(\hat{\mathbf{W}})\;\triangleq\;\mathbb{E}\big\|\hat{\mathbf{W}}\mathbf{u}-\mathbf{W}\mathbf{v}\big\|^2+\alpha\,\|\hat{\mathbf{W}}-\mathbf{W}\|_F^2 .
\label{eq:app-objective}
\end{equation}
With $\mathbf{H}=\mathbb{E}[\mathbf{x}\mathbf{x}^\top]$, $\mathbf{C}=\mathbb{E}[\mathbf{x}\boldsymbol{\eta}^\top]$ (hence $\mathbb{E}[\boldsymbol{\eta}\mathbf{x}^\top]=\mathbf{C}^\top$) and $\mathbf{N}=\mathbb{E}[\boldsymbol{\eta}\boldsymbol{\eta}^\top]$, expanding the outer products gives the second moments
\begin{align}
\mathbb{E}[\mathbf{u}\mathbf{u}^\top] &= \mathbf{H}+\lambda(\mathbf{C}+\mathbf{C}^\top)+\lambda^2\mathbf{N}, \label{eq:app-uu}\\
\mathbb{E}[\mathbf{v}\mathbf{u}^\top] &= \mathbf{H}+\lambda\mathbf{C}+\rho\mathbf{C}^\top+\rho\lambda\mathbf{N}, \label{eq:app-vu}\\
\mathbb{E}[\mathbf{v}\mathbf{v}^\top] &= \mathbf{H}+\rho(\mathbf{C}+\mathbf{C}^\top)+\rho^2\mathbf{N}. \label{eq:app-vv}
\end{align}
In terms of these, the two matrices of \autoref{eq:wiener-filter} are
\begin{equation}
\mathbf{G}=\mathbb{E}[\mathbf{u}\mathbf{u}^\top]+\alpha\mathbf{I},\qquad
\mathbf{B}\;\triangleq\;\mathbb{E}[\mathbf{v}\mathbf{u}^\top]+\alpha\mathbf{I}=\mathbf{H}+\alpha\mathbf{I}+\lambda\mathbf{C}+\rho\mathbf{C}^\top+\rho\lambda\mathbf{N},
\qquad \mathbf{F}=\mathbf{B}\mathbf{G}^{-1}.
\label{eq:app-GB}
\end{equation}

\textbf{Assumption.} Throughout, we assume $\alpha>0$, as in \autoref{prop:pretransform-exact}. Since $\mathbb{E}[\mathbf{u}\mathbf{u}^\top]$ is a second-moment matrix and therefore positive semidefinite, $\mathbf{G}$ is positive definite. In particular, $\mathbf{G}$ is symmetric and invertible, so $\mathbf{F}$ is well defined.

\subsection{Proof of \autoref{prop:pretransform-exact}}
\label{app:proofs-pretransform}

\begin{restatedtheorem}{prop:pretransform-exact}[Pretransform-round equivalence]
For $\alpha>0$, the optimization problem in \autoref{eq:wiener-objective} is equivalent to
\begin{equation*}
\min_{\hat{\mathbf{W}}\in\mathcal{G}_b}\ \operatorname{tr}\!\big[(\hat{\mathbf{W}}-\mathbf{W}\mathbf{F})\,\mathbf{G}\,(\hat{\mathbf{W}}-\mathbf{W}\mathbf{F})^\top\big],
\end{equation*}
where
\begin{equation*}
\begin{aligned}
\mathbf{G}&\;\triangleq\;\mathbf{H}+\alpha\mathbf{I}+\lambda^2\mathbf{N}+\lambda(\mathbf{C}+\mathbf{C}^\top),\\
\mathbf{F}&\;\triangleq\;\big(\mathbf{H}+\alpha\mathbf{I}+\lambda\mathbf{C}+\rho\mathbf{C}^\top+\rho\lambda\mathbf{N}\big)\,\mathbf{G}^{-1}.
\end{aligned}
\end{equation*}
Here $\mathbf{H}=\mathbb{E}[\mathbf{x}\mathbf{x}^\top]$, $\mathbf{C}=\mathbb{E}[\mathbf{x}\boldsymbol{\eta}^\top]$, and $\mathbf{N}=\mathbb{E}[\boldsymbol{\eta}\boldsymbol{\eta}^\top]$ are assumed finite, with expectations over the joint distribution of $(\mathbf{x},\boldsymbol{\eta})$.
\end{restatedtheorem}
\begin{proof}
Using the notation of \autoref{app:proofs-setup}, let $\mathbf{D}\triangleq\mathbb{E}[\mathbf{v}\mathbf{v}^\top]+\alpha\mathbf{I}$. Expanding the reconstruction and proximity terms in \autoref{eq:app-objective} and using the invariance of trace under transposition yields
\begin{equation*}
J(\hat{\mathbf{W}})
=\operatorname{tr}\!\big[\hat{\mathbf{W}}\mathbf{G}\hat{\mathbf{W}}^\top\big]
-2\operatorname{tr}\!\big[\mathbf{W}\mathbf{B}\hat{\mathbf{W}}^\top\big]
+\operatorname{tr}\!\big[\mathbf{W}\mathbf{D}\mathbf{W}^\top\big].
\end{equation*}
Set $\tilde{\mathbf{W}}=\mathbf{W}\mathbf{F}$. Since $\mathbf{F}=\mathbf{B}\mathbf{G}^{-1}$, we have $\tilde{\mathbf{W}}\mathbf{G}=\mathbf{W}\mathbf{B}$. Using this equality and the symmetry of $\mathbf{G}$ to complete the square gives
\begin{equation*}
\begin{aligned}
J(\hat{\mathbf{W}})
&=\operatorname{tr}\!\big[(\hat{\mathbf{W}}-\tilde{\mathbf{W}})\mathbf{G}(\hat{\mathbf{W}}-\tilde{\mathbf{W}})^\top\big]+c,\\
c&=\operatorname{tr}\!\big[\mathbf{W}\big(\mathbf{D}-\mathbf{F}\mathbf{G}\mathbf{F}^\top\big)\mathbf{W}^\top\big].
\end{aligned}
\end{equation*}
The constant $c$ does not depend on $\hat{\mathbf{W}}$. This identity holds for every real candidate, hence for every candidate on the same grid $\mathcal{G}_b$. Adding a constant does not change the set of minimizers, proving the stated equivalence.
\end{proof}

\subsection{Proof of \autoref{thm:inject-lipschitz}}
\label{app:proofs-injection}

We consider a single layer with propagated upstream noise $\boldsymbol{\eta}^{\uparrow}$ and independent local injection $\boldsymbol{\xi}$, so the total noise is $\boldsymbol{\eta}=\boldsymbol{\eta}^{\uparrow}+\boldsymbol{\xi}$. The proof uses only $\mathbb{E}[\boldsymbol{\xi}]=\mathbf{0}$, $\mathbb{E}[\boldsymbol{\xi}\boldsymbol{\xi}^\top]=\boldsymbol{\Sigma}=\operatorname{diag}(\sigma_1^2,\dots,\sigma_d^2)$, and independence of $\boldsymbol{\xi}$ from $(\mathbf{x},\boldsymbol{\eta}^{\uparrow})$; Gaussianity is not needed.

\begin{restatedtheorem}{thm:inject-lipschitz}[Local noise injection acts as a magnitude-reducing ridge]
Let $\boldsymbol{\eta}=\boldsymbol{\eta}^{\uparrow}+\boldsymbol{\xi}$, where $\boldsymbol{\xi}\sim\mathcal{N}(\mathbf{0},\boldsymbol{\Sigma})$ is independent of $(\mathbf{x},\boldsymbol{\eta}^{\uparrow})$ and $\boldsymbol{\Sigma}=\operatorname{diag}(\sigma_1^2,\dots,\sigma_d^2)$. Up to a constant independent of the weights, Eq.\ \ref{eq:wiener-objective} equals its reconstruction term with $\boldsymbol{\eta}=\boldsymbol{\eta}^{\uparrow}$ plus the modified ridge regularization
\begin{equation*}
\sum_c\big(\alpha+\lambda^2\sigma_{c}^2\big)\,\big\|\hat{\mathbf{W}}_{:,c}-s_{c}\,\mathbf{W}_{:,c}\big\|_2^2, \qquad s_{c}=\frac{\alpha+\rho\lambda\sigma_{c}^2}{\alpha+\lambda^2\sigma_{c}^2},
\end{equation*}
assuming $\alpha+\lambda^2\sigma_c^2>0$ for every $c$. For $\rho<\lambda$ and $\sigma_c>0$, the penalty favors a down-scaled target $s_c\mathbf{W}_{:,c}$ for column $c$.
\end{restatedtheorem}
\begin{proof}
\emph{Step 1: the injection adds a column-wise penalty.} With the injected noise, the reconstruction term of the objective is
\begin{equation*}
\mathbb{E}_{\mathbf{x},\boldsymbol{\eta}^{\uparrow},\boldsymbol{\xi}}\big\|\hat{\mathbf{W}}\big(\mathbf{x}+\lambda(\boldsymbol{\eta}^{\uparrow}+\boldsymbol{\xi})\big)-\mathbf{W}\big(\mathbf{x}+\rho(\boldsymbol{\eta}^{\uparrow}+\boldsymbol{\xi})\big)\big\|_2^2
=\mathbb{E}\big\|\mathbf{r}+\mathbf{P}\boldsymbol{\xi}\big\|_2^2,
\end{equation*}
where $\mathbf{r}\triangleq\hat{\mathbf{W}}(\mathbf{x}+\lambda\boldsymbol{\eta}^{\uparrow})-\mathbf{W}(\mathbf{x}+\rho\boldsymbol{\eta}^{\uparrow})$ collects everything that does not depend on $\boldsymbol{\xi}$, and $\mathbf{P}\triangleq\lambda\hat{\mathbf{W}}-\rho\mathbf{W}$. Expanding the square,
\begin{equation*}
\mathbb{E}\big\|\mathbf{r}+\mathbf{P}\boldsymbol{\xi}\big\|_2^2=\mathbb{E}\|\mathbf{r}\|_2^2+2\,\mathbb{E}\big[\mathbf{r}^\top\mathbf{P}\boldsymbol{\xi}\big]+\mathbb{E}\big\|\mathbf{P}\boldsymbol{\xi}\big\|_2^2 .
\end{equation*}
The cross term vanishes: $\mathbf{r}$ is a function of $(\mathbf{x},\boldsymbol{\eta}^{\uparrow})$ only, and $\boldsymbol{\xi}$ is independent of $(\mathbf{x},\boldsymbol{\eta}^{\uparrow})$ with zero mean, so $\mathbb{E}[\mathbf{r}^\top\mathbf{P}\boldsymbol{\xi}]=\mathbb{E}[\mathbf{r}]^\top\mathbf{P}\,\mathbb{E}[\boldsymbol{\xi}]=0$. For the last term, using $\mathbb{E}[\boldsymbol{\xi}\boldsymbol{\xi}^\top]=\boldsymbol{\Sigma}$ and writing $\mathbf{P}_{:,c}$ for column $c$ of $\mathbf{P}$,
\begin{equation*}
\mathbb{E}\big\|\mathbf{P}\boldsymbol{\xi}\big\|_2^2=\operatorname{tr}\!\big[\mathbf{P}\,\boldsymbol{\Sigma}\,\mathbf{P}^\top\big]=\big\|\mathbf{P}\boldsymbol{\Sigma}^{1/2}\big\|_F^2=\sum_c\sigma_c^2\big\|\mathbf{P}_{:,c}\big\|_2^2=\sum_c\sigma_c^2\big\|\lambda\hat{\mathbf{W}}_{:,c}-\rho\mathbf{W}_{:,c}\big\|_2^2,
\end{equation*}
the last expression defining the injection penalty $\mathcal{R}(\hat{\mathbf{W}})\triangleq\sum_c\sigma_c^2\big\|\lambda\hat{\mathbf{W}}_{:,c}-\rho\mathbf{W}_{:,c}\big\|_2^2$. The first term $\mathbb{E}\|\mathbf{r}\|_2^2$ is exactly the reconstruction term of \autoref{eq:wiener-objective} with $\boldsymbol{\eta}=\boldsymbol{\eta}^{\uparrow}$. Hence the objective with injection equals the objective without local injection plus $\mathcal{R}(\hat{\mathbf{W}})$.

\emph{Step 2: combining $\mathcal{R}$ with the proximity ridge.} Because $\boldsymbol{\Sigma}$ is diagonal and the proximity ridge $\alpha\|\hat{\mathbf{W}}-\mathbf{W}\|_F^2=\alpha\sum_c\|\hat{\mathbf{W}}_{:,c}-\mathbf{W}_{:,c}\|_2^2$ is column-separable, the combined penalty splits over columns:
\begin{equation*}
\mathcal{R}(\hat{\mathbf{W}})+\alpha\|\hat{\mathbf{W}}-\mathbf{W}\|_F^2=\sum_c\Big(\sigma_c^2\,\|\lambda\hat{\mathbf{W}}_{:,c}-\rho\mathbf{W}_{:,c}\|_2^2+\alpha\,\|\hat{\mathbf{W}}_{:,c}-\mathbf{W}_{:,c}\|_2^2\Big).
\end{equation*}
Fix a column $c$ and let $\langle\cdot,\cdot\rangle$ denote the Euclidean inner product. Expanding the two squared norms of that column,
\begin{equation*}
\sigma_c^2\,\|\lambda\hat{\mathbf{W}}_{:,c}-\rho\mathbf{W}_{:,c}\|_2^2+\alpha\,\|\hat{\mathbf{W}}_{:,c}-\mathbf{W}_{:,c}\|_2^2
\end{equation*}
\begin{equation*}
=(\alpha+\lambda^2\sigma_c^2)\,\|\hat{\mathbf{W}}_{:,c}\|_2^2-2(\alpha+\rho\lambda\sigma_c^2)\,\langle\hat{\mathbf{W}}_{:,c},\mathbf{W}_{:,c}\rangle+(\alpha+\rho^2\sigma_c^2)\,\|\mathbf{W}_{:,c}\|_2^2 .
\end{equation*}
Let $\kappa_c\triangleq\alpha+\lambda^2\sigma_c^2$, positive whenever $\alpha>0$ or $\lambda\sigma_c\neq0$, and $s_c=(\alpha+\rho\lambda\sigma_c^2)/\kappa_c$ as in the statement. Then
\begin{equation*}
\kappa_c\,\|\hat{\mathbf{W}}_{:,c}-s_c\mathbf{W}_{:,c}\|_2^2=\kappa_c\,\|\hat{\mathbf{W}}_{:,c}\|_2^2-2(\alpha+\rho\lambda\sigma_c^2)\,\langle\hat{\mathbf{W}}_{:,c},\mathbf{W}_{:,c}\rangle+\kappa_c\,s_c^2\,\|\mathbf{W}_{:,c}\|_2^2 ,
\end{equation*}
which matches the previous display in every term that depends on $\hat{\mathbf{W}}_{:,c}$. The difference is the column constant
\begin{equation*}
\mathrm{const}_c=\Big(\alpha+\rho^2\sigma_c^2-\frac{(\alpha+\rho\lambda\sigma_c^2)^2}{\alpha+\lambda^2\sigma_c^2}\Big)\|\mathbf{W}_{:,c}\|_2^2
=\frac{\alpha\,\sigma_c^2(\lambda-\rho)^2}{\alpha+\lambda^2\sigma_c^2}\,\|\mathbf{W}_{:,c}\|_2^2,
\end{equation*}
where the last equality uses $(\alpha+\rho^2\sigma_c^2)(\alpha+\lambda^2\sigma_c^2)-(\alpha+\rho\lambda\sigma_c^2)^2=\alpha\sigma_c^2(\lambda^2-2\lambda\rho+\rho^2)$. Summing over $c$ gives the stated decomposition with $\mathrm{const}=\sum_c\mathrm{const}_c$, which does not depend on $\hat{\mathbf{W}}$.
\end{proof}

\textbf{Remark (down-scaling).} For any channel with $\sigma_c>0$ and $\lambda>0$ we have $s_c<1$ if and only if $\rho\lambda\sigma_c^2<\lambda^2\sigma_c^2$, i.e.\ if and only if $\rho<\lambda$. In this regime the combined regularizer pulls each column $\hat{\mathbf{W}}_{:,c}$ toward the down-scaled $s_c\mathbf{W}_{:,c}$, with $s_c\to\rho/\lambda$ as $\alpha\to0$ and $s_c\to1$ as $\alpha\to\infty$. For $\rho=\lambda$ we obtain $s_c=1$ and the regularizer is centered at the original weight column.

\end{document}